\documentclass[final,onefignum,onetabnum]{siamonline250211}

\usepackage[section]{placeins} % Figure barrier per section.
\usepackage{subcaption}
\usepackage{algpseudocode}
\usepackage{setspace}
\usepackage{booktabs}
\usepackage{amssymb}
\usepackage{microtype}

\usepackage{enumitem}

\newcommand{\func}[2]{#1[#2]} % function
\newcommand{\vmat}[1]{\boldsymbol{#1}} % variable matrix
\newcommand{\cmat}[1]{\mathbf{#1}} % constant matrix
\newcommand{\vvec}[1]{\boldsymbol{#1}} % variable column vector
\newcommand{\cvec}[1]{\mathbf{#1}} % constant column vector
\newcommand{\transpose}[1]{{#1}^{\!\top}} % transpose
\newcommand{\cscalar}[1]{\mathtt{#1}} % constant scalar
\newcommand\eqdef{\stackrel{\mathrm{def}}{=}}
\DeclareMathOperator{\trace}{Tr}
\DeclareMathOperator{\rms}{RMS\!}
\DeclareMathOperator{\mean}{mean}
\DeclareMathOperator{\diag}{\mathbf{Diag}}

\newsiamthm{prediction}{Prediction}
\newsiamthm{assumption}{Assumption}

\usepackage{url}

\usepackage[authoryear,round]{natbib}

\headers{Simulating Training Spikes of Deep Neural Networks}{R. Tian and C. Paquette}

\title{Simulating Training Spikes of Deep Neural Networks by Layer-wise Random Feature Models}
\author{Ran Tian\thanks{Google Deepmind 
  (\email{tianran@google.com})
  }
\and Courtney Paquette\footnotemark[1]
}

\begin{document}

\maketitle

\begin{abstract}
Intermittent ``spikes'' in loss, during the pre-training of Large Language Models (LLMs), remains a mysterious phenomenon both theoretically and empirically. In this paper, we show that simple quadratic models with heavy-tailed, power-law distributed inputs can replicate these spikes. We introduce a theoretical framework and simulation algorithm to model per-layer training dynamics of Deep Neural Networks (DNNs), using such quadratic models. In a case study on an Adam hyperparameter configuration (where $\beta_2$ is set close to $1$) known to induce spikes in Transformer LMs, our simulation successfully predicts bumps indicating spiky behavior from a given checkpoint. Furthermore, we derive an appropriate $\beta_2$ schedule from observed metrics, which empirically improves performance.
\end{abstract}

\begin{keywords}
Training spikes, Stochastic optimization, Random Feature models, Power law
\end{keywords}

\section{Introduction}

``Training spikes'', the temporary, sharp surges in loss (Figure~\ref{fig:spiky_loss_all}), are a phenomenon often observed during the pre-training of Large Language Models (LLMs) \citep{palm,DBLP:journals/corr/abs-2304-09871}. Empirically, these spikes are considered pathological because the increase in loss may indicate training instability, the hyperparameters being sub-optimal, or the final model less predictive. The wild fluctuations in training metrics can trigger numerical issues as well, since large-scale Deep Neural Networks (DNNs) are usually trained with low-precision arithmetic. While their empirical occurrence is well documented, the underlying causes of these training spikes are not yet fully understood, partly because their occurrences can be affected by many factors, sometimes difficult to reliably reproduce in controlled settings \citep{nishida-etal-2024-initialization,takase2025spike}.

Stochastic optimizers, such as SGD and Adam \citep{kingma2014adam}, inherently produce fluctuations due to the random sampling of training data. However, the ``training spikes'' discussed in this paper are distinct in their sensitivity to \emph{both} data shuffling and model initialization. Typically, the trajectory of training loss fluctuations (the ``up-and-downs'' of the curve) remains consistent under an identical data shuffle (Figure~\ref{fig:spiky_loss_train_partial}), irrespective of model initialization or minor architectural modifications. 
% Typically, it is observed that the topology of training loss fluctuations, or the ``up-and-downs'' of a loss curve as zoomed in, remains consistent as long as the data shuffle is identical (Figure~\ref{fig:spiky_loss_train_partial}), regardless of model initialization or minor architectural changes. 
However, training spikes do not follow this pattern. As shown in Figure~\ref{fig:spiky_loss_train_full}, two training trajectories, which differ only by their random initialization seeds, can spike at different steps. This extra randomness complicates the study and mitigation of training spikes. For example, simply removing ``data outliers'' as a remedy to training spikes will only partially solve the problem, since one can often still observe random surges in the \emph{validation loss} (Figure~\ref{fig:spiky_loss_valid_partial}). This suggests that training spikes are usually not caused by a specific occurrence of one single data point.

%It is noteworthy that when stochastic optimizers such as SGD and Adam are used, training loss is inherently subject to fluctuations due to the random sampling of training examples. What distinguishes training spikes from such ``normal'' fluctuations is the sensitivity to not only the data-shuffle, but also the initialization -- It is observed that the up-and-downs of training loss normally exhibit the same ``topology'', independent of model initialization or small modifications of architecture, when the data-shuffle is kept the same (Figure~\ref{fig:spiky_loss_train_partial}). However, this property no longer holds for training spikes: In Figure~\ref{fig:spiky_loss_train_full}, we show an example of two training trajectories, with random-seed of model initialization being the only difference, spiking at different places during training. This demonstration of extra randomness makes the study of training spikes more difficult, and its remedy more elusive. For example, without being fully aware of the different types of randomness, one might propose to remove the ``data outliers'' as a remedy for training spikes. That is unlikely the complete solution, because we observe the validation loss spiking at the same place (Figure~\ref{fig:spiky_loss_valid_partial}). It suggests that training spikes are actually sudden increases of the \emph{average} loss, which is not caused by the specific data points where the spikes happen.

In \S~\ref{sec:sketch}, we will detail our theoretical framework, which posits that normal fluctuations in loss are driven solely by data shuffling because high model dimensionality typically cancels out model-originated randomness. This concentration of measure property is conditional; it breaks down when the data distribution is heavy-tailed, resulting in training spikes. We demonstrate this mechanism using Power-Law Random Feature (PLRF) models -- a class of simple linear-regression models with quadratic loss. Furthermore, we introduce an algorithm that uses PLRF-like dynamics to simulate the training behavior of individual layers within a DNN.

%In Section~\ref{sec:sketch}, we will elaborate our theoretical understanding that the property of ``normal'' fluctuations of training loss being driven by the data-shuffle only, is due to the effect that the high-dimensionality of model can cancel out all the randomness originated from the model. This \emph{concentration of measure} property holds with conditions -- and it can be broken if the data distribution is heavy-tailed -- when broken, training spikes are observed. This mechanism can be demonstrated by a class of simple, linear-regression models with quadratic loss, called the Power-Law Random Feature (PLRF) models. More generally, we will present an algorithm to simulate the training behavior of each layer of a DNN by PLRF-like dynamics. 

\begin{figure}[t]
 \centering
 \begin{subfigure}[b]{0.6\textwidth}
     \centering
     \includegraphics[width=\textwidth]{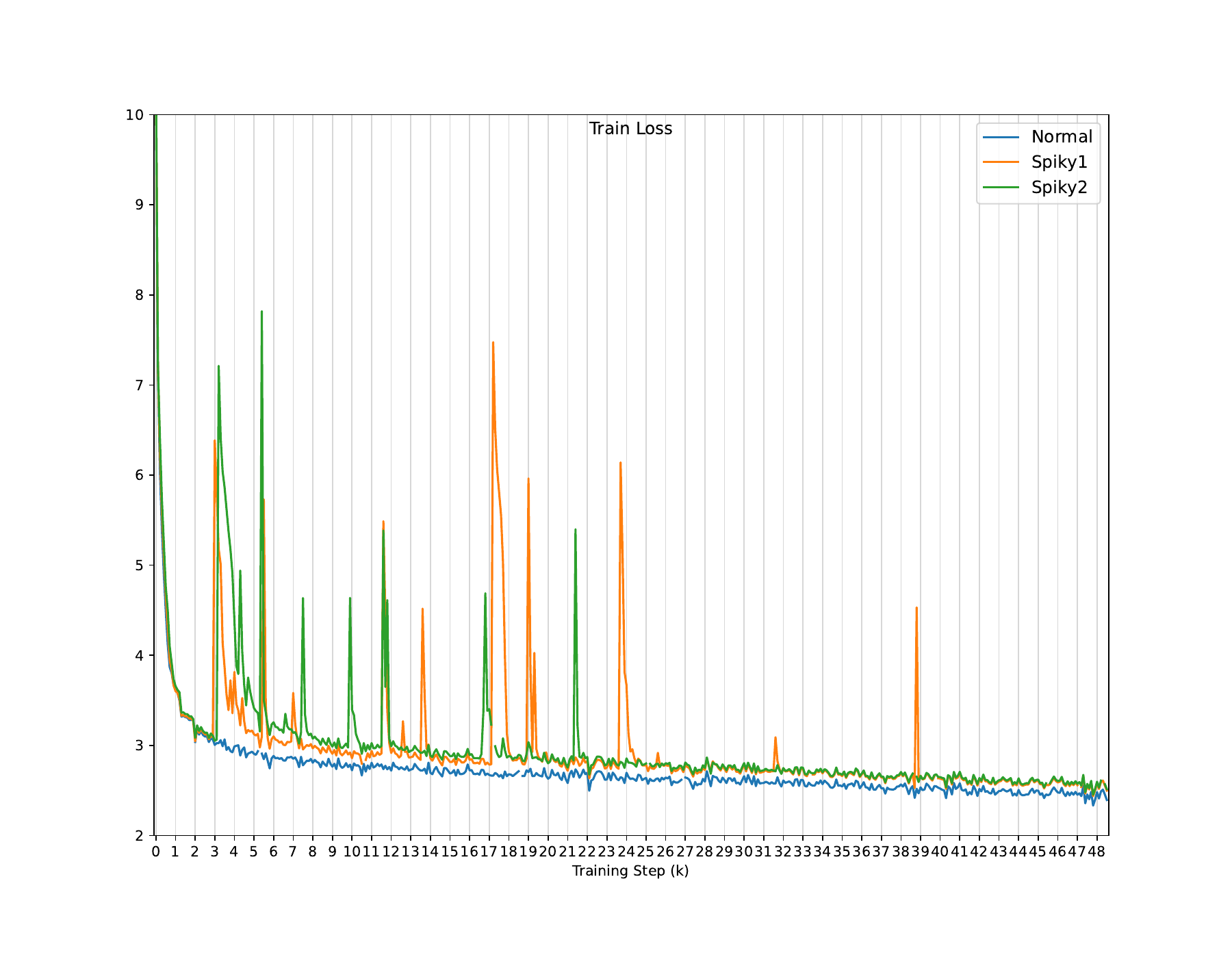}
     \caption{Overall trajectories of the training loss.}
     \label{fig:spiky_loss_train_full}
     \end{subfigure}
 \begin{minipage}[b]{0.39\textwidth}
 \centering
 \begin{subfigure}[b]{\textwidth}
     \centering
     \includegraphics[width=\textwidth]{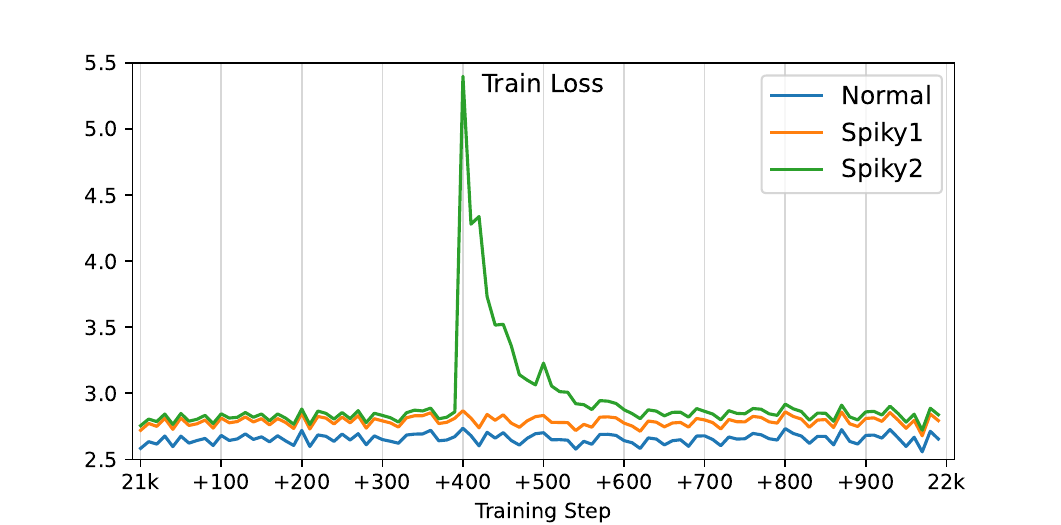}
     \caption{Training loss from 21k to 22k.}
     \label{fig:spiky_loss_train_partial}
 \end{subfigure}
 \\
 \begin{subfigure}[b]{\textwidth}
     \centering
     \includegraphics[width=\textwidth]{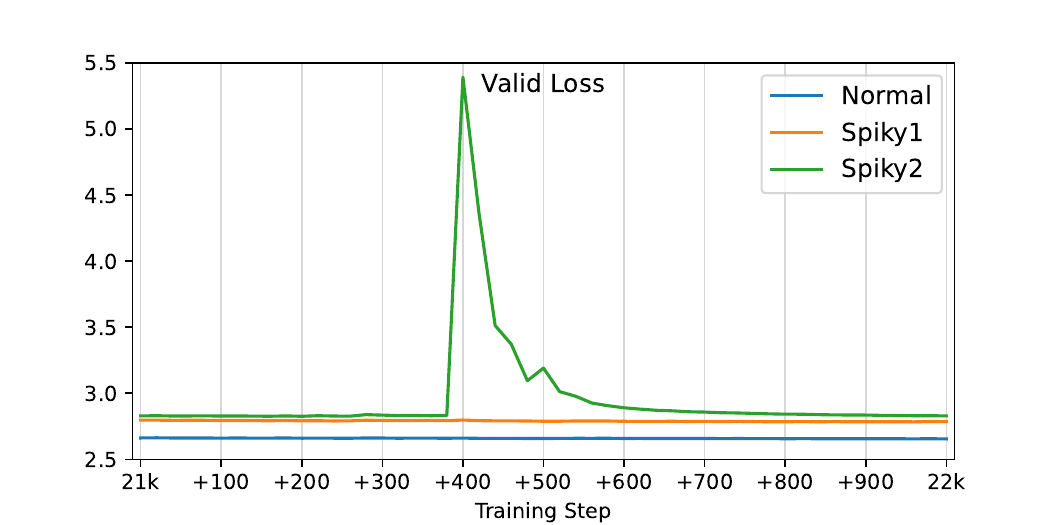}
     \caption{Validation loss from 21k to 22k.}
     \label{fig:spiky_loss_valid_partial}
 \end{subfigure}
 \end{minipage}
\caption{Training of a Transformer LM. Comparing a ``Normal'' training with two ``Spiky'' trainings. ``Normal'' and ``Spiky1'' differ in the $\beta_2$ hyperparameter of Adam only, while ``Spiky1'' and ``Spiky2'' differ in the random-seed of model initialization only. Training spikes can occur at different steps between ``Spiky1'' and ``Spiky2'', while ``normal fluctuations'' in training loss remain consistent.}
\label{fig:spiky_loss_all}
\end{figure}

Remarkably, even simple models like PLRF can exhibit training spikes. Figure~\ref{fig:plrf} compares various PLRF training trajectories -- with and without spikes -- across different optimizers, learning rates, and power-law exponents. Our findings highlight that the power-law exponent $\alpha$, which quantifies the ``heavy-tailedness'' of the data distribution, is critical; a larger $\alpha$ is a necessary condition for these spikes to occur. To our knowledge, this is the first study to explicitly link training spikes to heavy-tailed data distributions.

%It is already a surprise that such simple models like PLRF can exhibit training spikes. In Figure~\ref{fig:plrf}, we compare different training trajectories of PLRF models, with or without training spikes, across different optimizers, learning-rates, and power-law exponents of the data distribution. We note that the power-law exponent $\alpha$, which measures how heavy-tailed the data distribution is, plays a critical role -- with larger $\alpha$ (i.e., more heavy-tailed distribution) being a necessary condition for training spikes. To our knowledge, this is the first paper that explicitly points out the connection between training spikes and the heavy-tail distributions of data.

Our simulation framework is largely inspired by \citet{DBLP:conf/nips/LeeXSBNSP19}, utilizing a Taylor expansion of the DNN at training step $t$ to predict dynamics up to a future step $t+s$. We advance this approach by incorporating stochastic optimization algorithms like SGD and Adam, which are more representative of real-world practices than full-batch gradient descent. This shift requires a 2nd-order Taylor expansion to account for stochastic gradient noise; Fortunately, recent work by \citet{sgd_in_large,four_plus_three} provides the theoretical tools to precisely solve training dynamics under this setting. The 2nd-order analysis is essential; as we show in Figure~\ref{fig:dnn_simulation}, empirical evidence suggests that training spikes of DNNs likely originate from 2nd-order gradient noise within some layers.

%Our simulation framework is largely inspired by \citet{DBLP:conf/nips/LeeXSBNSP19}, of which the idea is to take Taylor expansion of the DNN at training step $t$, and use the information to predict training dynamics until a future step $t+s$. Instead of full-batch gradient descent, our new progress is to consider stochastic optimization algorithms such as SGD and Adam, which are more realistic in practice. In order to do so, one has to take Taylor expansion to the 2nd-order (instead of the 1st-), and calculate the stochastic noise of gradients from Hessian. Fortunately, recent studies \citep{sgd_in_large,four_plus_three} have developed substantial theoretical tools to solve the training dynamics precisely under this setting. The analysis of 2nd-order terms is essential, as in Figure~\ref{fig:dnn_simulation} we provide empirical evidence that training spikes probably originate from the 2nd-order gradient noise at some layer.

To summarize our contributions:

1. We developed a framework aligning DNN layers with simple Random Feature models to enable lightweight simulations of training dynamics (\S~\ref{sec:simulation_algorithm}). This approach facilitates the scientific study of elusive phenomena like training spikes in DNNs.

%We have developed a framework to align each layer of a DNN to simple Random Feature models, enabling lightweight simulations of the training dynamics. This is convenient for scientific study of DNNs, including the elusive phenomenon of training spikes.

2. While spiky behavior is often linked to the complex, non-convex nature of feature learning in DNNs \citep[as seen in the ``Edge-of-Stability'' regime, ][]{eos,adaptive_eos}, it is surprising that even convex PLRF models exhibit similar training spikes (\S~\ref{sec:background}).

%Related to the Edge-of-Stability \citep{eos,adaptive_eos} regime, spiky behavior has been considered as a demonstration of complicated, non-convex and non-linear feature-learning nature of DNNs. The finding of convex, PLRF models exhibiting training spikes, is a surprise.

3. We conducted a case-study where the Adam hyperparameter $\beta_2$ causes training spikes in a Transformer LM. Our simulation successfully predicts the spiky behavior for a given model checkpoint (\S~\ref{sec:simulation_transformer_LM}). We then analyze several training metrics in details, and offer optimization recipes such as an appropriate $\beta_2$ schedule, which can avoid spikes and empirically improve performance (\S~\ref{sec:appropriate_beta2}).

%We analyze a concrete example where the choice of hyperparameter $\beta_2$ of Adam is known to cause training spikes of a Transformer LM. We show that our simulation is able to predict spiky training from a given model checkpoint, in this specific case. We discuss details of the mechanism, pointing to conjectures beyond Random Feature models as well.

%4. Generally, we believe this work is a first step towards detailed understanding of the complicated training process of DNNs, which could potentially improve training and optimization in the future.

\section{Background}
\label{sec:background}

A Random Feature model assumes a non-trainable $\cscalar{D}\times\cscalar{V}$ matrix $\cmat{A}$ which maps a $\cscalar{V}$-dimensional random input $\cvec{x}$ to $\cscalar{D}$-dimensional features. Then, the features are fed to linear regression with a quadratic loss
\begin{equation}
\label{eq:rf}
\func{\ell}{\cvec{x};\vvec{w}}\eqdef \frac{c}{2}\big(\transpose{\vvec{w}}\cmat{A}\cvec{x}-\func{b}{\cvec{x}}\big)^2.
\end{equation}
Here, $\vvec{w}$ is the $\cscalar{D}$-dimensional trainable model weights, the label $\func{b}{\cvec{x}}$ a scalar function of $\cvec{x}$, and $c$ a constant. The constant matrix $\cmat{A}$ is usually set to random Gaussian if not else specified.

In addition, the Power-Law Random Feature model \citep[PLRF]{solvable_model_of_neural_scaling_laws,explain_neural_scaling_laws} assumes that $\cvec{x}=(i^{-\alpha}x_i)_{1\leq i\leq\cscalar{V}}$, where each $x_i$ is Gaussian i.i.d.~drawn from $\mathcal{N}(0,1)$, and $\alpha$ is the power-law exponent of input data. In this work (\S~\ref{sec:dnn2rf}), we will use $\cmat{A}\cvec{x}$ to approximate hidden vectors in a Deep Neural Network (DNN), so $\cscalar{D}$ is the model dimension of the DNN, while $\cscalar{V}$ is a very large ($\approx\infty$) number intended for a latent space that can linearize the sub-DNN which produces the hidden vector, as inspired by Neural Tangent Kernel \citep{DBLP:conf/nips/JacotHG18,DBLP:conf/icml/BordelonAP24}.

\begin{figure}[t]
 \centering
 \begin{subfigure}[b]{0.32\textwidth}
     \centering
     \includegraphics[width=\textwidth]{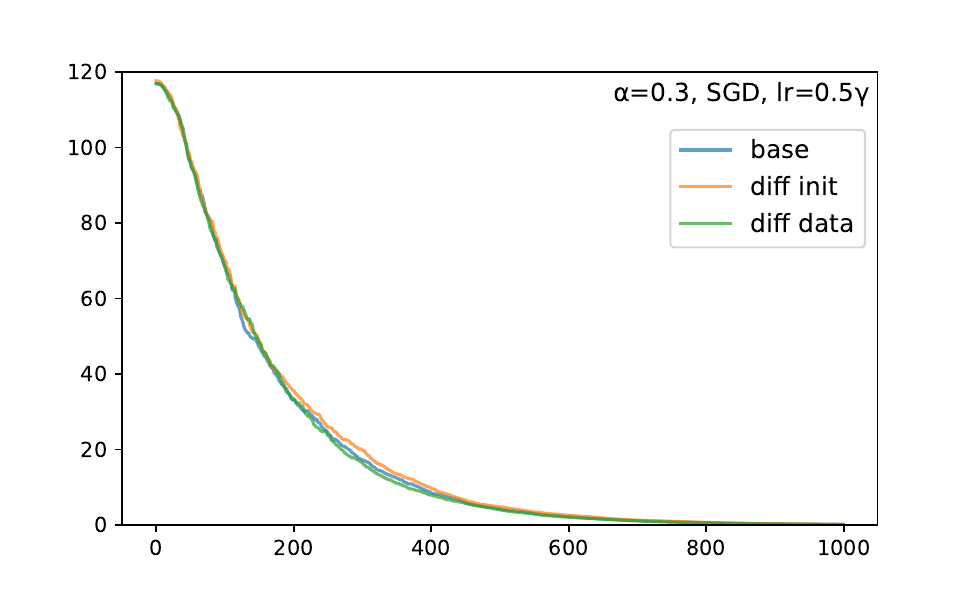}
     %\caption{(a)}
     %\label{fig:y equals x}
 \end{subfigure}
 \begin{subfigure}[b]{0.32\textwidth}
     \centering
     \includegraphics[width=\textwidth]{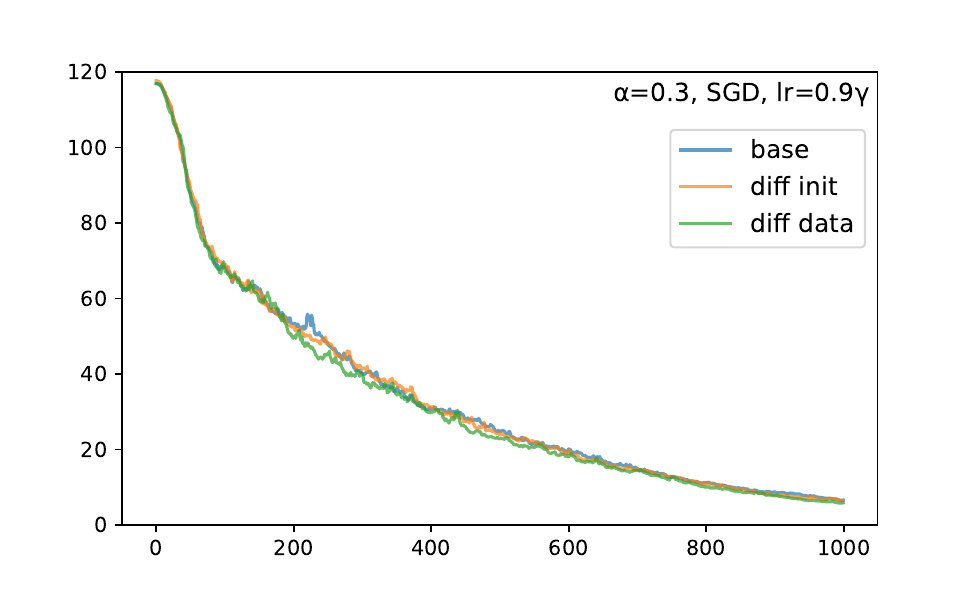}
     %\caption{(a)}
     %\label{fig:y equals x}
 \end{subfigure}
 \begin{subfigure}[b]{0.32\textwidth}
     \centering
     \includegraphics[width=\textwidth]{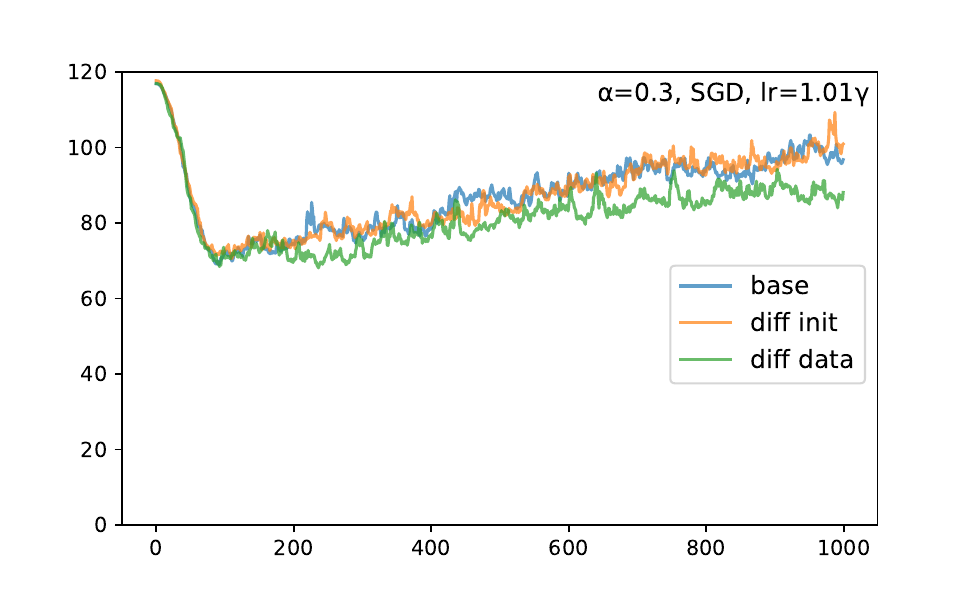}
     %\caption{(a)}
     %\label{fig:y equals x}
 \end{subfigure}\\
 \begin{subfigure}[b]{0.32\textwidth}
     \centering
     \includegraphics[width=\textwidth]{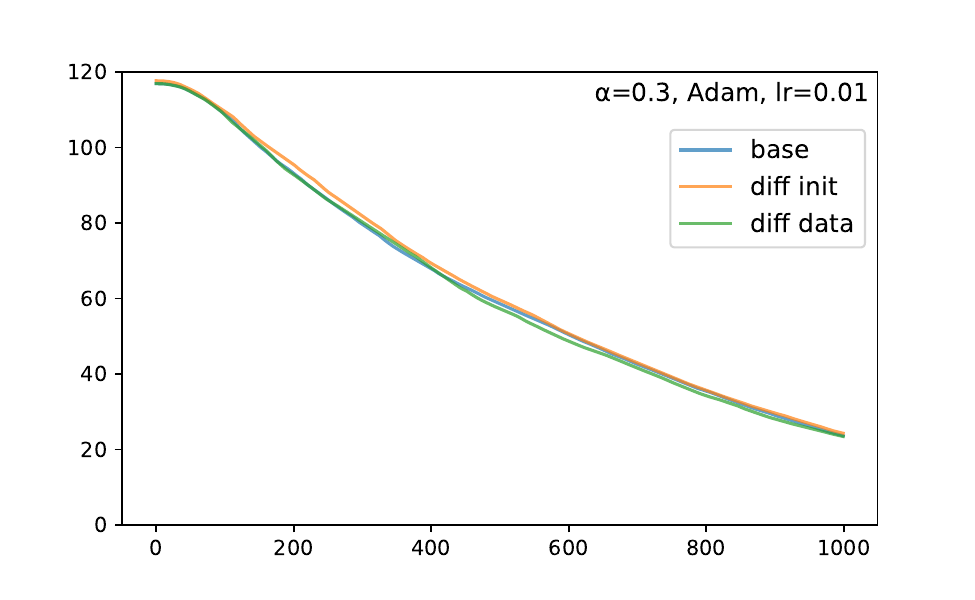}
     %\caption{(a)}
     %\label{fig:y equals x}
 \end{subfigure}
 \begin{subfigure}[b]{0.32\textwidth}
     \centering
     \includegraphics[width=\textwidth]{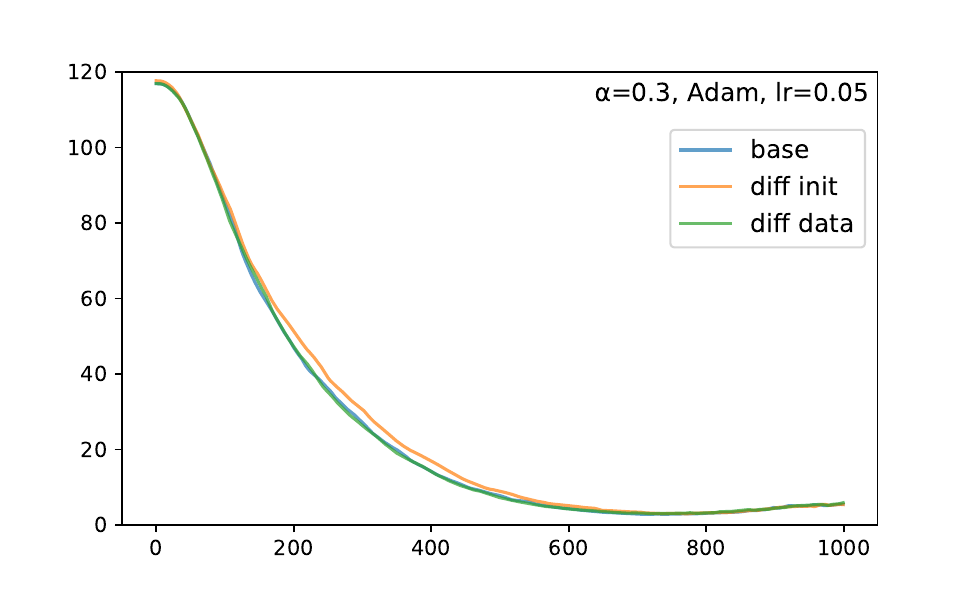}
     %\caption{(a)}
     %\label{fig:y equals x}
 \end{subfigure}
 \begin{subfigure}[b]{0.32\textwidth}
     \centering
     \includegraphics[width=\textwidth]{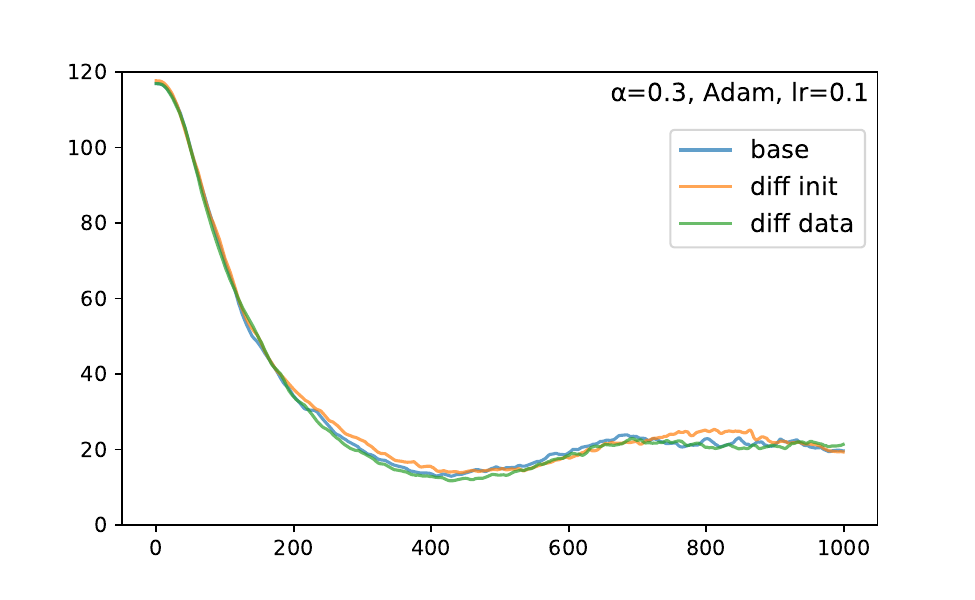}
     %\caption{(a)}
     %\label{fig:y equals x}
 \end{subfigure}\\
 \begin{subfigure}[b]{0.32\textwidth}
     \centering
     \includegraphics[width=\textwidth]{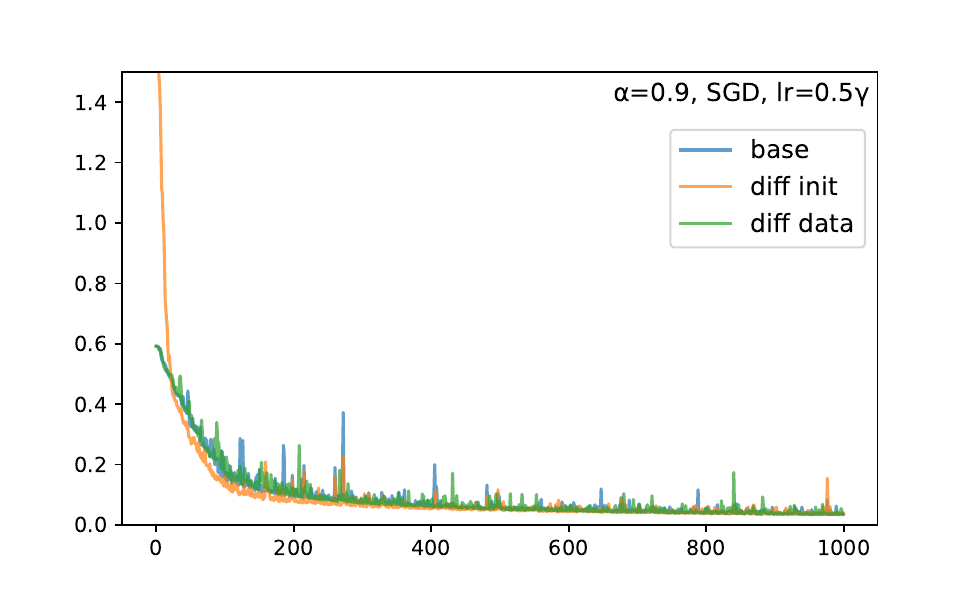}
     %\caption{(a)}
     %\label{fig:y equals x}
 \end{subfigure}
 \begin{subfigure}[b]{0.32\textwidth}
     \centering
     \includegraphics[width=\textwidth]{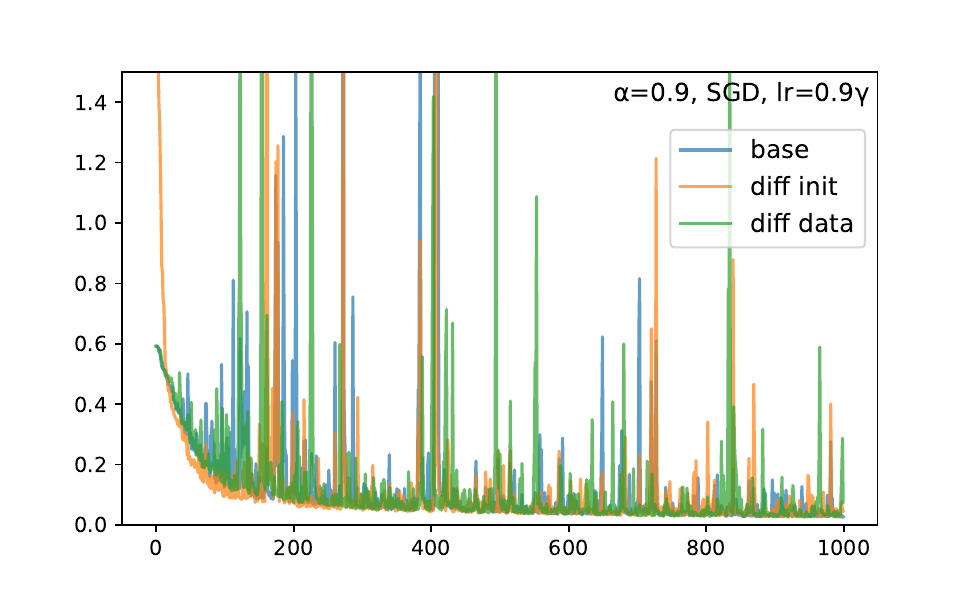}
     %\caption{(a)}
     %\label{fig:y equals x}
 \end{subfigure}
 \begin{subfigure}[b]{0.32\textwidth}
     \centering
     \includegraphics[width=\textwidth]{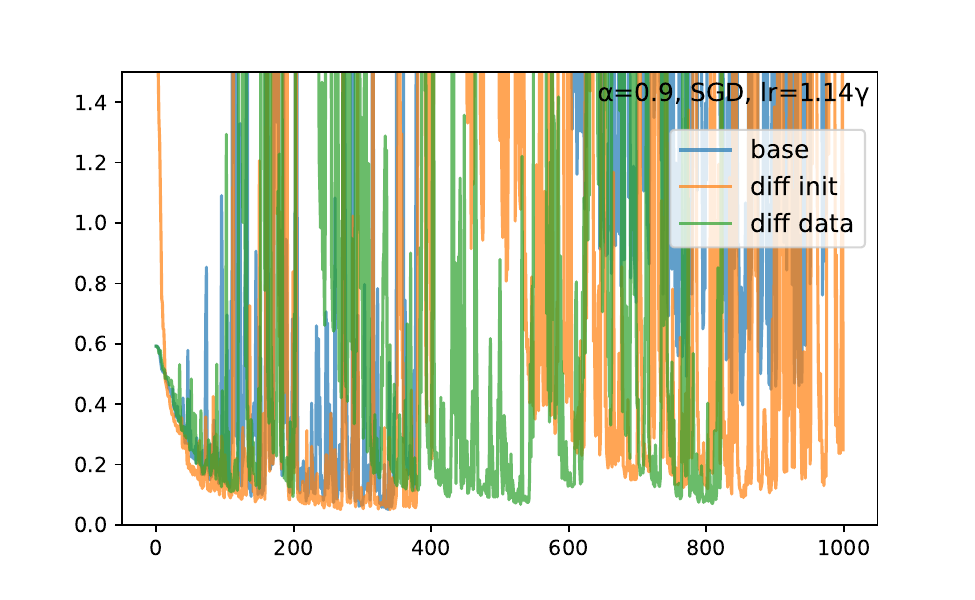}
     %\caption{(a)}
     %\label{fig:y equals x}
 \end{subfigure}\\
 \begin{subfigure}[b]{0.32\textwidth}
     \centering
     \includegraphics[width=\textwidth]{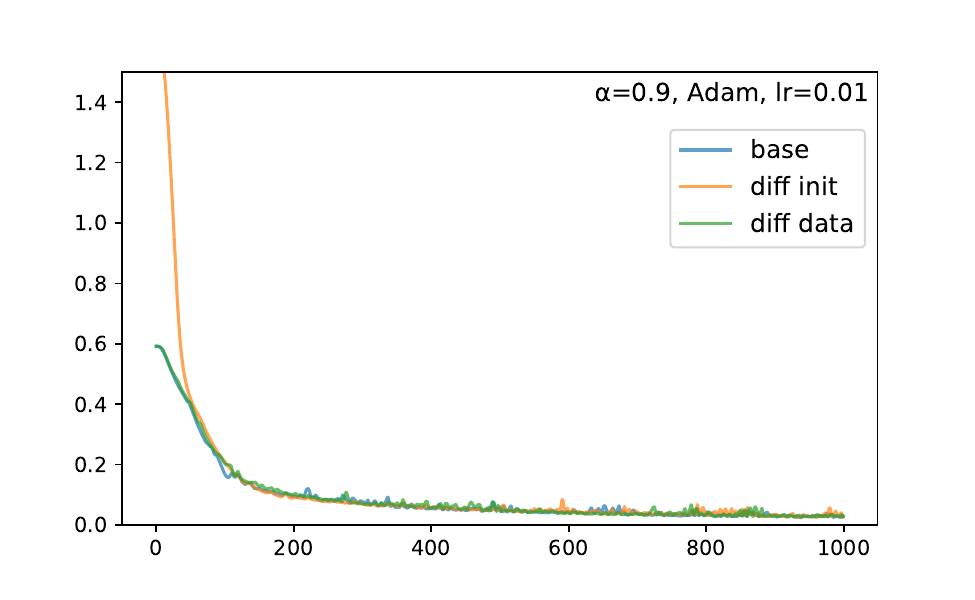}
     %\caption{(a)}
     %\label{fig:y equals x}
 \end{subfigure}
 \begin{subfigure}[b]{0.32\textwidth}
     \centering
     \includegraphics[width=\textwidth]{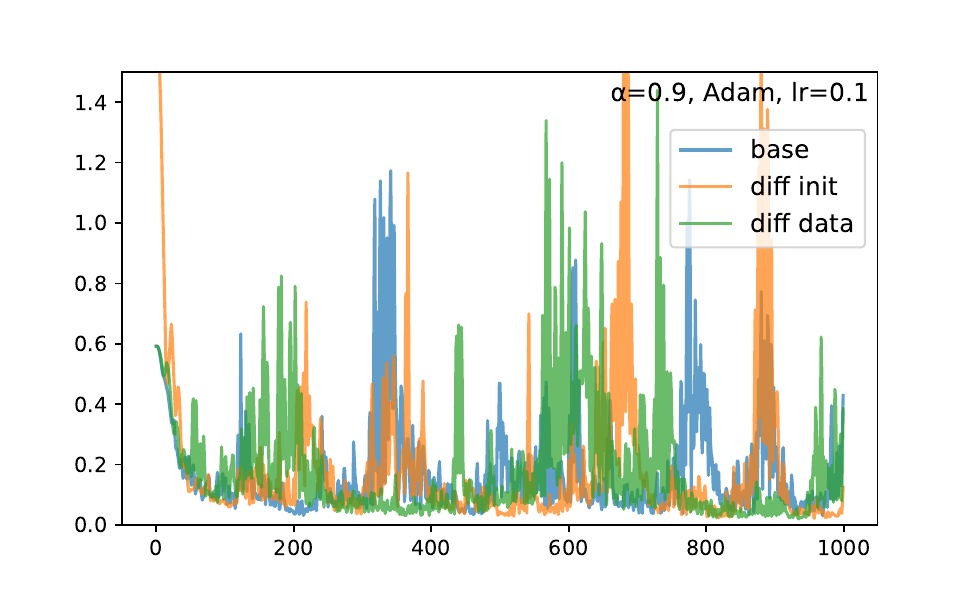}
     %\caption{(a)}
     %\label{fig:y equals x}
 \end{subfigure}
 \begin{subfigure}[b]{0.32\textwidth}
     \centering
     \includegraphics[width=\textwidth]{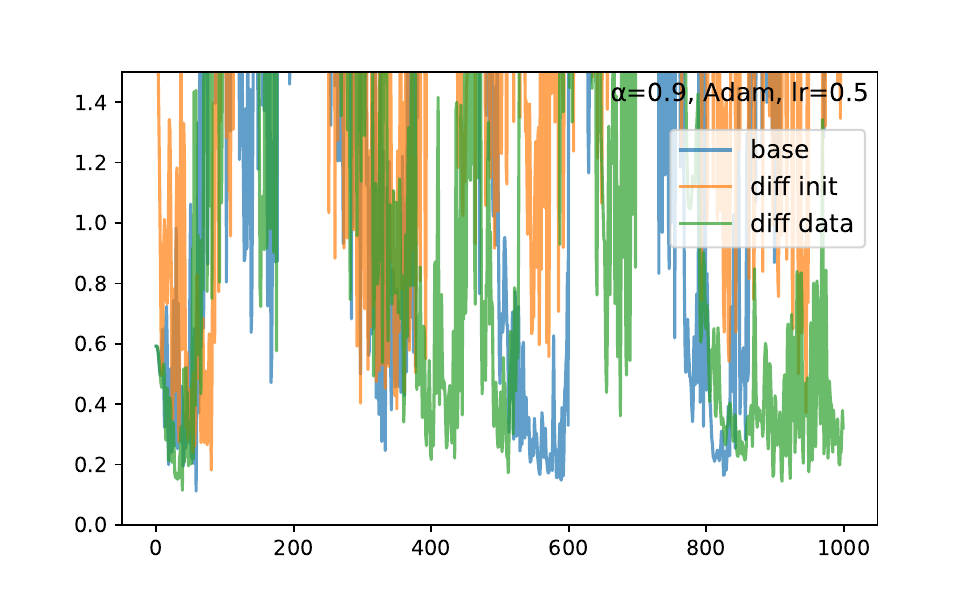}
     %\caption{(a)}
     %\label{fig:y equals x}
 \end{subfigure}
\caption{The validation loss of PLRF, from training step 0 to 1k. We compare $\alpha\in\{0.3,0.9\}$, for SGD and Adam, with different learning-rates. Each plot shows 3 random seeds (to distinguish initialization and data-shuffle separately), and learning-rates are chosen such that training converges, nearly diverges, or diverges, respectively. For SGD, the required learning-rates are calculated from the threshold $\gamma$ given in formula \eqref{eq:gamma}. For Adam, the learning-rates are found by manual grid search.}
\label{fig:plrf}
\end{figure}

When Stochastic Gradient Descent (SGD) is applied to the optimization of model \eqref{eq:rf}, the training dynamics can be analytically solved \citep{sgd_in_large,four_plus_three}. Moreover, since the power-law data distribution in PLRF implies neural scaling laws \citep{solvable_model_of_neural_scaling_laws,explain_neural_scaling_laws}, solving the dynamics allows for numerical predictions of \emph{compute-optimal} scaling \citep{four_plus_three}. In this study, we further demonstrate that a large $\alpha$ can cause training spikes as well.

%Moreover, as \citet{solvable_model_of_neural_scaling_laws,explain_neural_scaling_laws} point out that the power-law distribution of input data in PLRF can lead to neural scaling laws, solving the training dynamics of PLRF enables numerical predictions of \emph{compute-optimal} neural scaling laws \citep{four_plus_three}. In this work, we further demonstrate that large $\alpha$ can cause training spikes as well.

In Figure~\ref{fig:plrf}, we compare the training dynamics of PLRF for $\alpha\in\{0.3,0.9\}$. We set $\cscalar{D}=1024$, $\cscalar{V}=20000$, $c=1$, and apply SGD and Adam across different learning-rates. The batch-size $\cscalar{B}=8$. When $\alpha=0.3$, as the learning-rate increases, validation loss diverges without showing any spikes. In contrast, when $\alpha=0.9$, both optimizers exhibit training spikes before divergence. This suggests two conditions for training spikes in PLRF: (1) heavy-tailed data (large $\alpha$) and (2) near divergence (large learning-rate). We provide the Python code for these experiments in supplementary materials.

%The spiky behavior of PLRF training is surprising for two reasons. First, while SGD-based training of Random Feature models is a random process influenced by data shuffling, Paquette et al. (2021) demonstrated that validation loss typically converges to a smooth, deterministic curve—a phenomenon known as the "Concentration of SGD"—when D and V are sufficiently large. Training spikes in PLRF seemingly contradict this concentration effect. However, for PLRF with α ≥ 0.5, while the expected validation loss remains smooth, convergence in the high-dimensional limit can no longer be proven (see Section 3.4 and Appendix E). Thus, while the scaling analysis of Paquette et al. (2024) remains valid, it must be noted that these results describe average-case behavior; at α ≥ 0.5, training can still exhibit stochastic spikes that temporarily deviate from predicted dynamics, even at massive scales.

The spiky behavior of PLRF is surprising for two reasons. First, while SGD-based training is a random process influenced by data shuffling, \citet{sgd_in_large} proves that the validation loss of Random Feature models typically converges to a smooth, \emph{deterministic} curve -- a phenomenon known as the ``Concentration of SGD'' -- when $\cscalar{D}$ and $\cscalar{V}$ are sufficiently large. Random training spikes of PLRF seemingly contradict this concentration effect -- As it turns out, for PLRF with $\alpha>0.5$, while the \emph{expected} validation loss remains smooth, and the scaling analysis of \citet{four_plus_three,ferbach2025dimensionadapted} still valid, it must be noted that these results describe \emph{average-case} behavior; But one can no longer prove the \emph{convergence} to expectation (see \S~\ref{sec:sketch} and \S~\ref{sec:concentration_alpha}). In other words, at $\alpha>0.5$, training can exhibit random spikes that temporarily deviate from the expectation, even at massive scales.

%although the training of Random Feature models by SGD is a random process depending on the data-shuffle, \citet{sgd_in_large} proves that the validation loss will (under certain conditions) converge to a \emph{deterministic} and smooth curve when $\cscalar{D}$ and $\cscalar{V}$ are sufficiently large (which is called the \emph{Concentration of SGD}). This concentration effect contradicts the training spikes observed in PLRF. It turns out that for PLRF with $\alpha\geq0.5$, the \emph{expectation} of the validation loss is still a smooth curve under SGD, but one can no longer prove the \emph{convergence} at the high-dimension limit (see Section~\ref{sec:sketch} and Appendix~\ref{sec:concentration_alpha} for more details). The scaling analysis in \citet{four_plus_three} still holds, but an extra note should emphasize that the results are about the \emph{average-case} behavior; For $\alpha\geq0.5$, the training can exhibit random spikes in validation loss that temporarily deviate from the predicted dynamics, even at the largest scales.

Second, while similar spiky behavior has been observed in DNNs under \emph{full-batch gradient descent}, it was previously considered impossible for quadratic models \citep{eos,DBLP:conf/iclr/DamianNL23,central_flow}. This is because full-batch gradient descent on quadratic loss has straightforward dynamics: It diverges if the learning rate exceeds $2/\lambda_\textrm{max}$ (where $\lambda_\textrm{max}$ is the largest eigenvalue of the Hessian). Typically, only DNNs with 3rd-order terms exhibit spikes and the subsequent stabilization characteristic of the so-called ``Edge-of-Stability'' regime, when the learning rate nears this threshold.

%Secondly, under \emph{full-batch gradient descent} training, similar spiky behavior has been observed with DNNs, but considered impossible for quadratic models \citep{eos,DBLP:conf/iclr/DamianNL23,central_flow}. This is because the training dynamics of full-batch gradient descent is simple for quadratic loss: It will diverge once the learning-rate exceeds $2/\lambda_\textrm{max}$, where $\lambda_\textrm{max}$ is the largest eigenvalue of the Hessian. Only for DNNs with 3rd-order terms one can observe the spikes and stabilizing behavior when the learning-rate is close to $2/\lambda_\textrm{max}$ -- the so-called Edge-of-Stability regime.

However, this theory does not translate directly to stochastic optimization. Due to stochastic gradient noise captured by Hessian, training spikes can occur in quadratic models when the input data is heavy-tailed; 3rd-order terms are not necessary. Additionally, the learning-rate threshold $\gamma$ of divergence is not $2/\lambda_\textrm{max}$ for SGD; instead, when the batch-size $\cscalar{B}$ is small ($\cscalar{B}\ll\sqrt{\cscalar{D}}$), $\gamma$ is determined by the following equation \citep[see][C.2, or \S~\ref{sec:derive_gamma} of this paper]{four_plus_three}:
\begin{equation}
\label{eq:gamma}
\frac{1}{\cscalar{B}}\sum_i\frac{\gamma\lambda_i}{2-\gamma\lambda_i}=1,
\end{equation}
where $\{\lambda_i\}$ is the eigenvalues of the Hessian. In particular, $\gamma<\min\big(2\cscalar{B}/\sum_i\lambda_i,\,2/\lambda_\textrm{max}\big)$.
%rough approximation is $\gamma\approx2\cscalar{B}/\sum_i\lambda_i$. In particular, $\gamma$ depends on all the eigenvalues. Furthermore, when the learning-rate is close to $\gamma$, stochastic gradient noise captured by the Hessian can already produce training spikes, as demonstrated by PLRF models (see ...). No longer necessary to introduce 3rd-order terms.

\section{Theory}

Having demonstrated the rich dynamics of PLRF including training spikes, in this section we introduce a theoretical framework that aligns each DNN layer with a Random Feature model.

%Having seen the substantially rich dynamics of PLRF including training spikes, this section presents a theoretical framework which aligns each layer of a DNN with a Random Feature model.

%a variety of training spikes that look similar to what one observes in practice, we present a theory that explains why such spikes exist. The theoretical settings is not limited to PLRF models; in fact, we provide a formalization that can be applied to every layer in a Deep Neural Network (DNN).

\subsection{Layer-wise Quadratic Loss}

Let $\mathcal{W}$ be the model parameters of a DNN. Given a loss function $\func{\ell}{\mathfrak{x};\mathcal{W}}$, where $\mathfrak{x}$ is a data point, the DNN is trained by minimizing the loss over a data distribution $\mathfrak{x}\sim\mathcal{D}$. We are interested in the dynamics of the validation loss $\mathbb{E}_{\mathfrak{x}\sim\mathcal{D}}\big(\func{\ell}{\mathfrak{x};\mathcal{W}_t}\big)$, where $\mathcal{W}_t$ is the model checkpoint at training step $t$. 

%In addition, we are interested in the average squared norm of gradients of each layer $\vmat{W}\subset\mathcal{W}$, defined by
% \begin{equation}\label{eq:def_varphi}
% \varphi^{\vmat{W}}_t\eqdef\mathbb{E}_{\mathfrak{x}\sim\mathcal{D}}\,\Big\lVert\func{\frac{\partial\ell}{\partial\vmat{W}}}{\mathfrak{x};\mathcal{W}_t}\Big\rVert^2.
% \end{equation}

%In this paper, instead of directly investigating the above, we will take layer-wise 2nd-order approximations of the dynamics as surrogate metrics.
\begin{definition}\label{def:delta} Instead of directly investigating $\mathbb{E}_{\mathfrak{x}\sim\mathcal{D}}\,\func{\ell}{\mathfrak{x};\mathcal{W}_t}$, we take its Taylor expansion to the 2nd order at training step $t$, for a fixed layer $\vvec{w}\subset\mathcal{W}$:
\begin{equation*}%\label{eq:delta}
\delta^{\vvec{w}}_{t,s}\eqdef\mathbb{E}_{\mathfrak{x}\sim\mathcal{D}}\Big(\func{\frac{\partial\ell}{\partial\vvec{w}}}{\mathfrak{x};\mathcal{W}_{t}}\Big)(\vvec{w}_{t+s}\!-\!\vvec{w}_{t})+\frac{1}{2}\transpose{(\vvec{w}_{t+s}\!-\!\vvec{w}_{t})}\mathbb{E}_{\mathfrak{x}\sim\mathcal{D}}\Big(\func{\frac{\partial^2\ell}{\transpose{\partial\vvec{w}}\partial\vvec{w}}}{\mathfrak{x};\mathcal{W}_{t}}\Big)(\vvec{w}_{t+s}\!-\!\vvec{w}_{t}).
\end{equation*}
This is the approximate ``amount of change'' %(to the 2nd order) 
contributed by $\vvec{w}$ to the validation loss, when $\vvec{w}$ is trained to step $t+s$ while other parameters are fixed at step $t$. One subtle detail is that $\vvec{w}_{t+s}$ evolves by the \emph{original} updates calculated from the gradients of $\func{\ell}{\mathfrak{x}_{t+s};\mathcal{W}_{t+s}}$ (where $\mathfrak{x}_{t+s}\sim\mathcal{D}$ is drawn from the training data-shuffle), without freezing any model parameters. We only freeze model parameters %and use validation data 
when calculating the evaluation metric $\delta^{\vvec{w}}_{t,s}$.
\end{definition}

The metric $\delta^{\vvec{w}}_{t,s}$ is a quadratic function of $\vvec{w}_{t+s}$, which can be written as ``$\func{L}{\vvec{w}_{t+s}}+\mbox{Constant}$'' where
\begin{equation}\label{eq:layer_quadratic_loss}
\func{L}{\vvec{w}}\eqdef\frac{1}{2}\transpose{(\vvec{w}-\cvec{w}^\ast)}\mathbb{E}_{\mathfrak{x}\sim\mathcal{D}}\Big(\func{\frac{\partial^2\ell}{\transpose{\partial\vvec{w}}\partial\vvec{w}}}{\mathfrak{x};\mathcal{W}_{t}}\Big)(\vvec{w}-\cvec{w}^\ast).
\end{equation}
Here, $\cvec{w}^\ast$ is calculated such that the gradient of $L$ can match that of the validation loss at $\vvec{w}_t$:
\[
\func{\frac{\partial L}{\partial\vvec{w}}}{\vvec{w}_t}=\mathbb{E}_{\mathfrak{x}\sim\mathcal{D}}\Big(\func{\frac{\partial\ell}{\partial\vvec{w}}}{\mathfrak{x};\mathcal{W}_{t}}\Big).
\]
Our goal is to simulate the dynamics of $L$, the \emph{layer-wise quadratic loss}. In this paper, we rigorously define ``\emph{layers}'' of a DNN as below.

\begin{definition}\label{def:layer}
A \emph{layer} is a subset $\vmat{W}\subset\mathcal{W}$ of model parameters such that:
\begin{enumerate}[leftmargin=3.5em, labelwidth=2.5em, labelsep=1em]
\item[(1)~] $\vmat{W}$ acts as a matrix multiplication, i.e., it takes a vector $\cvec{h}$ as input, and returns 
$\vvec{y}\eqdef\vmat{W}\cvec{h}$
as the layer output.
\item[(2a)] The forward pass to the input $\cvec{h}$ does not depend on $\vmat{W}$. In other words, $\cvec{h}$ can be written as 
$\cvec{h}=\func{\cvec{h}}{\mathfrak{x};\mathcal{W}^\mathrm{F}}$ where $\mathcal{W}^\mathrm{F}\cap\vmat{W}=\emptyset$.
\item[(2b)] The backward pass to the output gradient $\frac{\partial\ell}{\partial\vvec{y}}$ does not depend on $\vmat{W}$. In other words, $\frac{\partial\ell}{\partial\vvec{y}}$ can be written as $\frac{\partial\ell}{\partial\vvec{y}}=\func{\frac{\partial\ell}{\partial\vvec{y}}}{\mathfrak{x};\vvec{y},\mathcal{W}^\mathrm{B}}$, where $\mathcal{W}^\mathrm{B}\cap\vmat{W}=\emptyset$.
\end{enumerate}
\end{definition}

Intuitively, Definition~\ref{def:layer} ensures that $\vmat{W}$ acts as a ``standalone linear unit'' in the DNN. In practice, most DNNs can be decomposed into layers under this definition\footnote{In certain model designs, ``shared weights'' can break the conditions. For example, some LMs reuse the input embedding layer as the output logits head, which breaks condition (2a) and (2b) because the same weights are used at different sites of the model. We will leave such cases to future work.}. Because $\vmat{W}$ is generally a \emph{matrix}, we flatten it to a \emph{vector} $\vvec{w}$ in Definition~\ref{def:delta}.

We can calculate the \emph{per-example} gradient of $\vmat{W}$ from Definition~\ref{def:layer}: %, the (per-example) gradient of $\vmat{W}$ can be calculated as
\begin{equation}
\func{\frac{\partial\ell}{\partial\vmat{W}}}{\mathfrak{x};\mathcal{W}}=\func{\cvec{h}}{\mathfrak{x};\mathcal{W}^\mathrm{F}}\func{\frac{\partial\ell}{\partial\vvec{y}}}{\mathfrak{x};\vvec{y},\mathcal{W}^\mathrm{B}}.
\end{equation}
In the RHS above, the only component that depends on $\vmat{W}$ is $\vvec{y}$. Therefore, if we focus on the evolution of $\vmat{W}_{t+s}$, everything except the linear regression $\vvec{y}_{t+s}$ can be treated as ``data given''. This is the main idea for the next section.

\subsection{Per-Layer Approximation by Random Feature Models}
\label{sec:dnn2rf}

Let $\vmat{W}$ be a $\cscalar{C}\times\cscalar{D}$ matrix. The next theorem suggests that one can approximate the per-example gradients of $\vmat{W}$, as if they come from a (generalized) Random Feature Model.
\begin{theorem}\label{thm:dnn2rf}
Fix a training step $t$. There exists a map $(\mathfrak{x},s)\mapsto\cvec{x}_{s}$, such that for any training example $\mathfrak{x}$ and step $s\geq 0$, a $\cscalar{V}$-dimensional vector $\cvec{x}_{s}$ %(where $\cscalar{V}$ can be very large) 
satisfies the following:
%For any training example $\mathfrak{x}$ at step $\tau$ (where $t\leq\tau\leq t+s$), there is a latent vector representation $\cvec{x}$, of $\cscalar{V}$ dimension (where $\cscalar{V}$ can be very large), such that:
\begin{enumerate}[leftmargin=3.5em, labelwidth=2.5em, labelsep=1em]
% \item[(a)] The $\cscalar{D}$-dimensional vector $\cvec{h}_{t+s}\eqdef\func{\cvec{h}}{\mathfrak{x};\mathcal{W}^\mathrm{F}_{t+s}}$ can be approximated by
% \[
% \cvec{h}_{t+s}\approx\cmat{A}\cvec{x}_{s}=\cvec{h}_{t}+\Big(\func{\frac{\partial\cvec{h}}{\partial\mathcal{W}^\mathrm{F}}}{\mathfrak{x};\mathcal{W}^\mathrm{F}_{t}}\Big)(\mathcal{W}^\mathrm{F}_{t+s}-\mathcal{W}^\mathrm{F}_{t}),
% \]
% where $\cmat{A}$ is a constant $\cscalar{D}\times\cscalar{V}$ matrix that does not depend on $\mathfrak{x}$ or $s$.
\item[(a)] The $\cscalar{D}$-dimensional vector $\cvec{h}_{t+s}\eqdef\func{\cvec{h}}{\mathfrak{x};\mathcal{W}^\mathrm{F}_{t+s}}$ can be approximated by $\cmat{A}\cvec{x}_{s}$,
where $\cmat{A}$ is a constant $\cscalar{D}\times\cscalar{V}$ matrix that does not depend on $\mathfrak{x}$ or $s$.
\item[(b)] The $\cscalar{C}$-dimensional vector $\nabla\vvec{y}_{t+s}\eqdef\func{\frac{\partial\ell}{\transpose{\partial\vvec{y}}}}{\mathfrak{x};\vvec{y}_{t+s},\mathcal{W}^\mathrm{B}_{t+s}}$ (where $\vvec{y}_{t+s}\eqdef\vmat{W}_{t+s}\cvec{h}_{t+s}$), can be approximated by $\vmat{G}_{t+s}\cvec{x}_{s}$, where $\vmat{G}_{t+s}$ is an affine function of $\vmat{W}_{t+s}$ and a $\cscalar{C}\times\cscalar{V}$ matrix that does not depend on $\mathfrak{x}$.
% \item[(b)] The $\cscalar{C}$-dimensional vector $\nabla\vvec{y}_{t+s}\eqdef\func{\frac{\partial\ell}{\transpose{\partial\vvec{y}}}}{\mathfrak{x};\vvec{y}_{t+s},\mathcal{W}^\mathrm{B}_{t+s}}$, where $\vvec{y}_{t+s}\eqdef\vmat{W}_{t+s}\cvec{h}_{t+s}$, can be approximated by
% \begin{multline*}
% \nabla\vvec{y}_{t+s}\approx\vmat{G}_{t+s}\cvec{x}_{s}\approx\nabla\vvec{y}_{t}+\Big(\func{\frac{\partial^2\ell}{\transpose{\partial\vvec{y}}\partial\vvec{y}}}{\mathfrak{x};\vvec{y}_{t},\mathcal{W}^\mathrm{B}_{t}}\Big)(\vvec{y}_{t+s}-\vvec{y}_{t})\\
% +\Big(\func{\frac{\partial^2\ell}{\transpose{\partial\vvec{y}}\partial\mathcal{W}^\mathrm{B}}}{\mathfrak{x};\vvec{y}_{t},\mathcal{W}^\mathrm{B}_{t}}\Big)(\mathcal{W}^\mathrm{B}_{t+s}-\mathcal{W}^\mathrm{B}_{t}),
% \end{multline*}
% where $\vmat{G}_{t+s}$ is an affine function of $\vmat{W}_{t+s}$, and a $\cscalar{C}\times\cscalar{V}$ matrix that does not depend on $\mathfrak{x}$. %, and %$\vmat{G}_{t+s}=\func{\vmat{G}}{\vmat{W}_{t+s}}$ is 
%.
%\item[(c)] Furthermore, for any training example $\mathfrak{x}$ we have exactly
% \[
% \vmat{G}_{t+s}\cvec{x}_{0}\transpose{\cvec{x}_{0}}\transpose{\cmat{A}}=\func{\frac{\partial^2\ell}{\partial\vmat{W}\partial\vmat{W}}}{\mathfrak{x};\mathcal{W}_{t}}(\vmat{W}_{t+s}-\vmat{W}_{t})+\func{\frac{\partial\ell}{\partial\vmat{W}}}{\mathfrak{x};\mathcal{W}_{t}},
% \]
% which is the gradient of $\vmat{W}$ as obtained from the 2nd-order Taylor expansion at $\mathcal{W}_{t}$.
\end{enumerate}
\end{theorem}

%At first glance, such linear representations as required by Theorem~\ref{thm:dnn2rf} may seem unreasonably strong, because DNNs are highly non-linear. However, we can naturally ``linearize'' the DNN as inspired by the Neural Tangent Kernel work \citep{DBLP:conf/nips/LeeXSBNSP19}.

At first glance, the \emph{linear} approximation as required by Theorem~\ref{thm:dnn2rf} may seem unreasonably strong, because DNNs are highly non-linear -- which is why our proof (given in \S~\ref{sec:proof_of_thm}) uses a latent space of extremely high dimensionality ($\cscalar{V}\approx\mbox{\#\,training-examples}\times\mbox{\#\,model-parameters}$) to ``linearize'' the DNN, as inspired by Neural Tangent Kernel. Therefore, the 
%A proof of Theorem~\ref{thm:dnn2rf} is given in Appendix~\ref{sec:proof_of_thm}. The 
vector $\cvec{x}_{s}$ and matrices $\cmat{A}$, $\vmat{G}_{t+s}$ are %extremely high-dimensional (with $\cscalar{V}\approx\mbox{\#\,training-examples}\times\mbox{\#\,model-parameters}$), and 
infeasible to calculate in practice -- However, they provide the necessary framework to apply powerful assumptions (e.g.~the independence of entries of $\cvec{x}_{s}$ under suitable change of basis) and theoretical tools (e.g.~SVD decomposition of $\cmat{A}$), enabling predictions that can be empirically verified.

%they bring the necessary language to enable powerful assumptions (e.g.~the independence of entries of $\cvec{x}_{s}$ under suitable change of basis) and theoretical tools (e.g.~SVD decomposition of $\cmat{A}$) that can make predictions to be empirically verified.

Under stochastic optimization, $\mathfrak{x}\sim\mathcal{D}$ is randomly drawn from the training data, so its corresponding representation $\cvec{x}_s$ is a random vector. Although the distribution of $\mathfrak{x}$ does not depend on the training step, the construction of $\cvec{x}_{s}$ in Theorem~\ref{thm:dnn2rf} does. In other words, there is generally \emph{distribution drift} of $\cvec{x}_{s}$ across different $s$. However, in this work, we impose a strong assumption that the distribution drift is negligible when $s$ is small. In fact, we assume that the distribution of $\cvec{x}_{s}$ is explicitly given as i.i.d.~Gaussian, so that we can align with a PLRF model:
\begin{assumption}\label{ass:iid}
One can replace the $\cvec{x}_{s}$ in Theorem~\ref{thm:dnn2rf} by a suitable linear transformation $\cmat{T}\cvec{x}$, where $\cmat{T}$ does not depend on $\mathfrak{x}$ or $s$, and the entries of $\cvec{x}$ are i.i.d.~Gaussian of mean $0$ and variance $1$.
\end{assumption}
When Assumption~\ref{ass:iid} is applied, we can omit the subscript of $\cvec{x}_{s}$ and absorb the linear transformation $\cmat{T}$ into $\cmat{A}$ and $\vmat{G}_{t+s}$. So $\cvec{h}_{t+s}\approx\cmat{A}\cvec{x}$ and $\nabla\vvec{y}_{t+s}\approx\vmat{G}_{t+s}\cvec{x}$, with the entries of $\cvec{x}$ being i.i.d.~Gaussian of mean $0$ and variance $1$. This assumption enables explicit calculation of the covariance of the per-example gradient, $\nabla\vmat{W}_{t+s}\eqdef\func{\frac{\partial\ell}{\transpose{\partial\vmat{W}}}}{\mathfrak{x};\mathcal{W}_{t+s}}\approx\vmat{G}_{t+s}\cvec{x}\,\transpose{\cvec{x}}\transpose{\cmat{A}}$, as below.

\begin{prediction}\label{pred:grad_norm}
Using Assumption~\ref{ass:iid}, and assume that $\cscalar{D}$ is sufficiently large. Let $\mathbb{E}(\cvec{h}_{t}\transpose{\cvec{h}}_{t})$ be the correlation matrix of input hidden vectors, where $\mathbb{E}$ is averaging over the training data $\mathfrak{x}\sim\mathcal{D}$. Let $\mathbb{E}(\cvec{h}_{t}\transpose{\cvec{h}}_{t})=\cmat{U}\diag(\vvec{\lambda})\transpose{\cmat{U}}$ be the eigen decomposition. Then,
\begin{equation}
\label{eq:egg}
\mathbb{E}(\vvec{g}_i\transpose{\vvec{g}_j})\approx\begin{cases}
\lambda_i\,\mathbb{E}(\nabla\vvec{y}_{t+s}\nabla\transpose{\vvec{y}}_{t+s})\quad&\mbox{if $i=j$} \\
\cmat{0}\quad&\mbox{if $i\neq j$}
\end{cases}
\end{equation}
where $\vvec{g}_i$, $\vvec{g}_j$ are the $i,j$-th column vectors of the $\cscalar{C}\times\cscalar{D}$ matrix $(\nabla\vmat{W}_{t+s})\cmat{U}$. In particular, we have
\begin{equation}\label{eq:grad_norm}
\mathbb{E}(\lVert\nabla\vmat{W}_{t+s}\lVert^2)\approx\sum_{i=1}^\cscalar{D}\lambda_i\,\mathbb{E}(\lVert\nabla\vvec{y}_{t+s}\rVert^2)
\end{equation}
by taking the trace of \eqref{eq:egg}.
\end{prediction}
We theoretically prove Prediction~\ref{pred:grad_norm} in \S~\ref{sec:proof_of_pred} and empirically verify in \S~\ref{sec:verification_of_pred}. This is a non-trivial result relying on the high-dimensionality of $\cscalar{D}$: Intuitively, it claims that for $\nabla\vmat{W}_{t+s}=(\nabla\vvec{y}_{t+s})\,\transpose{\cvec{h}}_{t+s}$, one can calculate the expected gradient norm $\mathbb{E}(\lVert\nabla\vmat{W}_{t+s}\lVert^2)$ as if $\lVert\nabla\vvec{y}_{t+s}\lVert$ and $\lVert\cvec{h}_{t+s}\rVert$ are independent -- Of course, $\nabla\vvec{y}_{t+s}\approx\vmat{G}_{t+s}\cvec{x}$ and $\cvec{h}_{t+s}\approx\cmat{A}\cvec{x}$ are \emph{not} independent because they share a common factor $\cvec{x}$ -- But the ``dependent part'' turns out to be negligible when $\cscalar{D}$ is large. In our case-study of a Transformer LM, we show that the empirical error of \eqref{eq:grad_norm} is mostly within $5\%$, even though the two sides simultaneously vary multiple magnitudes across different layers and training steps.

\subsection{Simulation Algorithm}
\label{sec:simulation_algorithm}

From the layer-wise quadratic loss $\func{L}{\vvec{w}}$ one can only obtain the \emph{average} gradient (which matches the approximate average gradient of the DNN, as calculated from the layer-wise Taylor expansion to the 2nd order); But in order to simulate stochastic optimization, one has to synthesize \emph{per-example} gradients, or more generally the \emph{mini-batched} gradients of the DNN.

Our idea is to use Prediction~\ref{pred:grad_norm}, which provides the 2nd moment of per-example gradients. Roughly speaking, we can synthesize ``pseudo gradients'' by randomly sampling Gaussians\footnote{For Random Feature models, the distribution of the per-example gradient $\vmat{G}_{t+s}\cvec{x}\,\transpose{\cvec{x}}\transpose{\cmat{A}}$ is \emph{not} Gaussian -- it's a product of two Gaussian random variables. However, since the distribution is still of the exponential family, and the simulation is a long-term accumulation of gradients across multiple steps, Gaussian sampling at each step should be a good approximation for long-term, average-case analysis.} to match the average and 2nd moment, at each step of the simulation.

In Prediction~\ref{pred:grad_norm}, one can directly measure $\mathbb{E}(\cvec{h}_{t}\transpose{\cvec{h}}_{t})$ from forward activations on a validation set (and Assumption~\ref{ass:iid} implies it does not change over time), but $\mathbb{E}(\nabla\vvec{y}_{t+s}\nabla\transpose{\vvec{y}}_{t+s})$ is tricky because it depends on $\vmat{W}_{t+s}$. However, for the Random Feature model given in \eqref{eq:rf}, where 
\[
y_{t+s}\eqdef\transpose{\vvec{w}}_{t+s}\cmat{A}\cvec{x}
\]
is a \emph{scalar}, one has
\begin{equation}
\label{eq:rf_simplification_volterra}
\mathbb{E}\big((\nabla y_{t+s})^2\big)=\mathbb{E}\Big(c^2\big(\transpose{\vvec{w}}_{t+s}\cmat{A}\cvec{x}-\func{b}{\cvec{x}}\big)^2\Big)=2c\,\mathbb{E}\big(\func{\ell}{\cvec{x};\vvec{w}_{t+s}}\big)
\end{equation}
which is a \emph{multiple of the validation loss}, a presumed quantity that we can handle. Essentially, this simplification \eqref{eq:rf_simplification_volterra} leads to the Volterra equation investigated in \citet{sgd_in_large} and the analytic solution of SGD training dynamics for Random Feature models. In this paper, we further apply \eqref{eq:rf_simplification_volterra} to DNNs, by making the following assumption:
%When the model is large, this is also difficult to be empirically measured even for a fixed step\footnote{For LMs, $\nabla\vvec{y}_{t+s}$ should be considered of $\cscalar{L}\times\cscalar{C}$ dimension where $\cscalar{L}$ is the sequence length, because different tokens in a sequence have correlations. This makes $\nabla\vvec{y}_{t+s}\nabla\transpose{\vvec{y}}_{t+s}$ a tensor of $(\cscalar{L}\times\cscalar{C})^2$ total dimensions, which is infeasible to calculate.}. However, the trace $\trace(\nabla\vvec{y}_{t+s}\nabla\transpose{\vvec{y}}_{t+s})=\lVert\nabla\vvec{y}_{t+s}\rVert^2$ is easy to measure. In order to conduct the simulation, we further make the following assumption:
\begin{assumption}\label{ass:trace}
$\mathbb{E}(\nabla\vvec{y}_{t+s}\nabla\transpose{\vvec{y}}_{t+s})$ can be roughly approximated by $\big(\mathbb{E}(\lVert\nabla\vvec{y}_{t+s}\rVert^2)/\cscalar{C}\big)\cmat{I}$, where $\cmat{I}$ is the identity matrix.
\end{assumption}
This final assumption hypothesizes that most eigenvalues of $\mathbb{E}(\nabla\vvec{y}_{t+s}\nabla\transpose{\vvec{y}}_{t+s})$ are close to each other, so that one can approximately collapse it into a scalar operator (which equals the mean eigenvalue $\mathbb{E}(\lVert\nabla\vvec{y}_{t+s}\rVert^2)/\cscalar{C}$). As a result, our simulation approximates the behavior of a $\cscalar{D}$-dimensional ``average slice'' of the $\cscalar{C}\times\cscalar{D}$ matrix $\vmat{W}$, so that one can apply \eqref{eq:rf_simplification_volterra} and predict the dynamics.
%. This leads exactly to a simple form of Random Feature model as defined in \eqref{eq:rf}, where $\transpose{\vvec{w}}$ is corresponding to the average slice and $\transpose{\vvec{w}}\cmat{A}\cvec{x}$ is a scalar.
\begin{algorithm}[t]
\caption{Simulating the approximate behavior of the layer-wise quadratic loss $L$ from step $t$ to $t+s$, where a $\cscalar{D}$-dimensional average slice $\transpose{\vvec{w}}$ of the $\cscalar{C}\times\cscalar{D}$ matrix $\vmat{W}$ gets trained.}\label{alg:simulation}
\begin{algorithmic}\setstretch{1.2}
\Require $\vvec{\lambda}=(\lambda_i)_{1\leq i\leq\cscalar{D}}$ and $\cmat{U}$, the eigenvalues and eigenvectors of $\mathbb{E}(\cvec{h}_{t}\transpose{\cvec{h}}_{t})$, respectively.
\Require $\mathbb{E}(\nabla\vmat{W}_{t})$, the average gradient of the DNN layer, at step $t$.
\Require $\mathbb{E}(\lVert\nabla_\cscalar{B}\vmat{W}_{t}\lVert^2)$, the mean squared-norm of per-batch gradients.
\Require $\trace\big(\mathbb{E}(\nabla^2\vmat{W}_{t})\big)$, the trace of the layer-wise Hessian.
\Require $\func{\phi}{\bullet}$, an optimizer to map (synthetic) per-batch gradients to updates.
\Ensure $\{L_{t+j}\}_{1\leq j\leq s}$, a trajectory of simulated layer-wise quadratic loss.
\State $c \gets \trace\big(\mathbb{E}(\nabla^2\vmat{W}_{t})\big)/(\sum_{i}\lambda_i\cscalar{C})$ \Comment{Approximate the constant $c$ in \eqref{eq:rf}.}
\State $\mathfrak{B}\gets\dfrac{(\sum_i\lambda_i)^2}{\sum_i\lambda_i^2}\dfrac{\lVert\mathbb{E}(\nabla\vmat{W}_{t})\rVert^2}{\mathbb{E}(\lVert\nabla_\cscalar{B}\vmat{W}_{t}\lVert^2)}$ \Comment{The equivalent batch factor.}
\State $\vvec{z} \gets \rms\big(\mathbb{E}(\nabla\vmat{W}_{t})\cmat{U},\,\texttt{axis}=0\big)/\vvec{\lambda}$ \Comment{Initialize $c(\vvec{w}-\cvec{w}^\ast)$.}
\For{$j=0,\ldots,s$}
\State $L_{t+j} \gets \dfrac{\cscalar{C}}{2c}\sum_i\lambda_iz_i^2$ \Comment{Output.}
\State $n_k\sim\mathcal{N}\Big(0,\lambda_k\big\lvert\sum_i\lambda_iz_i^2/\mathfrak{B}-\lambda_k z_k^2\big\rvert\Big)$ \Comment{Sample stochastic noise of per-batch gradient.}
\State $\vvec{g} \gets \vvec{\lambda}\odot\vvec{z} + (n_k)_{1\leq k\leq\cscalar{D}}$ \Comment{Synthetic per-batch gradient.}
\State $\vvec{u} \gets \transpose{\cmat{U}}\func{\phi}{\cmat{U}\vvec{g}}$ \Comment{Call optimizer.}
\State $\vvec{z} \gets \vvec{z}-c\vvec{u}$ \Comment{Update.}
\EndFor
\end{algorithmic}
\end{algorithm}

The simulation pseudo-code is given in Algorithm~\ref{alg:simulation}. We choose the constant $c$ in the Random Feature model \eqref{eq:rf} such that the trace of the Hessian matches that of the DNN Hessian in the layer-wise quadratic loss \eqref{eq:layer_quadratic_loss}. This approximation is necessary also because when the DNN is large, the full Hessian tensor of $(\cscalar{C}\times\cscalar{D})^2$ total dimensions is infeasible to calculate. Then, our algorithm tracks the $\cscalar{D}$-dimensional vector $\vvec{z}=c(\vvec{w}-\cvec{w}^\ast)$, and initializes $\vvec{z}$ such that the norms of the initial average gradients match in every eigenmode. Finally, we synthesize \emph{per-batch} gradients instead of \emph{per-example} gradients, depending on the batch-size. For this purpose, we define an \emph{equivalent batch factor} $\mathfrak{B}$ which is empirically measured to match the initial mean squared-norm of per-batch gradients. Under certain conditions, $\mathfrak{B}$ coincides with the actual batch-size for Random Feature models (\S~\ref{sec:equiv_batch_factor_derivation}), and we break down the details of $\mathfrak{B}$ 
%The equivalent batch factor under certain conditions coincides with the actual batch-size for Random Feature models (\S~\ref{sec:equiv_batch_factor_derivation}), and we will break down the details of $\mathfrak{B}$ 
for DNNs in \S~\ref{sec:equiv_batch_factor_dnn_breakdown}. The theoretical derivation of Algorithm~\ref{alg:simulation} is given in \S~\ref{sec:derive_alg}. When applied to PLRF models, we empirically verify that the algorithm correctly captures training dynamics, including spikes (Figure~\ref{fig:plrf_simulation}).

Because our simulation does not call the backward pass of a DNN at each step -- it synthesizes a random gradient instead -- the algorithm is extremely lightweight. The initial inputs require $\mathbb{E}(\nabla\vmat{W}_{t})$, $\mathbb{E}(\lVert\nabla_\cscalar{B}\vmat{W}_{t}\lVert^2)$, $\trace\big(\mathbb{E}(\nabla^2\vmat{W}_{t})\big)$ and (the eigen decomposition of) $\mathbb{E}(\cvec{h}_{t}\transpose{\cvec{h}}_{t})$, which cost one pass over a validation set. See Table~\ref{tab:compute} for a summary of the compute requirements.

\subsection{A Sketch of why Training Spikes Occur}
\label{sec:sketch}

In this section, we briefly explain why power-law distributed inputs can lead to training spikes. %the stochastic noise of per-example gradients combined with power-law distributed inputs can lead to training spikes.
For simplicity, we analyze the dynamics of $\vvec{w}$ under SGD. (We expect optimizers with momentum and precondition, such as Adam, to follow a similar principle.) At any step $t$, let $\eta_{t}$ be the learning-rate. The update of SGD is given by
%Using Assumption~\ref{ass:iid}, we will analyze the dynamics of $\vmat{W}$ under SGD (in our experiments, preconditioned optimizers such as RMSProp will be approximated by averaging the preconditioning to a scalar learning-rate). For any $s\geq 0$, let $\gamma_{t+s}$ be the learning-rate. The update rule is given by
\begin{equation}\label{eq:update}
\vvec{w}_{t+1}=\vvec{w}_{t}-\eta_{t}\nabla\vvec{w}_{t},
\end{equation}
where $\nabla\vvec{w}_{t}$ is a per-example gradient. This is a random process depending on the data sample at each step; We apply the decomposition $\nabla\vvec{w}_{t}=\mathbb{E}(\nabla\vvec{w}_{t})+\vvec{n}_{t}$, where $\mathbb{E}$ averages over the data distribution, and $\vvec{n}_{t}$ is a random vector of mean $\cvec{0}$, representing the stochastic noise of training samples.

Under a quadratic loss, the average gradient $\mathbb{E}(\nabla\vvec{w}_{t})$ is in the form $\mathbb{E}(\nabla\vvec{w}_{t})=\cmat{H}(\vvec{w}_{t}-\cvec{w}^\ast)$, where $\cmat{H}$ is the Hessian. Thus, we can decompose \eqref{eq:update} into $\vvec{w}_{t}=\overline{\vvec{w}}_{t}+\vvec{r}_{t}$, such that
\begin{align}
\overline{\vvec{w}}_{t+1}-\cvec{w}^\ast&=(\cmat{I}-\eta_{t}\,\cmat{H})(\overline{\vvec{w}}_{t}-\cvec{w}^\ast) \label{eq:predictable} \\
\mbox{and}\quad\vvec{r}_{t+1}&=(\cmat{I}-\eta_{t}\,\cmat{H})\,\vvec{r}_{t}-\eta_{t}\,\vvec{n}_{t}. \label{eq:randomness}
\end{align}
So $\overline{\vvec{w}}_{t}$ is the predictable full-batch gradient descent, while $\vvec{r}_{t}$ is a random walk with mean $\cvec{0}$, accumulating the stochastic noise $\vvec{n}_{t}$ across steps. When the Central Limit Theorem applies, $\vvec{r}_{t}$ should approximately follow a (anisotropic) Gaussian distribution.

Intuitively, one can imagine $\vvec{r}_{t}$ as a high-dimension (isotropic) Gaussian with a given variance $\sigma^2$; so for any random $\vvec{r}_{t}$, we know that $\lVert\vvec{r}_{t}\rVert^2\approx\cscalar{D}\sigma^2$ with high probability, while $\vvec{r}_{t}$ can point to any direction. The fact that $\lVert\vvec{r}_{t}\rVert^2\approx\cscalar{D}\sigma^2$ holds (with high probability) is an example of the more general principle called ``\emph{concentration of measure}'': the value of a \emph{scalar} function of $\vvec{r}_{t}$ ``concentrates'' to its expectation, because the random fluctuations of entries of $\vvec{r}_{t}$ can cancel out each other, much like in the Law of Large Numbers.

The validation loss $\mathbb{E}\big(\func{\ell}{\cvec{x};\bullet}\big)$ is a scalar, quadratic function of $\vvec{w}_{t}=\overline{\vvec{w}}_{t}+\vvec{r}_{t}$. So if this quadratic function is similar to $\lVert\bullet\rVert^2$, by analogy to the above, the concentration of measure principle applies and $\mathbb{E}\big(\func{\ell}{\cvec{x};\overline{\vvec{w}}_{t}+\vvec{r}_{t}}\big)$ should (with high probability) not depend on the random direction of $\vvec{r}_{t}$. As a result, $\mathbb{E}\big(\func{\ell}{\cvec{x};\overline{\vvec{w}}_{t}+\vvec{r}_{t}}\big)$ converges to a deterministic trajectory at the high-dimension limit $\cscalar{D}\to\infty$.

However, in order for the ``concentration of measure'' to hold, the eigenvalues of the quadratic function $\mathbb{E}\big(\func{\ell}{\cvec{x};\bullet}\big)$ cannot be too heavy-tailed. If there are some very large eigenvalues that cannot be counter-weighted by all the small ones, the value of $\mathbb{E}\big(\func{\ell}{\cvec{x};\overline{\vvec{w}}_{t}+\vvec{r}_{t}}\big)$ will suddenly surge when $\vvec{r}_{t}$ randomly points toward directions of the large eigenvalues -- a ``\emph{spike}'' will occur.

We also note that, the learning-rate has to be relatively large so that $\vvec{r}_{t}$ can accumulate to a significant amount, for the training spikes to manifest.
%in order for the spikes to be observed,    the random $\vvec{r}_{t}$ (which evolves as in \eqref{eq:randomness}, depending on the learning-rate $\eta_t$) has large norm. 
On the other hand, $\vvec{r}_{t}$ starts to grow larger and larger if the learning-rate is too high -- which means the training diverges. The threshold $\gamma$ of divergence can be accurately calculated for SGD, which is given in \eqref{eq:gamma}. One observes large training spikes when the learning-rate is slightly below $\gamma$. For general optimizers, one can use our simulation to find the threshold learning-rate, because Algorithm~\ref{alg:simulation} faithfully models the accumulation of random gradient noise. For more technical details, please refer to Appendix~\ref{sec:derive_gamma} and \ref{sec:concentration_alpha}.

\begin{table}[t]
\centering
\resizebox{\textwidth}{!}{%
\begin{tabular}{@{}lcccccc@{}}
\toprule
Experiment & Train            & Eval $\delta^{\vvec{w}}_{t,s}$ & Eval $\mathbb{E}(\cvec{h}_{t}\transpose{\cvec{h}}_{t})$ & $\mathbb{E}(\nabla\vmat{W}_{t})$ \& $\mathbb{E}(\lVert\nabla_\cscalar{B}\vmat{W}_{t}\lVert^2)$ & $\trace\big(\mathbb{E}(\nabla^2\vmat{W}_{t})\big)$ & Simulation \\ \midrule
Data size  & 25.5B tokens     & 336M tokens     & 336M tokens                                             & 336M tokens                                                                                     & 336M tokens                                        & 1500 steps \\
Model size & 1.3B (24 blocks) & 24 \texttt{ffn\_0} & all layers                                              & all layers                                                                                      & 6 \texttt{ffn\_0}                                & 1 layer    \\
Time       & 40 hrs           & 1.1 hrs           & 8 min                                                   & 15 min                                                                                          & 1.7 hrs                                            & 5 sec      \\
Machine    & TPUv4$\times32$  & TPUv4$\times32$ & TPUv4$\times32$                                         & TPUv4$\times32$                                                                                 & TPUv4$\times32$                                    & CPU        \\ \bottomrule
\end{tabular}%
}
\caption{Summary of compute resources for running experiments in this paper.}
\label{tab:compute}
\end{table}

% \begin{table}[t]
% \scriptsize
% \centering
% \begin{tabular}{@{}lcccccc@{}}
% \toprule
% Experiment & Train            & Eval $\delta^{\vvec{w}}_{t,s}$ & Eval $\mathbb{E}(\cvec{h}_{t}\transpose{\cvec{h}}_{t})$ & $\mathbb{E}(\nabla\vmat{W}_{t})$ \& $\mathbb{E}(\lVert\nabla_\cscalar{B}\vmat{W}_{t}\lVert^2)$ & $\trace\big(\mathbb{E}(\nabla^2\vmat{W}_{t})\big)$ & Simulation \\ \midrule
% Data size  & 25.5B tokens     & 336M tokens     & 336M tokens                                             & 336M tokens                                                                         & 336M tokens                                        & 1500 steps \\
% Model size & 1.3B (24 blocks) & 24 \texttt{ffn\_0} & all layers                                              & all layers                                                                          & 6 \texttt{ffn\_0}                            & 1 layer    \\
% Time       & 40 hrs           & 1.1 hrs           & 8 min                                                   & 15 min                                                                              & 1.7 hrs                                            & 5 sec      \\
% Machine    & TPUv4$\times32$  & TPUv4$\times32$ & TPUv4$\times32$                                         & TPUv4$\times32$                                                                     & TPUv4$\times32$                                    & CPU        \\ \bottomrule
% \end{tabular}
% \caption{Summary of compute resources for running experiments in this paper.}
% \label{tab:compute}
% \end{table}

\section{Simulation of Transformer LM Training}
\label{sec:simulation_transformer_LM}

\begin{figure}[t!]
 \centering
 \begin{subfigure}[b]{0.32\textwidth}
     \centering
     \includegraphics[width=\textwidth]{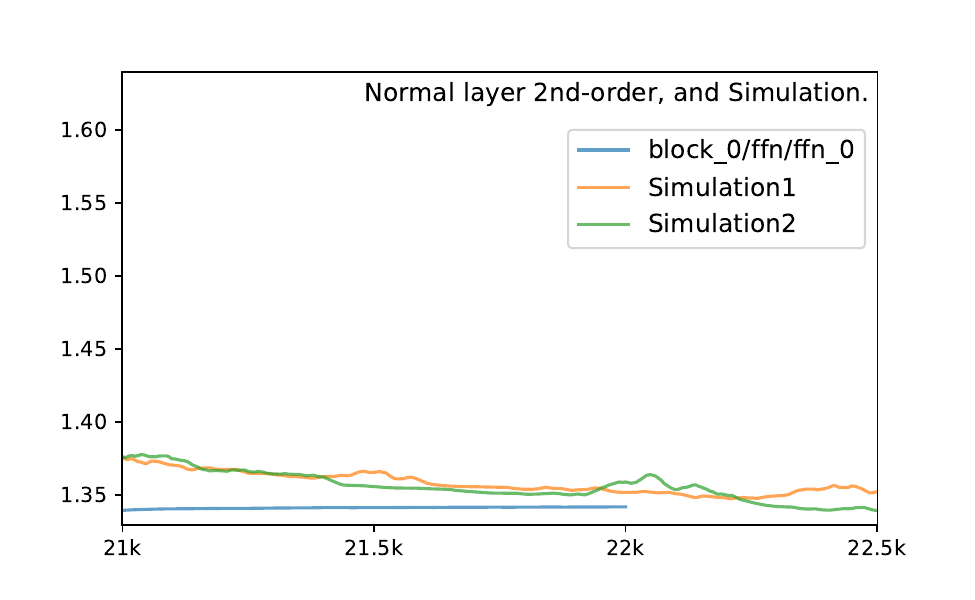}
     %\caption{(a)}
     %\label{fig:y equals x}
 \end{subfigure}
 \begin{subfigure}[b]{0.32\textwidth}
     \centering
     \includegraphics[width=\textwidth]{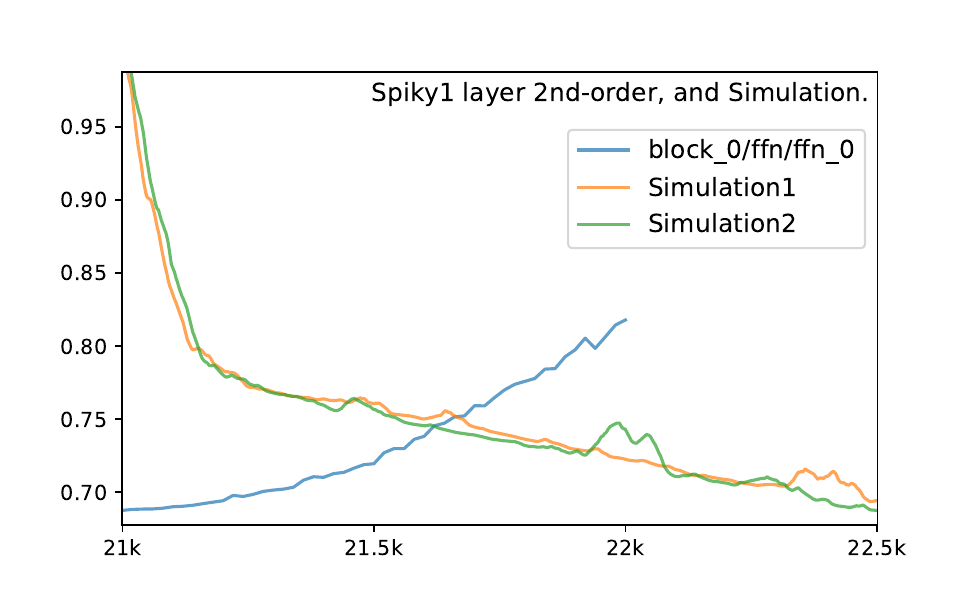}
     %\caption{(a)}
     %\label{fig:y equals x}
 \end{subfigure}
 \begin{subfigure}[b]{0.32\textwidth}
     \centering
     \includegraphics[width=\textwidth]{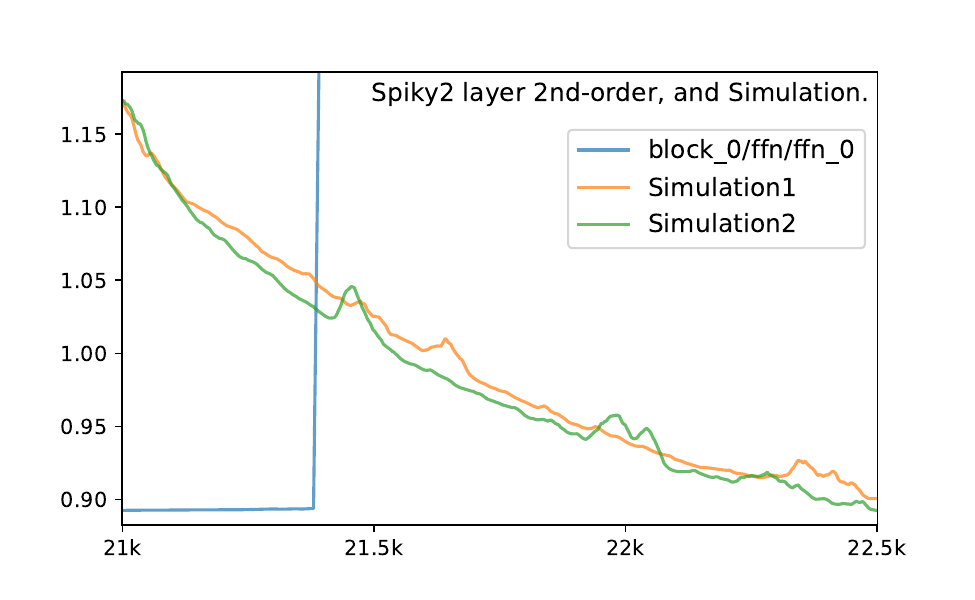}
     %\caption{(a)}
     %\label{fig:y equals x}
 \end{subfigure}\\
 \begin{subfigure}[b]{0.32\textwidth}
     \centering
     \includegraphics[width=\textwidth]{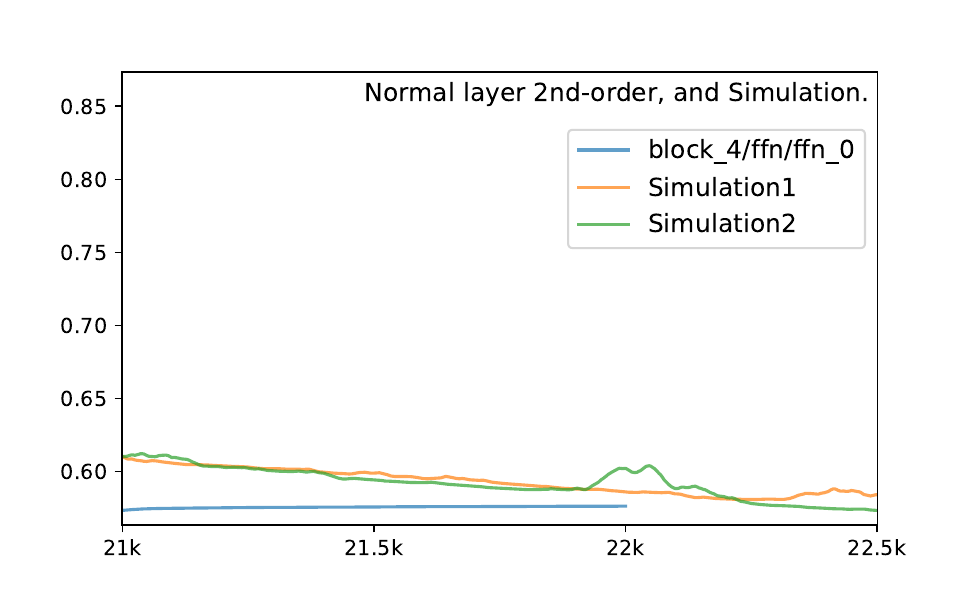}
     %\caption{(a)}
     %\label{fig:y equals x}
 \end{subfigure}
 \begin{subfigure}[b]{0.32\textwidth}
     \centering
     \includegraphics[width=\textwidth]{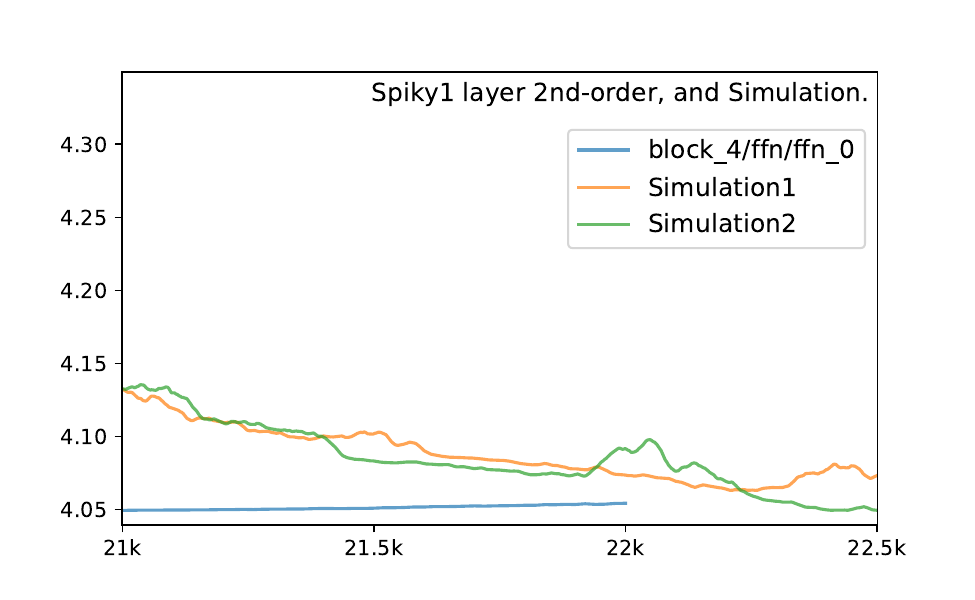}
     %\caption{(a)}
     %\label{fig:y equals x}
 \end{subfigure}
 \begin{subfigure}[b]{0.32\textwidth}
     \centering
     \includegraphics[width=\textwidth]{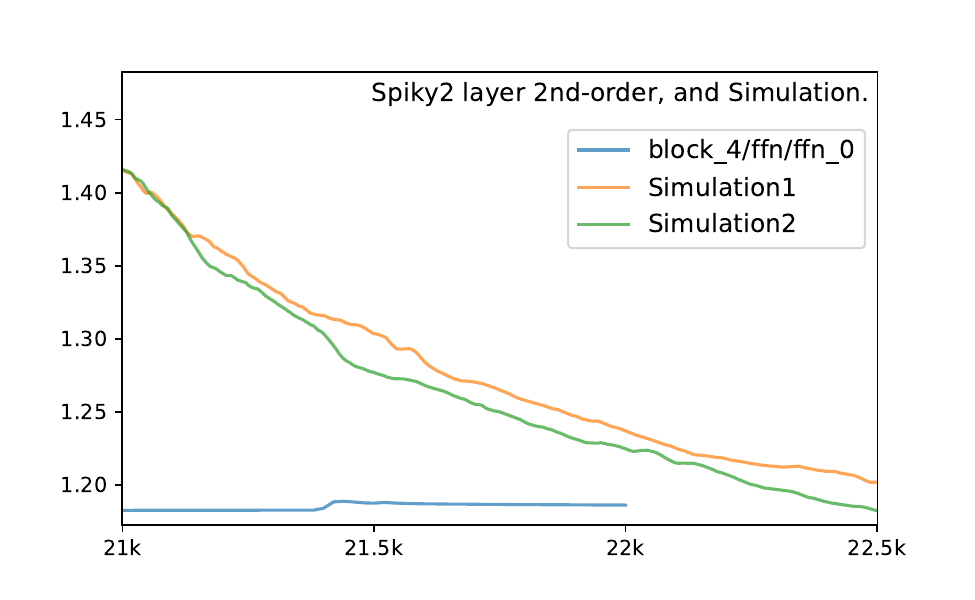}
     %\caption{(a)}
     %\label{fig:y equals x}
 \end{subfigure}\\
 \begin{subfigure}[b]{0.32\textwidth}
     \centering
     \includegraphics[width=\textwidth]{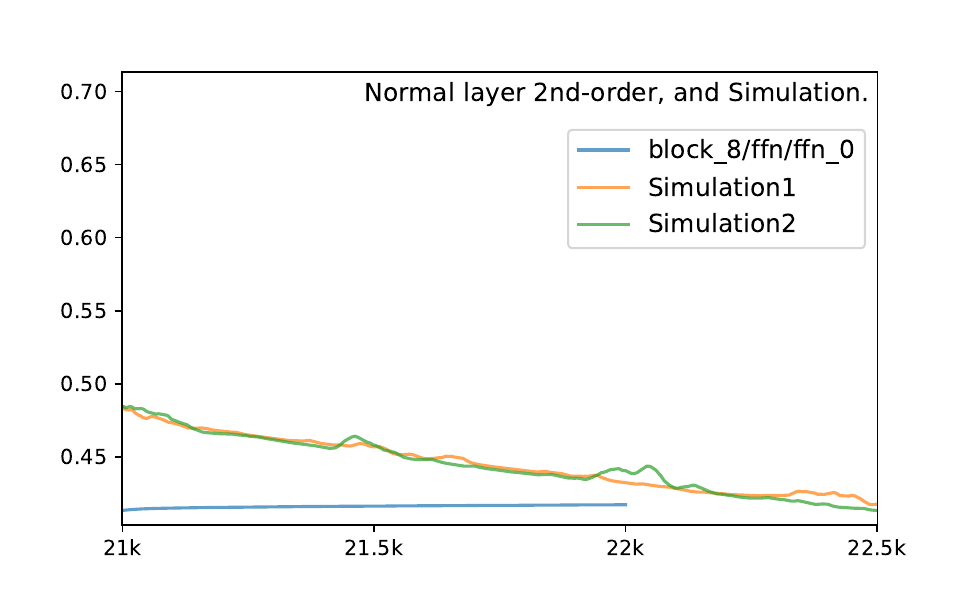}
     %\caption{(a)}
     %\label{fig:y equals x}
 \end{subfigure}
 \begin{subfigure}[b]{0.32\textwidth}
     \centering
     \includegraphics[width=\textwidth]{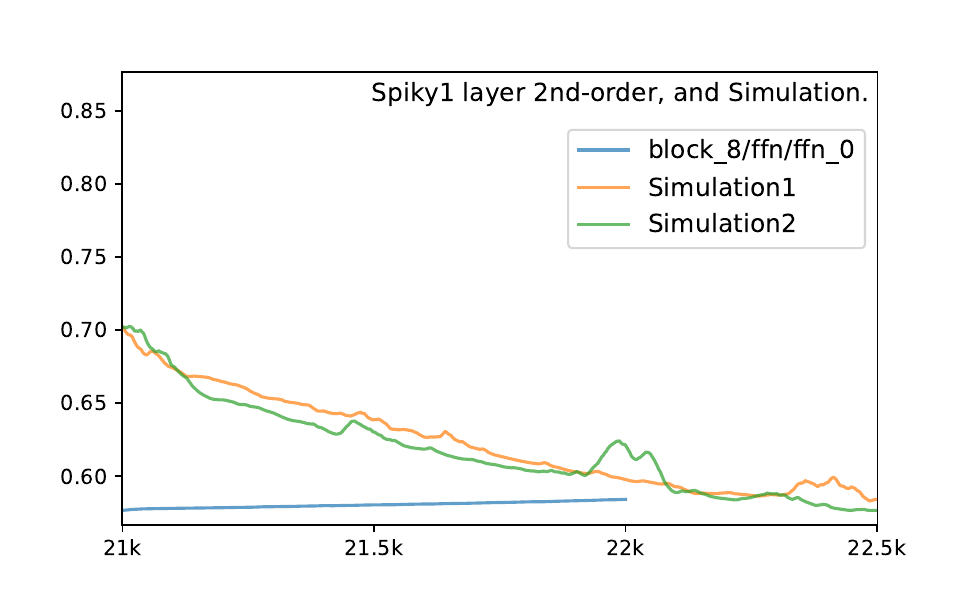}
     %\caption{(a)}
     %\label{fig:y equals x}
 \end{subfigure}
 \begin{subfigure}[b]{0.32\textwidth}
     \centering
     \includegraphics[width=\textwidth]{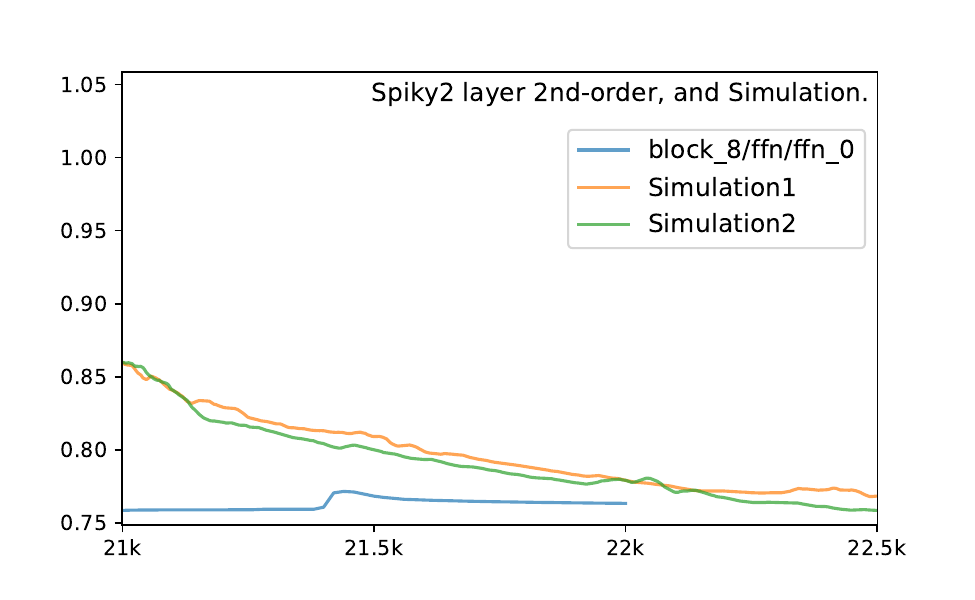}
     %\caption{(a)}
     %\label{fig:y equals x}
 \end{subfigure}\\
 \begin{subfigure}[b]{0.32\textwidth}
     \centering
     \includegraphics[width=\textwidth]{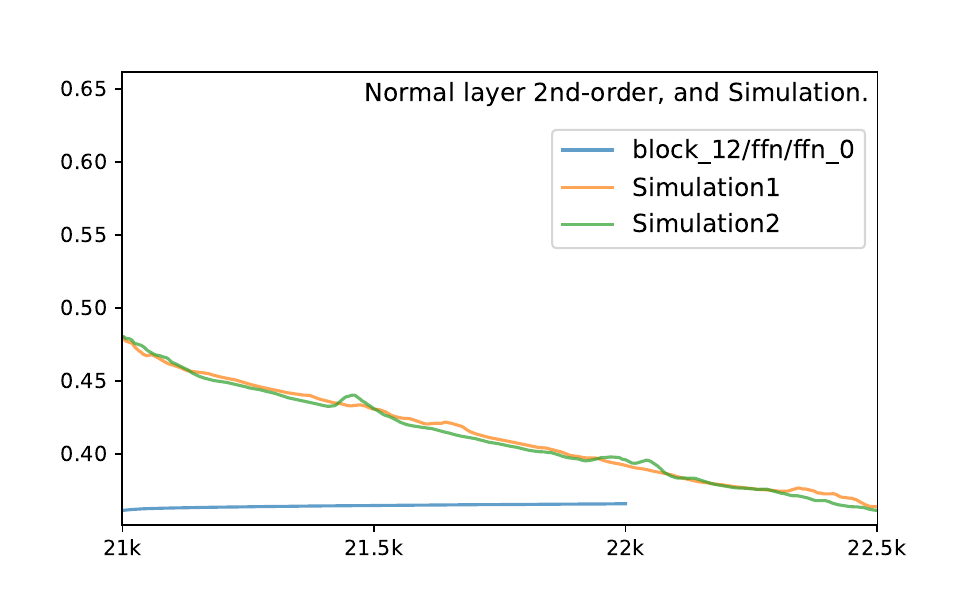}
     %\caption{(a)}
     %\label{fig:y equals x}
 \end{subfigure}
 \begin{subfigure}[b]{0.32\textwidth}
     \centering
     \includegraphics[width=\textwidth]{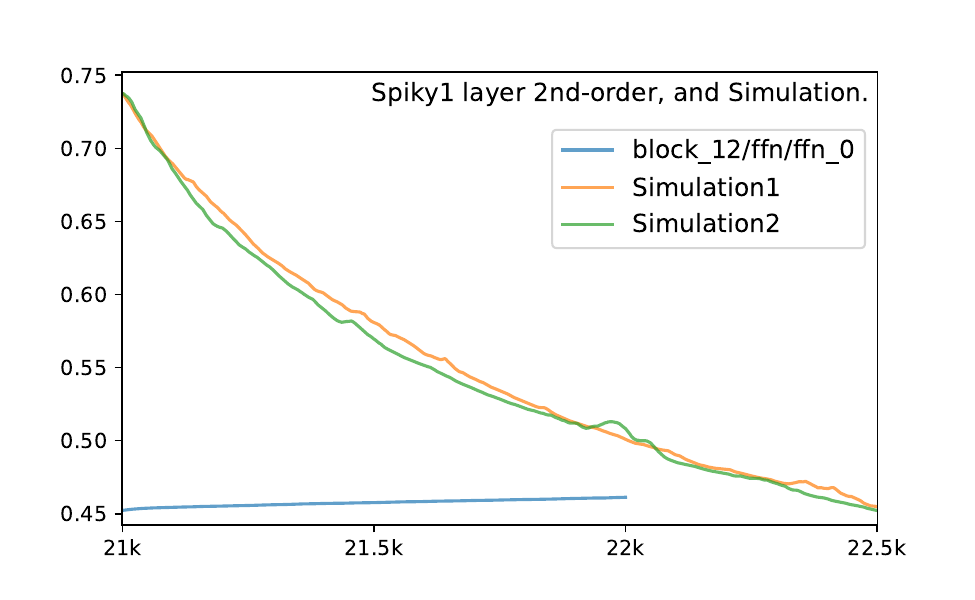}
     %\caption{(a)}
     %\label{fig:y equals x}
 \end{subfigure}
 \begin{subfigure}[b]{0.32\textwidth}
     \centering
     \includegraphics[width=\textwidth]{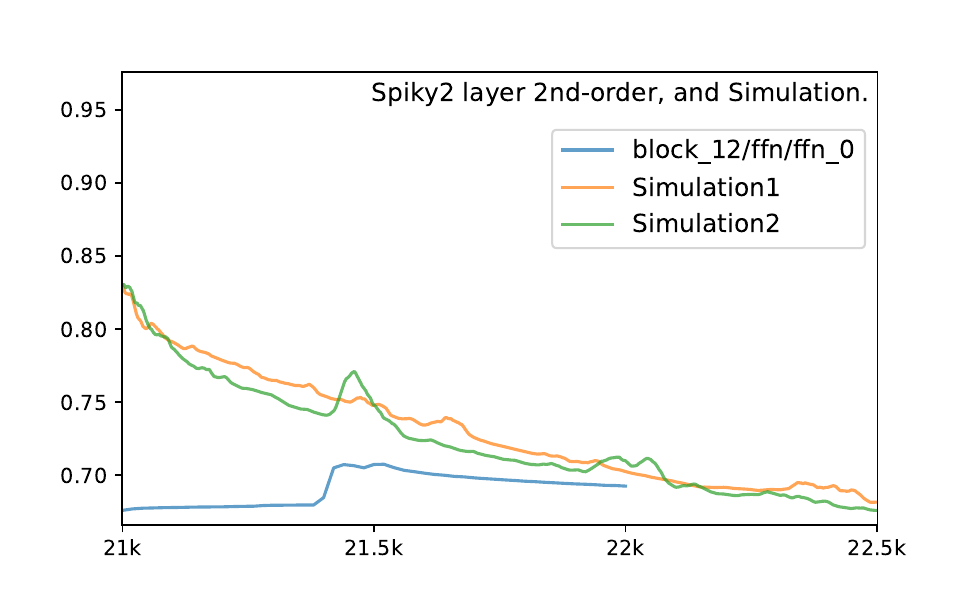}
     %\caption{(a)}
     %\label{fig:y equals x}
 \end{subfigure}\\
 \begin{subfigure}[b]{0.32\textwidth}
     \centering
     \includegraphics[width=\textwidth]{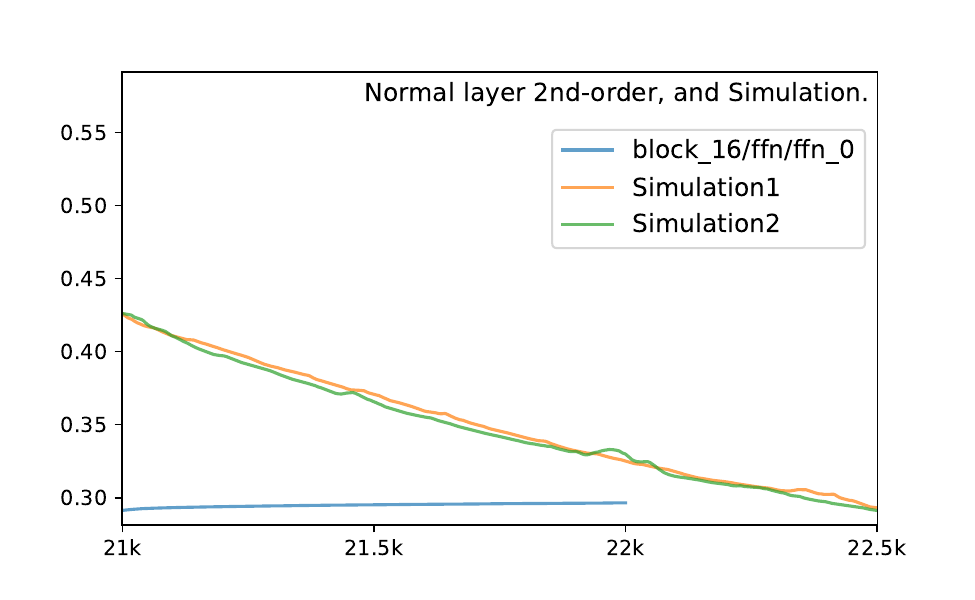}
     %\caption{(a)}
     %\label{fig:y equals x}
 \end{subfigure}
 \begin{subfigure}[b]{0.32\textwidth}
     \centering
     \includegraphics[width=\textwidth]{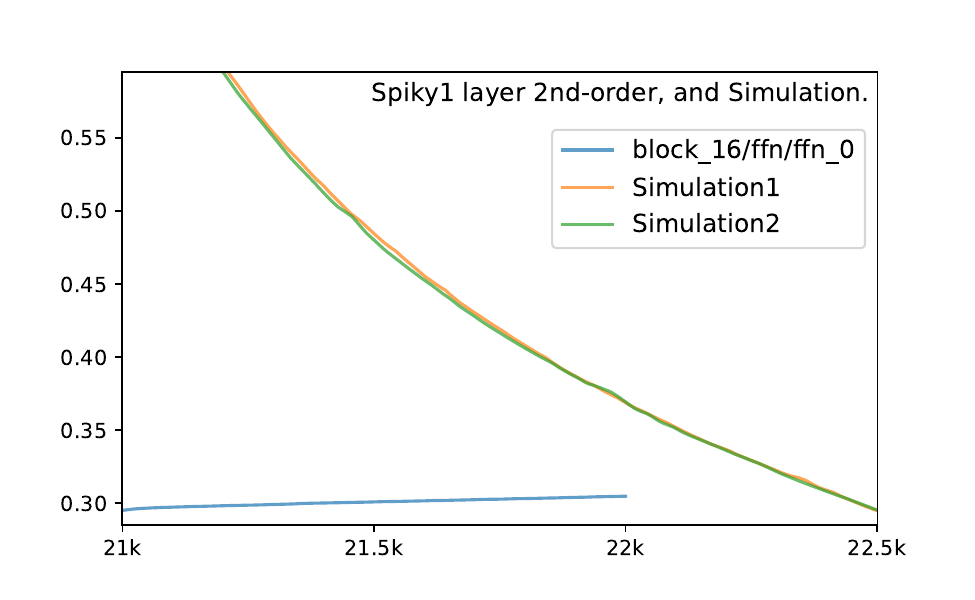}
     %\caption{(a)}
     %\label{fig:y equals x}
 \end{subfigure}
 \begin{subfigure}[b]{0.32\textwidth}
     \centering
     \includegraphics[width=\textwidth]{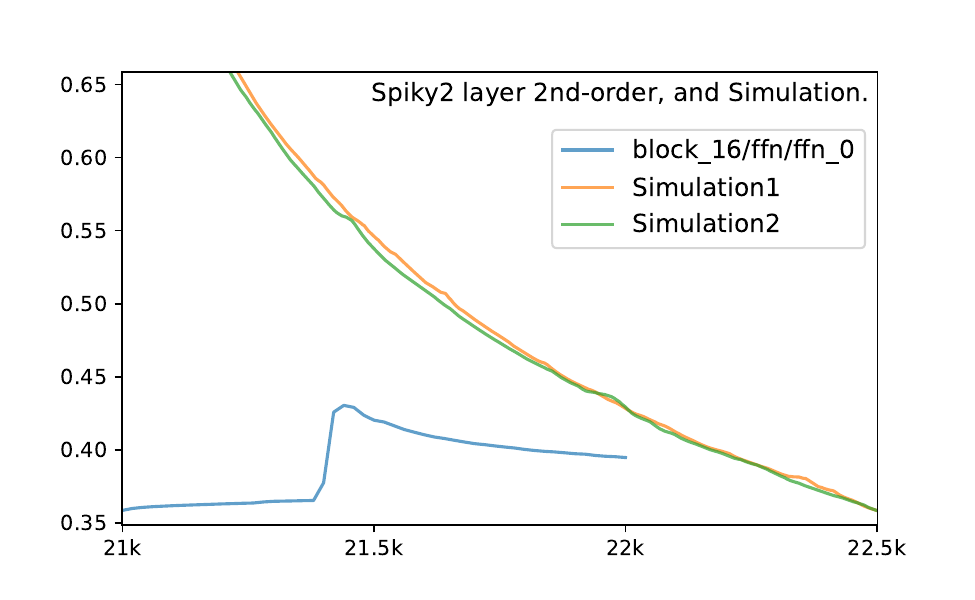}
     %\caption{(a)}
     %\label{fig:y equals x}
 \end{subfigure}\\
 \begin{subfigure}[b]{0.32\textwidth}
     \centering
     \includegraphics[width=\textwidth]{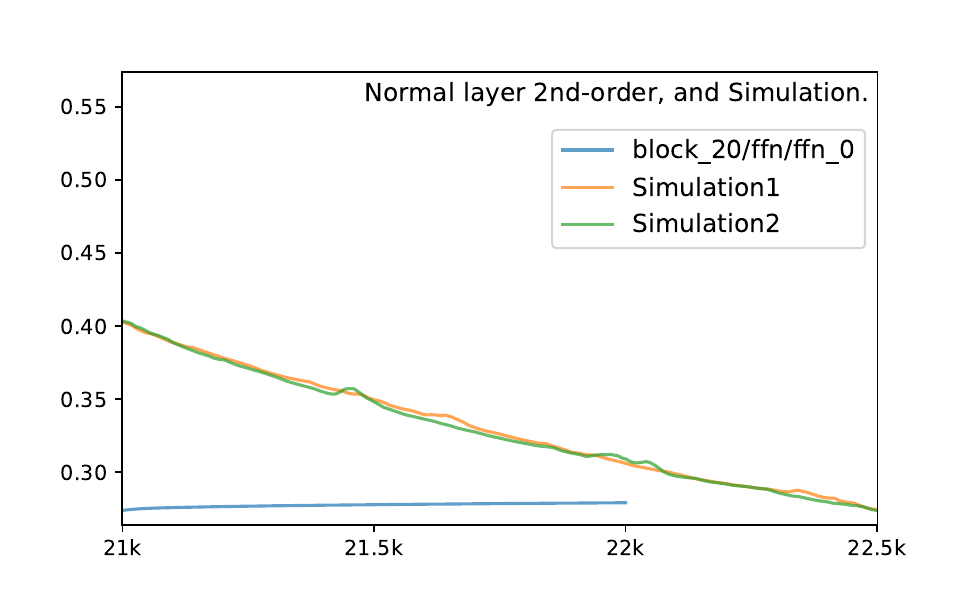}
     %\caption{(a)}
     %\label{fig:y equals x}
 \end{subfigure}
 \begin{subfigure}[b]{0.32\textwidth}
     \centering
     \includegraphics[width=\textwidth]{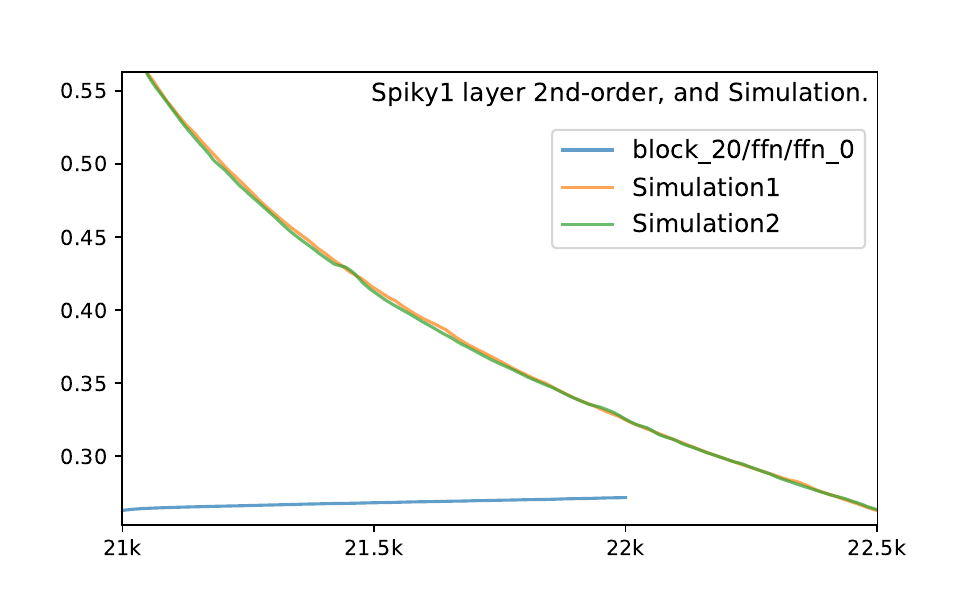}
     %\caption{(a)}
     %\label{fig:y equals x}
 \end{subfigure}
 \begin{subfigure}[b]{0.32\textwidth}
     \centering
     \includegraphics[width=\textwidth]{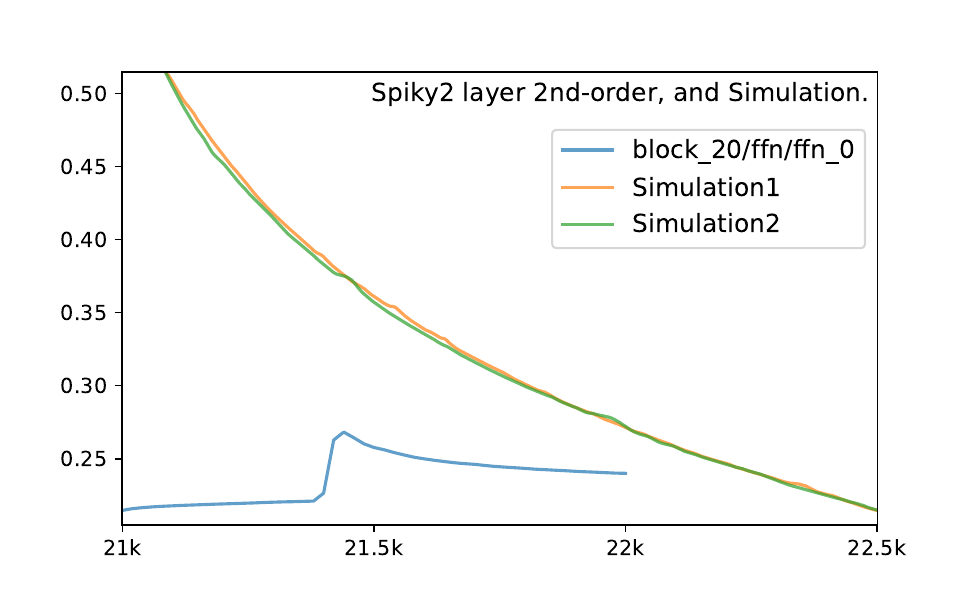}
     %\caption{(a)}
     %\label{fig:y equals x}
 \end{subfigure}
\caption{Simulation of the layer-wise quadratic approximation of validation loss at 6 layers (the feed-forward input of block\_\{0,4,8,12,16,20\}), from training step 21k to 22.5k. We compare 3 settings of training (Normal, Spiky1 and Spiky2) as in Figure~\ref{fig:spiky_loss_all}, and plot two simulations of different seeds, together with the actual 2nd-order term measured from corresponding training trajectories. The scale of the $y$-axis is kept the same for all plots, so that the fluctuations can be compared.}
\label{fig:dnn_simulation}
\end{figure}

As a case-study, we train a vanilla Transformer decoder model of 1.3B parameters (24 Transformer blocks) on the C4 dataset, to the Chinchilla optimal number of tokens. We use the AdamW optimizer with hyperparameters inherited from previous research \citep{DBLP:conf/icml/EverettXWANLGSK24,DBLP:journals/corr/abs-2503-09799}, which are tuned to near optimal. Then, we change $\beta_2$ of AdamW from $0.95$ to $0.999$ in order to produce training spikes. We obtain three training trajectories: Normal, Spiky1 and Spiky2, as illustrated in Figure~\ref{fig:spiky_loss_all}. More details are given in \S~\ref{sec:llm_train_settings}.

We take the model checkpoints at step 21k, and run simulations for all three configurations. Optimizer states are initialized from model checkpoints, but we did not apply weight-decay during the simulation. Figure~\ref{fig:dnn_simulation} shows the results of 6 feed-forward layers simulated (with 2 random seeds each) from step 21k to 22.5k. We see clear bumps in the simulation for \emph{both} Spiky1 and Spiky2 (e.g., block\_\{0,4,8\} at step $\sim$22k for Spiky1, and block\_\{0,12\} at step $\sim$21.5k for Spiky2), while the simulation for Normal is comparatively more smooth. It suggests that our simulation has the predictive power of telling whether a training configuration can lead to spikes, from a given model checkpoint. It is noteworthy that in the actual training trajectory (Figure~\ref{fig:spiky_loss_all}), only Spiky2 spikes between steps 21k and 22k; And the spiky positions in our simulations differ between random seeds, demonstrating the random nature of training spikes. In \S~\ref{sec:power_law_dist}, we will also show that the distribution of eigenvalues of $\mathbb{E}(\cvec{h}_{t}\transpose{\cvec{h}}_{t})$ is significantly more heavy-tailed for spiky configurations.

On the other hand, we plot the 2nd-order term of $\delta^{\vvec{w}}_{t,s}$ (Definition~\ref{def:delta}), namely
\[
\frac{1}{2}\transpose{(\vvec{w}_{t+s}\!-\!\vvec{w}_{t})}\mathbb{E}_{\mathfrak{x}\sim\mathcal{D}}\Big(\func{\frac{\partial^2\ell}{\transpose{\partial\vvec{w}}\partial\vvec{w}}}{\mathfrak{x};\mathcal{W}_{t}}\Big)(\vvec{w}_{t+s}\!-\!\vvec{w}_{t}),
\]
from step 21k to 22k, in Figure~\ref{fig:dnn_simulation} with a shifted origin but the same scale, in order to compare the largeness of fluctuations. We observe jumps in the 2nd-order terms from Spiky2 synchronized with the training spike, while the largeness of the jumps matches the fluctuation of simulations in some layers (e.g.~block\_\{4,8,12\}) if not all. %The 1st-order term of $\delta^{\vvec{w}}_{t,s}$, and more discussion can be found in Appendix~\ref{sec:response_plot}.
One big difference between our simulation and the empirical measurement of the 2nd-order terms of $\delta^{\vvec{w}}_{t,s}$, is that the 2nd-order terms of $\delta^{\vvec{w}}_{t,s}$ exhibit ``jumps'' rather than ``spikes'' (i.e., not resuming to the previous level after initial surge), which is also different from the behavior of the total validation loss as shown in Figure~\ref{fig:spiky_loss_valid_partial}. This suggests that the dynamics of the Taylor expansion $\delta^{\vvec{w}}_{t,s}$ may not be fully explained by layer-wise quadratic models, as it cannot fully explain the total validation loss either. One may need to consider correlations across layers, and Taylor expansions beyond 2nd order.

%\subsection{Discussion about \texorpdfstring{$\delta^{\vvec{w}}_{t,s}$}{delta}}
%\label{sec:response_plot}

\begin{figure}[t!]
 \centering
 \begin{subfigure}[b]{0.32\textwidth}
     \centering
     \includegraphics[width=\textwidth]{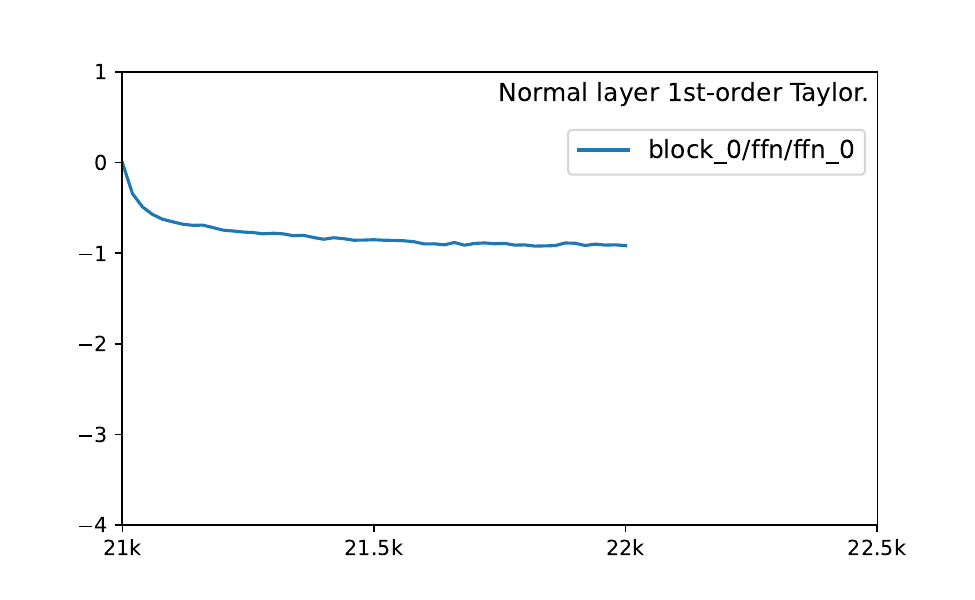}
     %\caption{(a)}
     %\label{fig:y equals x}
 \end{subfigure}
 \begin{subfigure}[b]{0.32\textwidth}
     \centering
     \includegraphics[width=\textwidth]{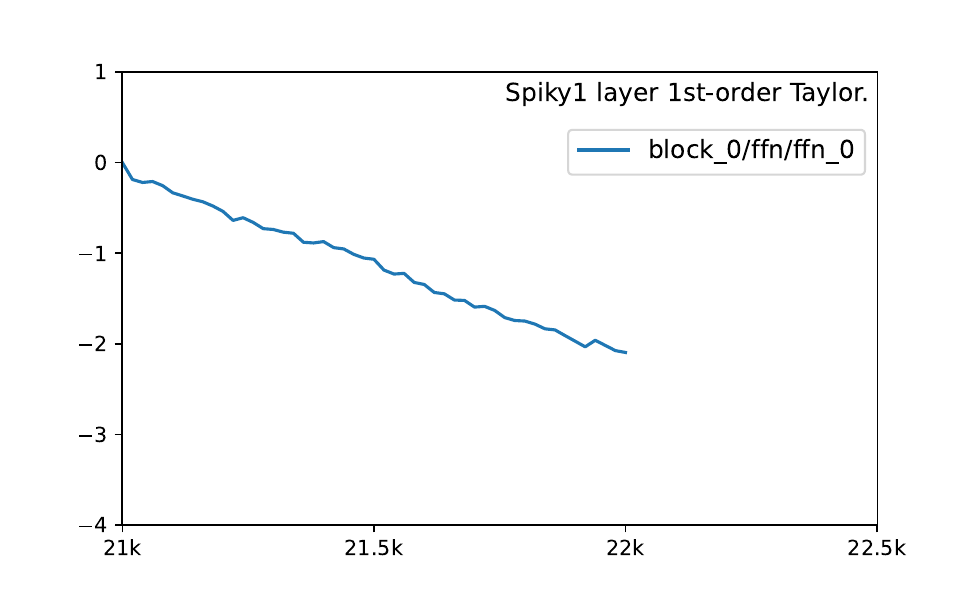}
     %\caption{(a)}
     %\label{fig:y equals x}
 \end{subfigure}
 \begin{subfigure}[b]{0.32\textwidth}
     \centering
     \includegraphics[width=\textwidth]{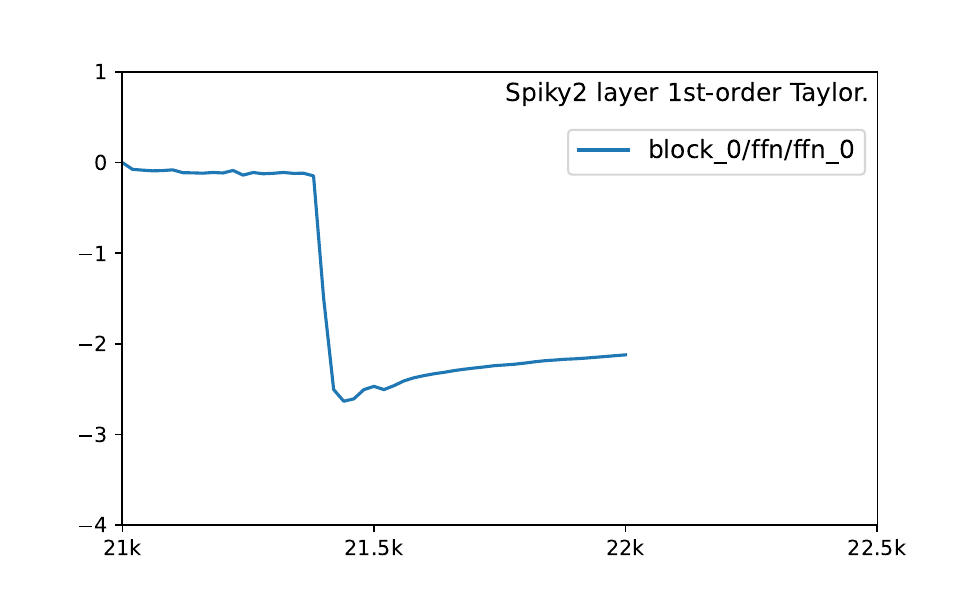}
     %\caption{(a)}
     %\label{fig:y equals x}
 \end{subfigure}\\
 \begin{subfigure}[b]{0.32\textwidth}
     \centering
     \includegraphics[width=\textwidth]{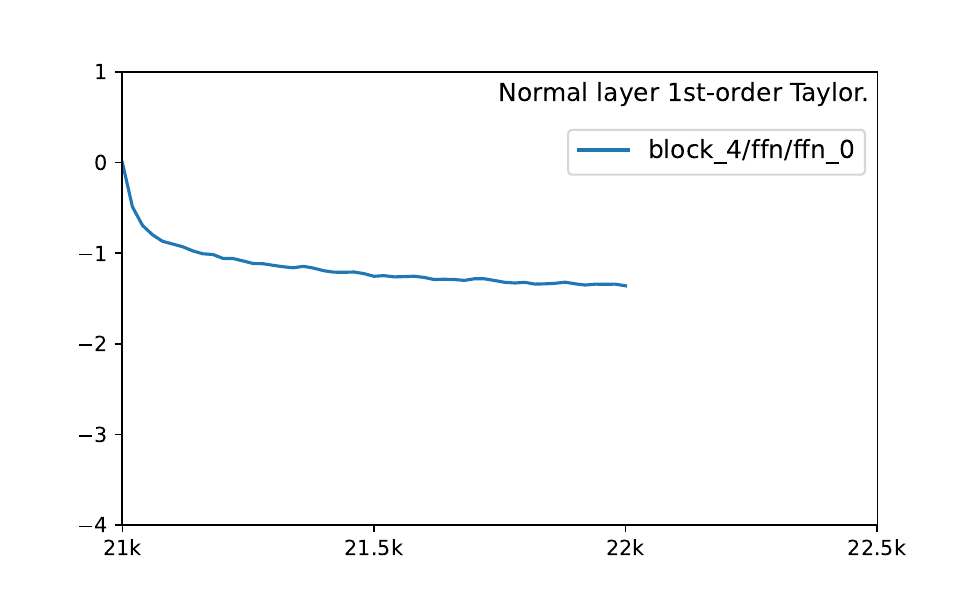}
     %\caption{(a)}
     %\label{fig:y equals x}
 \end{subfigure}
 \begin{subfigure}[b]{0.32\textwidth}
     \centering
     \includegraphics[width=\textwidth]{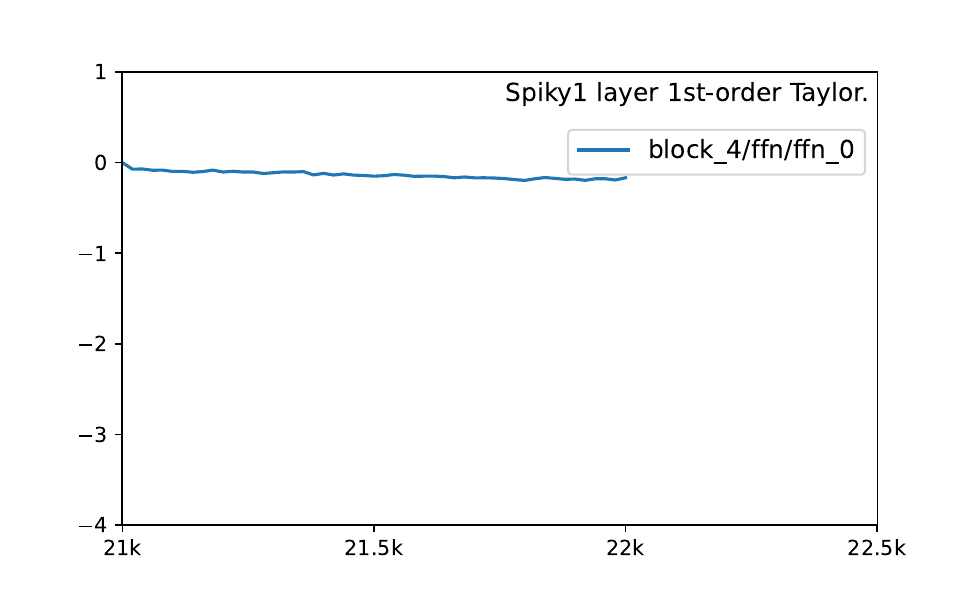}
     %\caption{(a)}
     %\label{fig:y equals x}
 \end{subfigure}
 \begin{subfigure}[b]{0.32\textwidth}
     \centering
     \includegraphics[width=\textwidth]{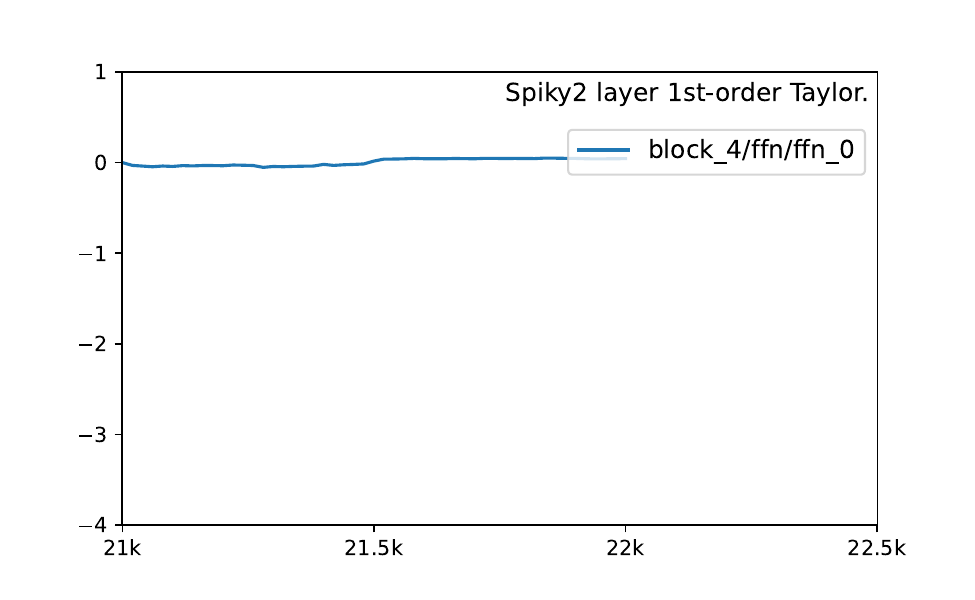}
     %\caption{(a)}
     %\label{fig:y equals x}
 \end{subfigure}\\
 \begin{subfigure}[b]{0.32\textwidth}
     \centering
     \includegraphics[width=\textwidth]{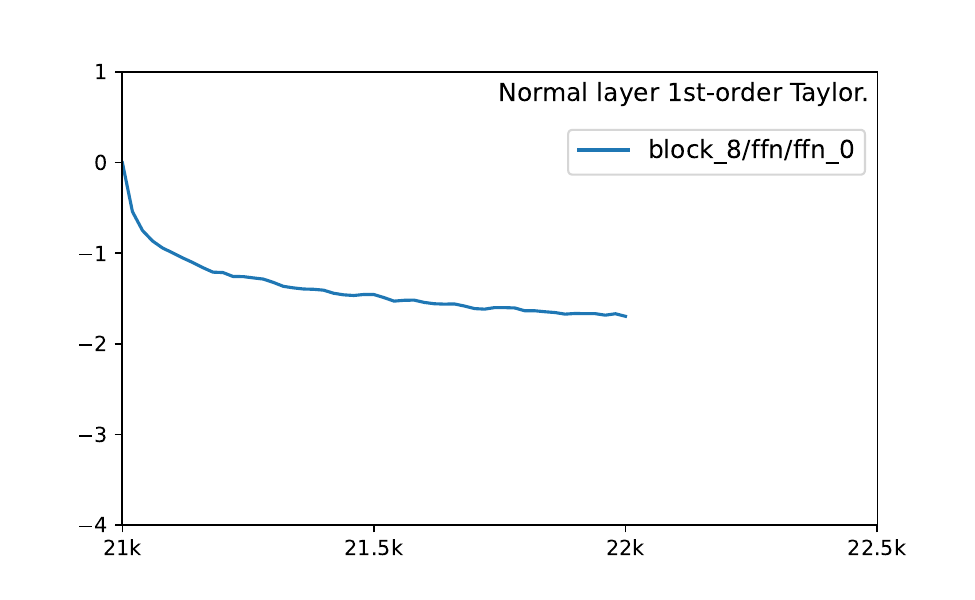}
     %\caption{(a)}
     %\label{fig:y equals x}
 \end{subfigure}
 \begin{subfigure}[b]{0.32\textwidth}
     \centering
     \includegraphics[width=\textwidth]{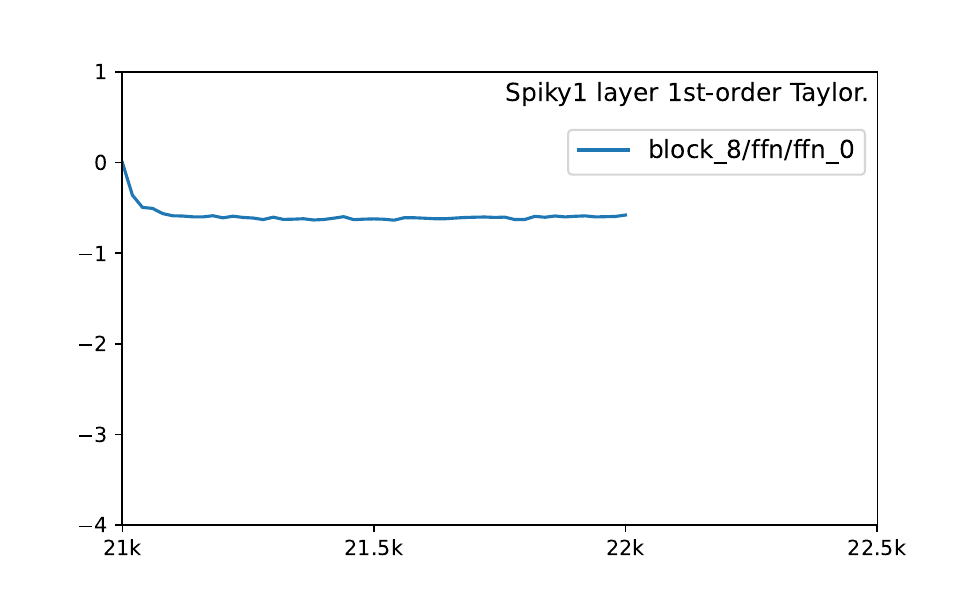}
     %\caption{(a)}
     %\label{fig:y equals x}
 \end{subfigure}
 \begin{subfigure}[b]{0.32\textwidth}
     \centering
     \includegraphics[width=\textwidth]{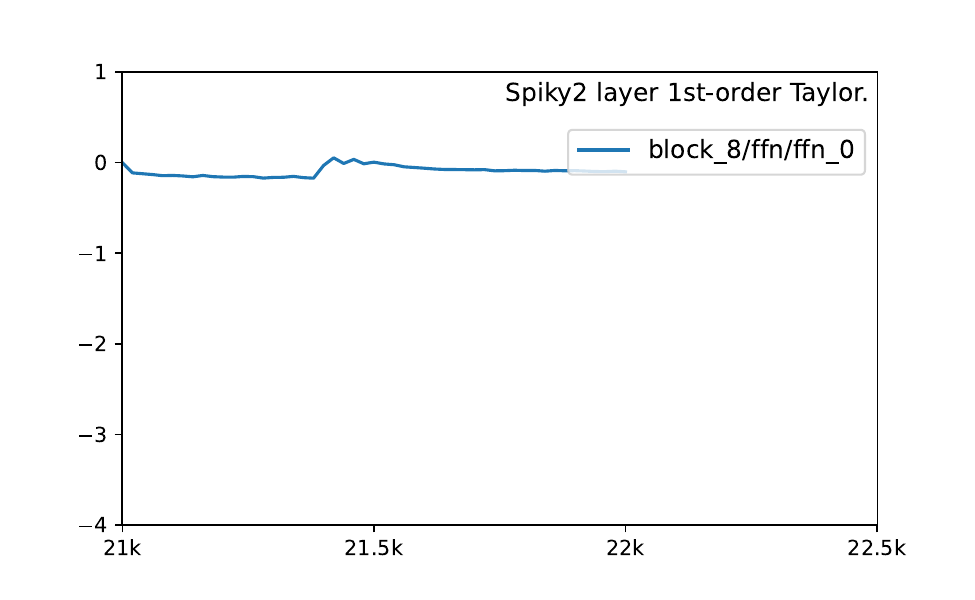}
     %\caption{(a)}
     %\label{fig:y equals x}
 \end{subfigure}\\
 \begin{subfigure}[b]{0.32\textwidth}
     \centering
     \includegraphics[width=\textwidth]{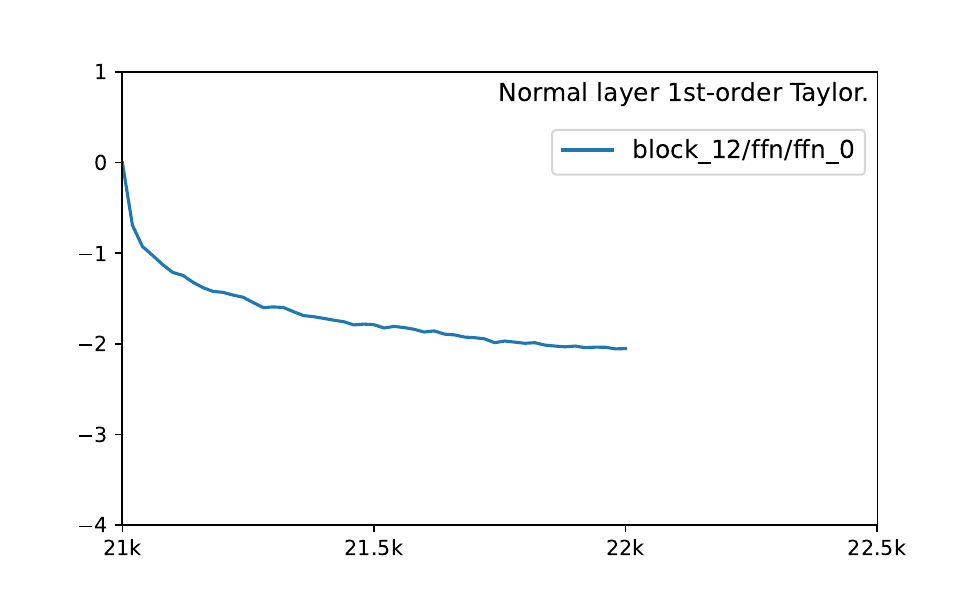}
     %\caption{(a)}
     %\label{fig:y equals x}
 \end{subfigure}
 \begin{subfigure}[b]{0.32\textwidth}
     \centering
     \includegraphics[width=\textwidth]{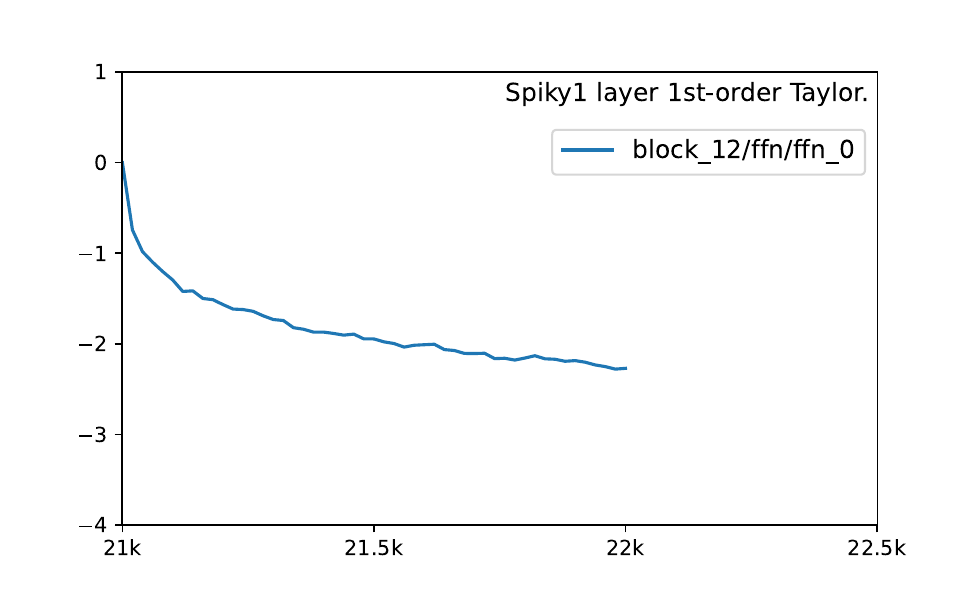}
     %\caption{(a)}
     %\label{fig:y equals x}
 \end{subfigure}
 \begin{subfigure}[b]{0.32\textwidth}
     \centering
     \includegraphics[width=\textwidth]{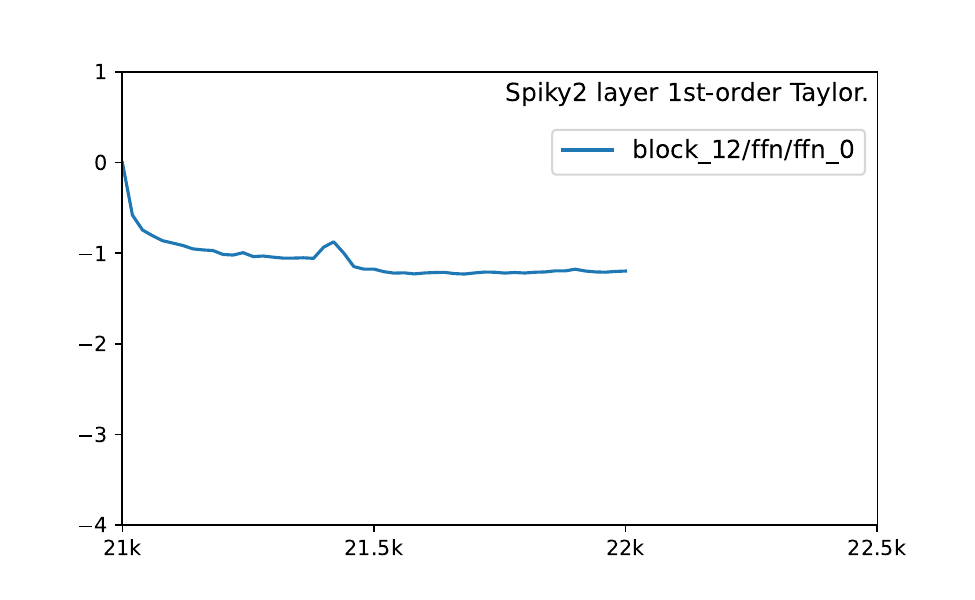}
     %\caption{(a)}
     %\label{fig:y equals x}
 \end{subfigure}\\
 \begin{subfigure}[b]{0.32\textwidth}
     \centering
     \includegraphics[width=\textwidth]{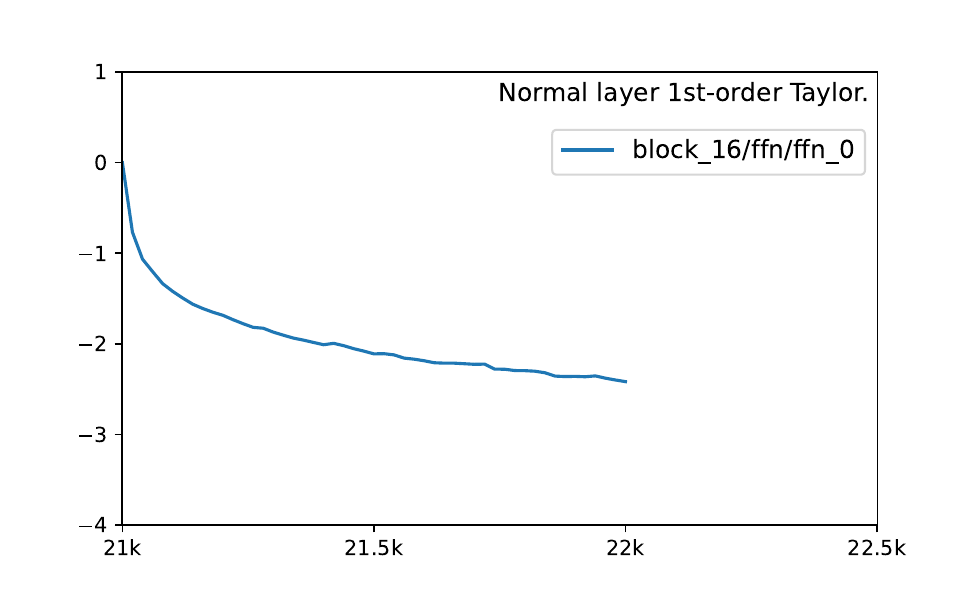}
     %\caption{(a)}
     %\label{fig:y equals x}
 \end{subfigure}
 \begin{subfigure}[b]{0.32\textwidth}
     \centering
     \includegraphics[width=\textwidth]{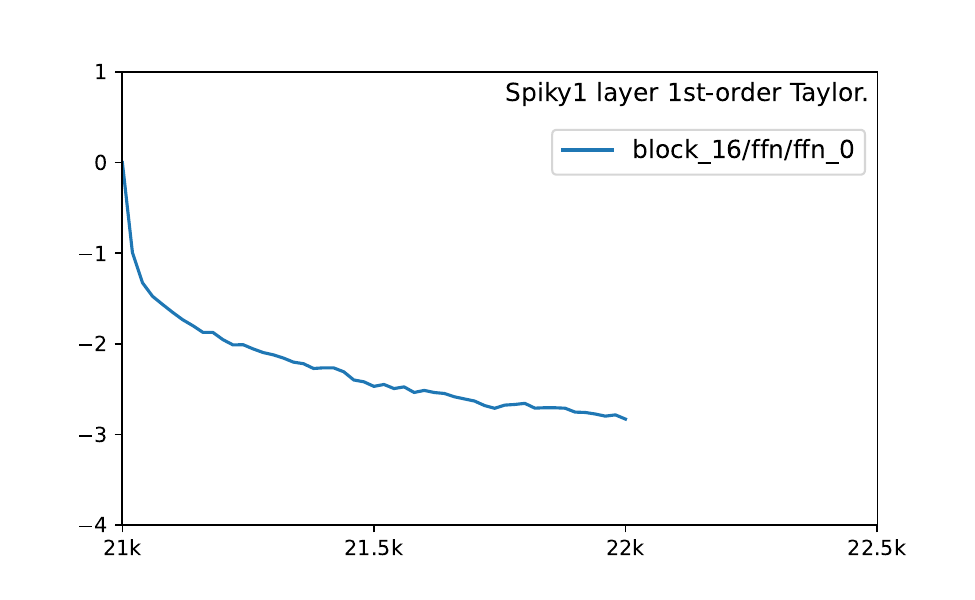}
     %\caption{(a)}
     %\label{fig:y equals x}
 \end{subfigure}
 \begin{subfigure}[b]{0.32\textwidth}
     \centering
     \includegraphics[width=\textwidth]{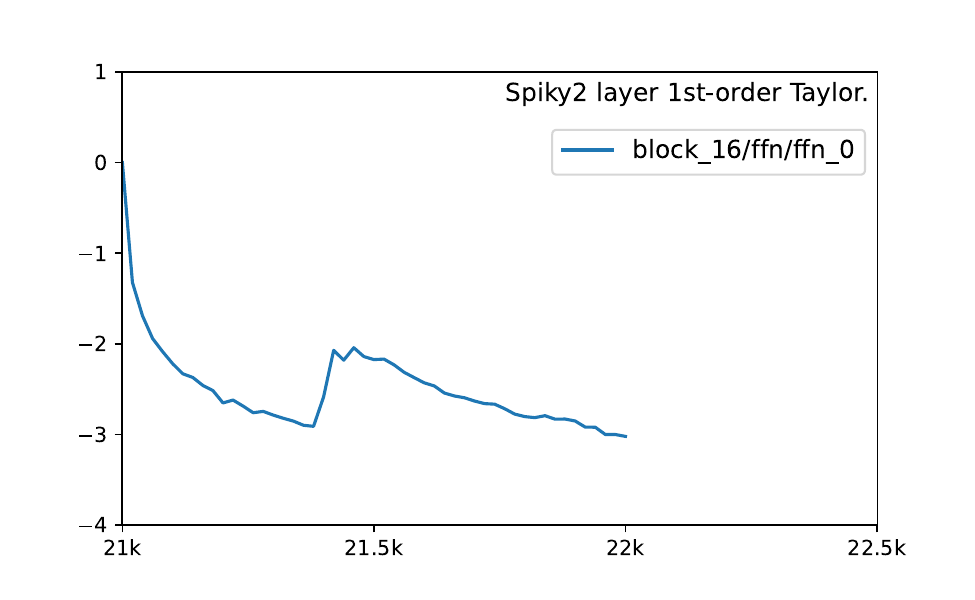}
     %\caption{(a)}
     %\label{fig:y equals x}
 \end{subfigure}\\
 \begin{subfigure}[b]{0.32\textwidth}
     \centering
     \includegraphics[width=\textwidth]{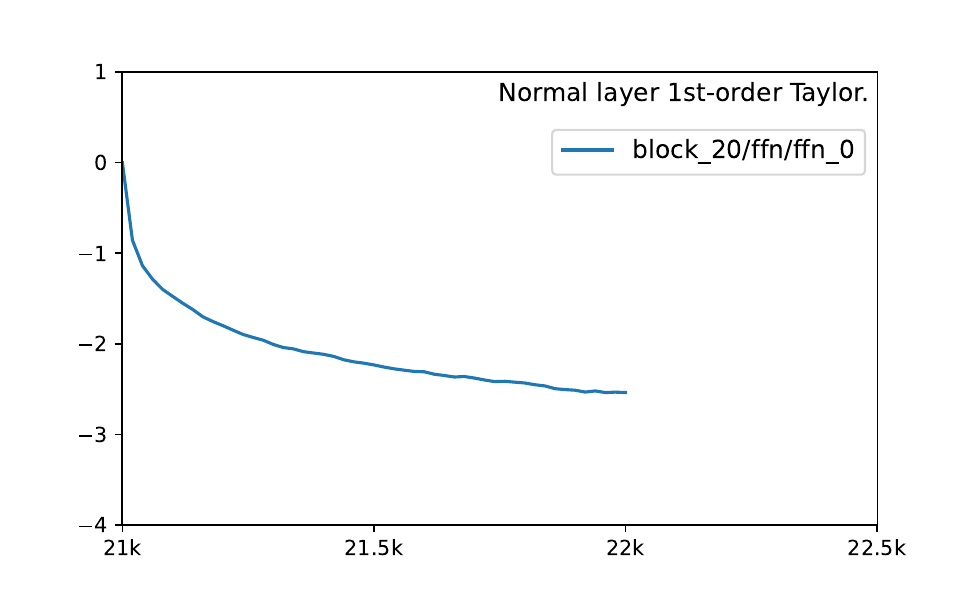}
     %\caption{(a)}
     %\label{fig:y equals x}
 \end{subfigure}
 \begin{subfigure}[b]{0.32\textwidth}
     \centering
     \includegraphics[width=\textwidth]{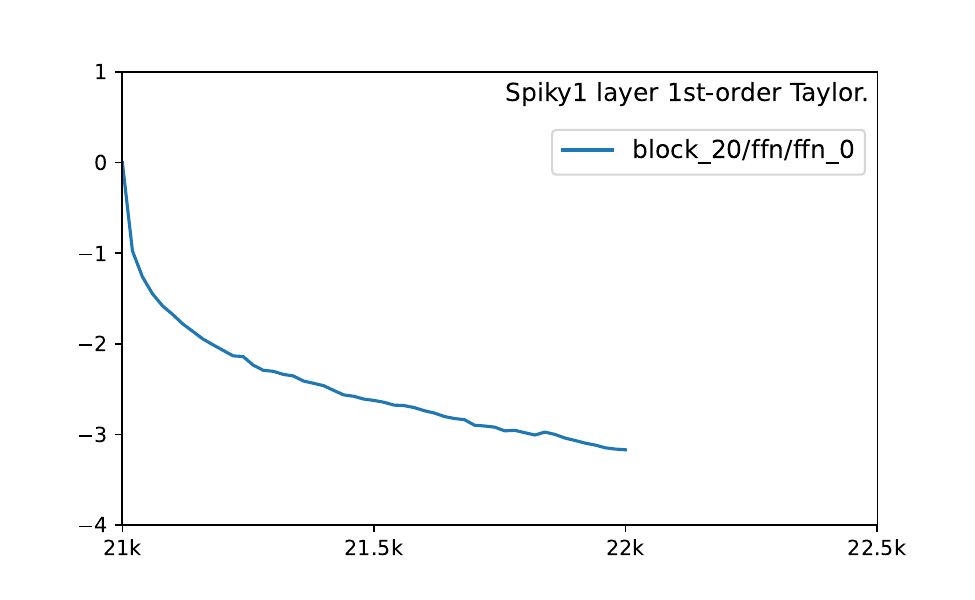}
     %\caption{(a)}
     %\label{fig:y equals x}
 \end{subfigure}
 \begin{subfigure}[b]{0.32\textwidth}
     \centering
     \includegraphics[width=\textwidth]{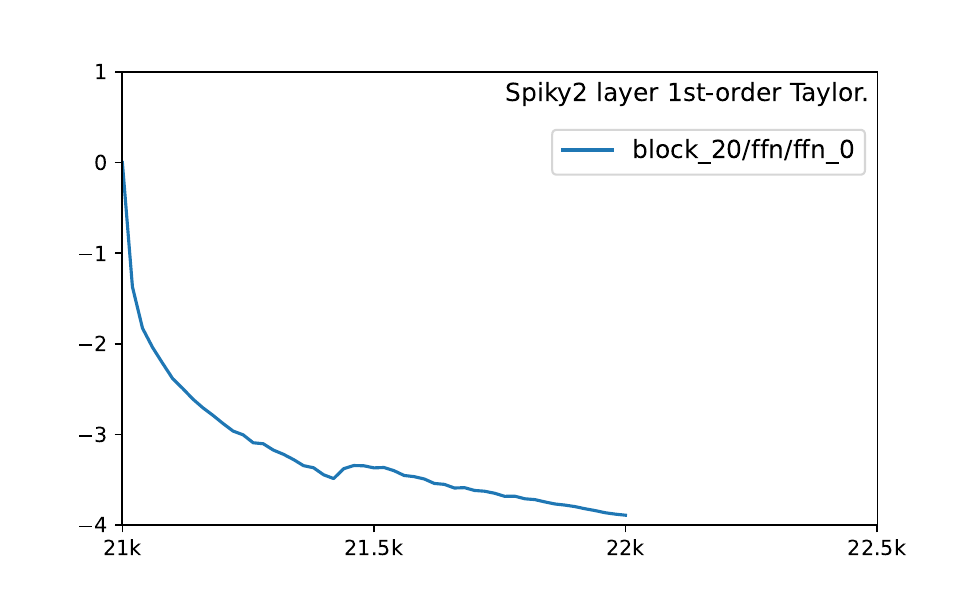}
     %\caption{(a)}
     %\label{fig:y equals x}
 \end{subfigure}
\caption{Plots of the 1st-order term of $\delta^{\vvec{w}}_{t,s}$.}
\label{fig:response}
\end{figure}

In addition, we plot the 1st order term of $\delta^{\vvec{w}}_{t,s}$, namely
\[
\mathbb{E}_{\mathfrak{x}\sim\mathcal{D}}\Big(\func{\frac{\partial\ell}{\partial\vvec{w}}}{\mathfrak{x};\mathcal{W}_{t}}\Big)(\vvec{w}_{t+s}\!-\!\vvec{w}_{t}),
\]
in Figure~\ref{fig:response}. We use the same scale of $y$-axis (from $-4$ to $1$) across all layers, and we compare the three training configurations (Normal, Spiky1 and Spiky2).

We found that the behavior of the 1st-order terms cannot be directly explained by layer-wise quadratic models either: First, note that the scale of $y$-axis in Figure~\ref{fig:response} is about $20\times$ the scale in Figure~\ref{fig:dnn_simulation}; the decrease of the 1st-order terms turns out to be \emph{much faster} than one would expect from a layer-wise quadratic model -- it seems that the highly non-linear nature of the DNN can somehow \emph{accelerate} the feature-learning towards a loss-decreasing direction. On the other hand, the highly non-linear nature of the DNN must also produce \emph{larger noises} than one would expect from a layer-wise quadratic model, because the overall validation loss of the DNN does not decrease this fast. Second, we observe that training spikes are reflected in the plots of 1st-order terms, as either ``bumps'' (block\_12 of Spiky2), ``jumps'' (block\_\{8,16,20\} of Spiky2), or ``dives'' (block\_0 of Spiky2); but layer-wise quadratic models cannot exhibit similar behavior at this scale.

How we can explain these discrepancies, and how the correlations across layers and Taylor expansions beyond 2nd order may play a role in the mechanism, are intriguing questions for future work.

\section{Discussion and Related Works}
\label{sec:discussion}

Theoretical analysis of training spikes in most of previous works has focused on the concept of \emph{sharpness}, defined as the largest eigenvalue $\lambda_\textrm{max}$ of the Hessian. But the truth is, none of them has been able to create toy models with loss dynamics, that can match aesthetically the phenomenon we empirically documented as ``training spikes'', such as in Figure~\ref{fig:spiky_loss_train_full}. Some works show fluctuations with a log-scale in $y$-axis (which will amplify the largeness of the fluctuations in appearance), and try to connect them with training spikes -- \textbf{I'm sorry, but they are not similar.}

Thus, we first want to emphasize the significance of our examples in PLRF models, solely from an aesthetics perspective: Please note the similarity of spikes between Figure~\ref{fig:spiky_loss_train_full} and Figure~\ref{fig:plrf}, and the fact that neither of them has log-scales in $y$-axis. We provide Python code of the PLRF examples in supplementary materials, which can run on Google Colab within 2 minutes. \textbf{Please just try it.}

In order to analyze the behavior of PLRF models under stochastic optimization, one has to investigate the gradient noise -- a factor that does not exist in full-batch training and has been largely overlooked in previous theoretical study of training spikes. In practice, although recent LMs are trained with large numbers of tokens per batch, we especially investigated the \emph{equivalent batch factor} in this paper and found that the gradient noise in training a Transformer LM is comparable to a Random Feature model with batch-size $\approx\!1$ (\S~\ref{sec:equiv_batch_factor_dnn_breakdown}). This is because the point-wise training loss of DNNs can increase gradient noise (\S~\ref{sec:equiv_batch_factor_derivation}) -- an effect specific to DNNs that can cancel the large batch-size. \textbf{In short, one cannot ignore gradient noise for training a Transformer LM.}

As a result, the sharpness $\lambda_\textrm{max}$ is only distantly relevant to training dynamics of stochastic optimization. What really matters is the \emph{distribution} of all eigenvalues $\{\lambda_i\}$ -- That said, we would like to discuss several streams of previous works next, which have contributed to our understanding of training spikes, as well as the more general topic of training process of DNNs.

\paragraph{Feature Learning and Sharpness} The Hessian is ever changing during the training of DNNs -- a fact that cannot be captured by quadratic models like PLRF. Understanding how the Hessian evolves, which we call ``feature learning'', is an important topic to understand deep learning. It is reported by \citet{jastrzębski2018on,eos,adaptive_eos} that during the training of DNNs, the sharpness $\lambda_\textrm{max}$ is increasing until it meets the threshold $2/\gamma$ (where $\gamma$ is the learning-rate); And when $\gamma$ decreases, $\lambda_\textrm{max}$ increases further. This process, called ``progressive sharpening'', can be affected by several optimization hyperparameters \citep{gilmer2022a}, and it is connected to training spikes through Adam preconditioner \citep{bai2026adaptive}, the ``lower-loss-as-sharper structure'' and condensation of model weights \citep{li2026}. The phase where learning-rate is set to larger than $2/\lambda_\textrm{max}$, called ``catapult'' \citep{lewkowycz2020largelearningratephase}, is directly connected to training spikes as well \citep{zhu2024catapultssgdspikestraining}. When the learning-rate temporarily exceeds $2/\lambda_\textrm{max}$, \citet{DBLP:conf/iclr/DamianNL23,central_flow} show that a 3rd-order term will automatically stabilize the sharpness and prevent it from increasing further -- the so-called ``Edge-of-Stability''. We believe the insights regarding 3rd-order terms can be valuable in stochastic optimization as well: As we discussed in \S~\ref{sec:simulation_transformer_LM} and \S~\ref{sec:response_plot}, both the simulation and empirical measurement of layer-wise quadratic loss cannot fully explain the scale of spikes observed in the total validation loss; the random fluctuations driven by gradient noises in quadratic forms may get \emph{amplified} and \emph{stabilized} by higher-order terms. %We will leave this direction to future work.

\paragraph{Spectrum Analysis of Hessian and Model Weights} Beyond the largest eigenvalue $\lambda_\textrm{max}$, our empirical measurement of the whole spectrum of the layer-wise Hessian (especially the trace of the Hessian) is largely inspired by \citet{pmlr-v97-ghorbani19b}. Unfortunately, the computation conducted in \citet{pmlr-v97-ghorbani19b} becomes infeasible for the model size of recent LMs; In this work, we take a shortcut by sticking to the \emph{layer-wise} Hessian instead of the Hessian of the whole model, and we make Assumption~\ref{ass:trace} to largely reduce the computation to Hessian trace and eigen decomposition of correlation matrices of forward activations. Nevertheless, we note that \citet{pmlr-v97-ghorbani19b} observed heavy-tailed distribution of eigenvalues similar to what we observe in Figure~\ref{fig:loglog}. From another perspective, \citet{JMLR:v22:20-410} investigated the spectrum of model weights, and recently the spectrum of gradients and moments become a topic of great interest \citep{fan2025implicit,chen2026muon}, thanks to the popularity of Muon optimizer \citep{jordan2024muon}. Our work is not directly connected to this stream because we simulate an average \emph{slice} of the model weight, rather than the whole matrix. Based on the observation by \citet{JMLR:v22:20-410} that the spectrum of model weights gets more heavy-tailed during training, we initially expected the same for Hessian; However, we found that although the Hessian trace is increasing (Figure~\ref{fig:beta2_est} Third Row, in \S~\ref{sec:appropriate_beta2}), the exponent $\alpha$ of the power-law fit is actually slightly \emph{decreasing} for Normal training (Figure~\ref{fig:powerlaw}, in \S~\ref{sec:power_law_dist}). Still, optimization hyperparameters such as $\beta_2$ of Adam clearly affect the distribution of eigenvalues of Hessian; how to understand its mechanism remains a largely open problem.

\paragraph{Scaling Statistics of Language Universals} It is not a coincidence that the eigenvalues of correlation matrices of forward activations of an LM are power-law distributed, because natural language data generally follow power-law distributions like Zipf's Law \citep{zipf1949human}. In fact, if we view the input embedding layer of a Transformer LM as a linear map from one-hot vectors to token embeddings, the correlation matrix of the inputs is $\diag((p_i))$ where $p_i$ is the unigram probability of the $i$-th token. So the distribution of the eigenvalues of $\diag((p_i))$ is exactly Zipf's Law, and the power-law exponent is exactly $\alpha=0.5$ in terms of the definition of $\alpha$ in PLRF (\S~\ref{sec:background}). Scaling statistics of language universals have long been studied \citep{tanaka2021statistical}, and we believe they are generally behind various neural scaling laws for training LLMs \citep{explain_neural_scaling_laws}.

\paragraph{Double Descent} The idea of distinguishing data-oriented and model-oriented randomness, as we did for training spikes, has also appeared in the explanation of Double Descent by Random Feature models \citep{3495724.3496649}. From a very general point of view, this is not a coincidence because Double Descent can be seen as some kind of ``spikes'' as well. Indeed, \citet{Nakkiran2020Deep} has coined the term ``epoch-wise double descent'' and used the concept to explain temporary increase of validation loss at training time.

% \paragraph{Efficient Rank}

\paragraph{The $\beta_2$ of Adam Optimizer} It has been widely known to practitioners that setting $\beta_2$ close to $1$ can cause training spikes in Transformer LMs \citep[e.g.~][]{liu2019robertarobustlyoptimizedbert}. \citet{DBLP:journals/corr/abs-2304-09871} and \citet{bai2026adaptive} specifically build their theory of training spikes upon this observation. In this paper, we also propose our recipe for choosing the appropriate $\beta_2$ schedule (\S~\ref{sec:appropriate_beta2}), and we show it improves performance in our case. The observation that an increasing schedule of $\beta_2$ can improve performance has been reported before \citep[e.g.~][]{on_convergence_of_adam_and_beyond,pmlr-v80-shazeer18a,dremov2025training,ferbach2026logarithmictimeschedulesscalinglanguage}. In practice, however, model performance and the occurrence of training spikes are affected by a large number of other factors as well. Nevertheless, the exact mechanism of how $\beta_2$ can affect the training dynamics in such a seemingly drastic way, especially how it affects the distribution of eigenvalues of the Hessian, is a curious mystery.

\paragraph{Optimization and Modeling Techniques to Improve Training Stability}

A huge amount of modeling and optimization techniques, such as normalization \citep{pmlr-v37-ioffe15,ba2016layernormalization,xiong2020layernormalizationtransformerarchitecture,pmlr-v202-dehghani23a,wortsman2024smallscale} and clipping \citep{pmlr-v28-pascanu13,wortsman2023stable,Zhang2020Why,marshall2024clipclipdynamicssgd}, \emph{inter alia}, have been proposed to improve the training stability of DNNs. From the perspective of (avoiding) training spikes, our theory suggests essential metrics, such as the power-law exponent of the distribution of eigenvalues of the correlation matrices of forward activations, to be monitored as an indicator of healthy training. Our lightweight simulation algorithm can be applied directly to testing optimization techniques as well. That said, the empirical recipes explored in this work such as $\beta_2$ schedule, are intended as ``proof of concept'' for the theory, rather than alternatives to replace previous methods.

%\paragraph{Worst-case Analysis of Optimization Techniques} clipping etc.

\paragraph{Gradient Noise and Average-case Analysis} Even a simple model like PLRF can exhibit complicated dynamics such as training spikes. We urge the community to move from analyzing full-batch gradient descent and worst-case bounds, to the more realistic and challenging settings of stochastic optimization and average-case behavior of DNN training, with heavy-tailed input data and scaling-up of model sizes. Towards this direction, \citet{pmlr-v48-mandt16,3322706.3361981} have developed stochastic differential equations to capture the training dynamics of SGD, which assume that the gradient noise is Gaussian (or more precisely, Brownian motion since the process is continuous), as similar to what we do for the simulation. What's different in this work is that the variance of gradient noise can be explicitly calculated for Random Feature models (i.e., Prediction~\ref{pred:grad_norm}), and we use that calculation instead of assuming the variance is a given constant (it's not constant, even for Random Feature models). On a different note, \citet{agarwala2025highdimensionalanalysisreveals} has coined the term ``stochastic edge-of-stability'', to emphasize the difference in learning-rate thresholds between full-batch gradient descent and SGD. In \citet{qiu2025scaling}, it is shown that the gradient noise is a major factor controlled by learning-rate schedules, that can shape universally similar loss trajectories across model scales.

\section{Conclusion}

We have seen PLRF models exhibiting training spikes, and presented the theory, algorithm and experiments to align PLRF models with layerwise training dynamics of DNNs. Given the apparent similarity between empirical training spikes and our simulations, we believe quadratic models at least partially reveals the essence of underlying mechanism. %Due to the page limit, we put more discussion and related works in Appendix~\ref{sec:discussion}.

%we put more discussion and related works in the Appendix (\S~\ref{sec:discussion}).

\section*{AI Use Statement}

We have not used generative AI tools for developing theory, formulating mathematical claims, designing experiments, or implementing methods. This is completely human work. During the writing of this paper, we have consulted AI for rephrasing some paragraphs of the draft, and summarizing some related papers; we have checked the outputs and manually rewrote the text. We take responsibility for the final content of this work, including text, claims, mathematical proofs and code.

\section*{Reproducibility Statement}

Our experiments with PLRF models are easy to reproduce; we provide the Python code for \href{https://colab.research.google.com/drive/19S5a-YPFcdWcz02nOadGcKkUOCA-tF_B?usp=sharing}{\underline{training spikes}} and \href{https://colab.research.google.com/drive/1QFUCui4N9UzDpMJSuDiOfp6j0GSlMJyg?usp=sharing}{\underline{the simulation}} in supplementary materials. Our Transformer LM experiments are based on a decoder-only Transformer with common components. All the hyperparameters are listed in \S~\ref{sec:llm_train_settings}, where we also describe how they were tuned. The training dynamics we present in this paper depends on the hyperparameters, especially the learning-rate schedule and weight-decay; But the phenomena we observe should be general across different modeling choices. For mathematical details of our theory, please refer to Appendix~\ref{sec:proof_of_thm}. 

\section*{Acknowledgment}

We thank Vincent Roulet for contributing to the ideas of how to calculate per-example gradient norms and Hessian trace, and providing internal reviews on drafts of this paper. We thank Asma Ghandeharioun for applying and verifying some findings of this paper to Gemma models. Asma Ghandeharioun uses generative AI tools to manage, implement and analyze her experiments, but she is not directly involved in this work. We thank Justin Gilmer for pointing us to several highly related previous works. We thank Lechao Xiao for providing the hyperparameters of Transformer LMs, including the configuration of training spikes. We thank Jeffrey Pennington for helpful discussion, and contributing to the ideas of distinguishing data-oriented and model-oriented randomness, and quantitatively measuring training spikes.

C. Paquette is a Canadian Institute for Advanced Research (CIFAR) AI
chair, Quebec AI Institute (Mila) and a Sloan Research Fellow in
Computer Science (2024). C. Paquette is supported by a Discovery Grant from the
 Natural Science and Engineering Research Council (NSERC) of Canada,
 NSERC CREATE grant Interdisciplinary Math and Artificial Intelligence
 Program (INTER-MATH-AI), Google x Mila research grant, Fonds de recherche du Quebec Nature et technologies (FRQNT) NOVA Grant (DOI: \href{https://doi.org/10.69777/363444}{https://doi.org/10.69777/363444}), and CIFAR AI Catalyst Grant. C. Paquette has 20\% part-time employment at Google DeepMind. 

\bibliography{main}
\bibliographystyle{plainnat}

\newpage
\appendix

%\section*{\Large Appendix}

\section{Theory and Proof}
\label{sec:proof_of_thm}

In this section, we give details of our theory and proof. For notation, we use bold lowercase (e.g.~$\cvec{h}$, $\vvec{w}$, $\vvec{y}$) to denote column vectors, and bold uppercase (e.g.~$\cmat{A}$, $\vmat{W}$) for matrices. We also make the convention to denote trainable tensors (e.g., $\vvec{w}$, $\vmat{W}$, and $\vvec{y}$ which depends on $\vmat{W}$) by italic, and non-trainable constant tensors (e.g.~$\cvec{h}$, $\cmat{A}$, $\cvec{x}$) otherwise. All function-calls are denoted exclusively by the $[\bullet]$ bracket, in order to distinguish from the ordinary parentheses $()$. Derivatives (i.e., gradients) can be represented by either $\frac{\partial\ell}{\partial\vvec{y}}$ or $\nabla\vvec{y}$, but we impose the convention that $\frac{\partial\ell}{\partial\vvec{y}}$ is a \emph{row} vector while $\nabla\vvec{y}$ is a column vector. Therefore, we have
\[
\frac{\partial\ell}{\partial\transpose{\vvec{y}}}=\nabla\vvec{y}
\]
where $\top$ is used to denote transpose. Under this convention, the 1st-order Taylor expansion can be written as
\[
\func{\ell}{\vvec{w}}\approx\func{\ell}{\cvec{w}_0}+\func{\frac{\partial\ell}{\partial\vvec{w}}}{\cvec{w}_0}\big(\vvec{w}-\cvec{w}_0\big).
\]
Also in this paper, an input data point to DNNs is exclusively denoted by $\mathfrak{x}$, and considered as having no internal structure (i.e., $\mathfrak{x}$ can be anything, not necessarily a tensor). And we assume that the training and validation datasets are both finite, and drawn from the same distribution $\mathcal{D}$.

\subsection{Proof of Theorem~\ref{thm:dnn2rf}}

Technically, we prove Theorem~\ref{thm:dnn2rf} by taking Taylor expansions of the DNN. The ``difficulty'' is that Taylor expansions can only be taken for the \emph{point-wise} loss $\func{\ell}{\mathfrak{x};\mathcal{W}}$, for a given data point $\mathfrak{x}$; While the Random Feature model requires a universal constant matrix $\cmat{A}$. Our proof shows that this is still possible, by ``concatenating'' all the data points in the training set into a huge linear map, exploiting the high dimensionality $\cscalar{V}$ of the latent space.

\begin{proof}[Proof of Theorem~\ref{thm:dnn2rf}]
%\paragraph{a} 
For \emph{(a)}, since the Jacobian matrix $\func{\cmat{J}}{\mathfrak{x}}\eqdef\func{\frac{\partial\cvec{h}}{\partial\mathcal{W}^\mathrm{F}}}{\mathfrak{x};\mathcal{W}^\mathrm{F}_{t}}$ depends on $\mathfrak{x}$ but does not depend on $s$, we can construct a matrix $\tilde{\cmat{A}}$ by concatenating the $\func{\cmat{J}}{\mathfrak{x}}$ (on axis 1) for all training examples $\mathfrak{x}\in\mathcal{D}$, and a vector $\tilde{\cvec{x}}$ by padding the flattened vector $(\mathcal{W}^\mathrm{F}_{t+s}-\mathcal{W}^\mathrm{F}_{t})$ with zeros, to obtain
\[
\Big(\func{\frac{\partial\cvec{h}}{\partial\mathcal{W}^\mathrm{F}}}{\mathfrak{x};\mathcal{W}^\mathrm{F}_{t}}\Big)(\mathcal{W}^\mathrm{F}_{t+s}-\mathcal{W}^\mathrm{F}_{t})=\tilde{\cmat{A}}\tilde{\cvec{x}}
\]
where $\tilde{\cmat{A}}$ does not depend on $\mathfrak{x}$ or $s$. Then, we define $\cmat{A}$ by further concatenating $\tilde{\cmat{A}}$ with an identity matrix (on axis 1), and $\cvec{x}$ by concatenating $\tilde{\cvec{x}}$ with $\cvec{h}_{t}$, so that $\cvec{h}_{t+s}$ has the 1st-order Taylor expansion 
\[
\cvec{h}_{t+s}\approx\cvec{h}_{t}+\Big(\func{\frac{\partial\cvec{h}}{\partial\mathcal{W}^\mathrm{F}}}{\mathfrak{x};\mathcal{W}^\mathrm{F}_{t}}\Big)(\mathcal{W}^\mathrm{F}_{t+s}-\mathcal{W}^\mathrm{F}_{t})=\cmat{A}\cvec{x},
\]
as required.

%\paragraph{b} 
For \emph{(b)}, recall from the previous paragraph that we have $\tilde{\cmat{A}}\cvec{x}\approx\cvec{h}_{t+s}-\cvec{h}_{t}$, and $\cvec{h}_{t}$ is contained in $\cvec{x}$, so we can write
%we take $\vvec{y}_{t+s}-\vvec{y}_{t}=\vmat{W}_{t+s}\cvec{h}_{t+s}-\vmat{W}_{t}\cvec{h}_{t}$ and recall from the previous paragraph that we have $\cvec{h}_{t+s}\approx\cmat{A}\cvec{x}$, while $\cvec{h}_{t}$ is contained in $\cvec{x}$, so we can write
\[
\vvec{y}_{t+s}-\vvec{y}_{t}=\vmat{W}_{t+s}(\cvec{h}_{t+s}-\cvec{h}_{t})+(\vmat{W}_{t+s}-\vmat{W}_{t})\cvec{h}_{t}\approx(\vmat{W}_{t+s}\tilde{\cmat{A}}+\widetilde{\Delta\vmat{W}})\cvec{x}
\]
where $\widetilde{\Delta\vmat{W}}$ is $(\vmat{W}_{t+s}-\vmat{W}_{t})$ padded with zeros (on axis 1), such that $\widetilde{\Delta\vmat{W}}\cvec{x}=(\vmat{W}_{t+s}-\vmat{W}_{t})\cvec{h}_{t}$.

Without changing the calculation so far, we can further modify $\cvec{x}$ by concatenating $(\mathcal{W}^\mathrm{B}_{t+s}-\mathcal{W}^\mathrm{B}_{t})$. Then, since both $\func{\frac{\partial^2\ell}{\partial\transpose{\vvec{y}}\,\partial\vvec{y}}}{\mathfrak{x};\vvec{y}_{t},\mathcal{W}^\mathrm{B}_{t}}$ and $\func{\frac{\partial^2\ell}{\partial\transpose{\vvec{y}}\,\partial\mathcal{W}^\mathrm{B}}}{\mathfrak{x};\vvec{y}_{t},\mathcal{W}^\mathrm{B}_{t}}$ are matrices that depend on $\mathfrak{x}$ but do not depend on $s$, we concatenate them (on axis 1) for all training examples $\mathfrak{x}\in\mathcal{D}$ and define
\[
\tilde{\vmat{G}}_{t+s}\eqdef\cmat{B}\,\Big(\textstyle\bigoplus_{k=1}^{\cscalar{n}}(\vmat{W}_{t+s}\widetilde{\cmat{J}_{k}}+\widetilde{\Delta\vmat{W}})\Big)+\tilde{\cmat{C}}
\]
where $\cmat{B}$ is the concatenation of $\func{\frac{\partial^2\ell}{\partial\transpose{\vvec{y}}\,\partial\vvec{y}}}{\mathfrak{x};\vvec{y}_{t},\mathcal{W}^\mathrm{B}_{t}}$, and $\tilde{\cmat{C}}$ the concatenation of $\func{\frac{\partial^2\ell}{\partial\transpose{\vvec{y}}\,\partial\mathcal{W}^\mathrm{B}}}{\mathfrak{x};\vvec{y}_{t},\mathcal{W}^\mathrm{B}_{t}}$. The $\bigoplus_{k=1}^{\cscalar{n}}$ denotes concatenation of matrices on axis 0, with $\cscalar{n}=\mbox{\#\,training-examples}$, and $\widetilde{\cmat{J}_{k}}$ denotes the Jacobian matrix $\func{\cmat{J}}{\mathfrak{x}^{(k)}}$ padded with zeros, for the $k$-th training example $\mathfrak{x}^{(k)}\in\mathcal{D}$. (Such that, $\sum_{k=1}^{\cscalar{n}}\widetilde{\cmat{J}_{k}}=\tilde{\cmat{A}}$ for $\tilde{\cmat{A}}$ in the previous paragraph). The matrices $\cmat{B}$, $\tilde{\cmat{C}}$, $\widetilde{\cmat{J}_{k}}$ and $\widetilde{\Delta\vmat{W}}$ do not depend on $\mathfrak{x}$ or $s$, while the construction of $\tilde{\vmat{G}}_{t+s}$ implies
\[
\Big(\func{\frac{\partial^2\ell}{\partial\transpose{\vvec{y}}\,\partial\vvec{y}}}{\mathfrak{x};\vvec{y}_{t},\mathcal{W}^\mathrm{B}_{t}}\Big)(\vvec{y}_{t+s}-\vvec{y}_{t})+\Big(\func{\frac{\partial^2\ell}{\partial\transpose{\vvec{y}}\,\partial\mathcal{W}^\mathrm{B}}}{\mathfrak{x};\vvec{y}_{t},\mathcal{W}^\mathrm{B}_{t}}\Big)(\mathcal{W}^\mathrm{B}_{t+s}-\mathcal{W}^\mathrm{B}_{t})\approx\tilde{\vmat{G}}_{t+s}\cvec{x}.
\]
Finally, we can modify $\cvec{x}$ by concatenating with $\nabla\vvec{y}_{t}$, and define $\cmat{C}$ by concatenating $\tilde{\cmat{C}}$ with an identity matrix (on axis 1) to absorb the constant term, and obtain
\[
\vmat{G}_{t+s}\eqdef\cmat{B}\,\Big(\textstyle\bigoplus_{k=1}^{\cscalar{n}}(\vmat{W}_{t+s}\widetilde{\cmat{J}_{k}}+\widetilde{\Delta\vmat{W}})\Big)+\cmat{C}
\]
such that the 1st-order Taylor expansion of $\nabla\vvec{y}_{t+s}$ is given by  $\vmat{G}_{t+s}\cvec{x}$. By construction, $\vmat{G}_{t+s}$ is an affine function of $\vmat{W}_{t+s}$.
%For \emph{(c)}, it is easy to see from the above construction that $\cmat{A}\cvec{x}_{0}=\cvec{h}_{t}$, and $\vmat{G}_{t+s}\cvec{x}_{0}$ is exactly the approximation of $\func{\frac{\partial\ell}{\partial\vvec{y}}}{\mathfrak{x};\vmat{W}_{t+s}\cvec{h}_{t},\mathcal{W}^\mathrm{B}_{t}}$ as obtained from the 2nd-order Taylor expansion at $\mathcal{W}_{t}$. So the statement follows.
\end{proof}

\subsection{Proof of Prediction~\ref{pred:grad_norm}}
\label{sec:proof_of_pred}

We have $\cvec{h}_{t}\approx\cmat{A}\cvec{x}$ where the entries of $\cvec{x}$ are i.i.d.~Gaussian of mean $0$ and variance $1$, by Assumption~\ref{ass:iid}. So
\[
\mathbb{E}(\cvec{h}_{t}\transpose{\cvec{h}}_{t})\approx\mathbb{E}(\cmat{A}\cvec{x}\transpose{\cvec{x}}\transpose{\cmat{A}})=\cmat{A}\transpose{\cmat{A}}.
\]
Therefore, the eigen decomposition of $\mathbb{E}(\cvec{h}_{t}\transpose{\cvec{h}}_{t})$ comes from the singular value decomposition
\[
\cmat{A}=\cmat{U}\diag(\sqrt{\vvec{\lambda}})\transpose{\cmat{V}},
\]
where $\sqrt{\vvec{\lambda}}$ is the entry-wise square-root of $\vvec{\lambda}$, and $\cmat{V}$ is a $\cscalar{V}\times\cscalar{D}$ matrix with orthonormal columns.

Now, $\vvec{g}_i$ is the $i$-th column vector of the $\cscalar{C}\times\cscalar{D}$ matrix 
\[
(\nabla\vmat{W}_{t+s})\cmat{U}=(\nabla\vvec{y}_{t+s})\,\transpose{\cvec{h}}_{t+s}\cmat{U}\approx\vmat{G}_{t+s}\cvec{x}\,\transpose{\cvec{x}}\transpose{\cmat{A}}\cmat{U}=\vmat{G}_{t+s}\cvec{x}\,\transpose{\cvec{x}}\cmat{V}\diag(\sqrt{\vvec{\lambda}}),
\]
So
\begin{equation}
\label{eq:gi}
\vvec{g}_i\approx\sqrt{\lambda_i}\vmat{G}_{t+s}\cvec{x}\,\transpose{\cvec{x}}\cmat{V}\cvec{e}_i
\end{equation}
where $\lambda_i$ is the $i$-th entry of $\vvec{\lambda}$ and $\cvec{e}_i$ is the $i$-th unit vector ($1\leq i\leq\cscalar{D}$).

On the other hand, we have $\nabla\vvec{y}_{t+s}\approx\vmat{G}_{t+s}\cvec{x}$
so
\begin{equation}
\label{eq:nablay}
\mathbb{E}(\nabla\vvec{y}_{t+s}\nabla\transpose{\vvec{y}}_{t+s})\approx\mathbb{E}(\vmat{G}_{t+s}\cvec{x}\transpose{\cvec{x}}\transpose{\vmat{G}}_{t+s})=\vmat{G}_{t+s}\transpose{\vmat{G}}_{t+s}.
\end{equation}
Thus, Prediction~\ref{pred:grad_norm} follows by combining the following lemma with \eqref{eq:gi} and \eqref{eq:nablay}. 

\begin{lemma}\label{lem:e2x2}
Assume $\cmat{V}$ is ``general enough''. Then, we have
\[
\mathbb{E}\big(\cvec{x}\,\transpose{\cvec{x}}\cmat{V}\cvec{e}_i\transpose{\cvec{e}_j}\transpose{\cmat{V}}\cvec{x}\,\transpose{\cvec{x}}\big)=\delta_{ij}\cmat{I}_{\cscalar{V}}
\]
at the $\cscalar{D}\to\infty$ limit. Here, $\delta_{ij}$ is the Kronecker delta and $\cmat{I}_{\cscalar{V}}$ is the $\cscalar{V}\times\cscalar{V}$ identity matrix.
\end{lemma}
\begin{proof}
The $(k,n)$-entry of the matrix $\cvec{x}\,\transpose{\cvec{x}}\cmat{V}\cvec{e}_i\transpose{\cvec{e}_j}\transpose{\cmat{V}}\cvec{x}\,\transpose{\cvec{x}}$ is a linear combination of random variables of the form
\[
x_k x_l x_m x_n,
\]
where $x_k, x_l, x_m, x_n$ are entries of $\cvec{x}$. Now recall Wick's formula (Isserlis's theorem), %\footnote{\url{https://en.wikipedia.org/wiki/Isserlis\%27s_theorem}}), 
which implies that in order to have
\[
\mathbb{E}(x_k x_l x_m x_n)\neq0,
\]
there are only three possibilities:
\begin{enumerate}[leftmargin=3.5em, labelwidth=2.5em, labelsep=1em]
\item $k=n$, $l=m$;
\item $k=m$, $l=n$;
\item $k=l$, $m=n$.
\end{enumerate}
When restricted to Case 1, only the diagonal entries (such that $k=n$) in $\mathbb{E}\big(\cvec{x}\,\transpose{\cvec{x}}\cmat{V}\cvec{e}_i\transpose{\cvec{e}_j}\transpose{\cmat{V}}\cvec{x}\,\transpose{\cvec{x}}\big)$ are nonzero, and they equal to $\transpose{\cvec{e}_j}\transpose{\cmat{V}}\cmat{V}\cvec{e}_i$ which is $\delta_{ij}$ because the column vectors of $\cmat{V}$ are orthonormal. So this case results in a $\delta_{ij}\cmat{I}_{\cscalar{V}}$ term.

For both Cases 2 and 3, they result in terms of the form $\cmat{V}\cvec{e}_i\transpose{\cvec{e}_j}\transpose{\cmat{V}}$, which are projectors that map any $\cscalar{V}$-dimensional vector into a 1-dimensional subspace (i.e., $\cmat{V}\cvec{e}_i$) of the $\cscalar{D}$-dimensional space spanned by the column vectors of $\cmat{V}$. When $\cmat{V}$ is ``general enough'', this is like a random orthogonal projection from a $\cscalar{D}$-dimensional space onto its 1-dimensional subspace; So ``on average'' these terms are of the magnitude $1/\sqrt{\cscalar{D}}$, which become negligible compared to $\cmat{I}_{\cscalar{V}}$ when $\cscalar{D}$ is sufficiently large.
%$\transpose{\cmat{V}}\cvec{e}_j\transpose{\cvec{e}_j}\cmat{V}$, which is the self- Kronecker product of the $j$-th row vector of $\cmat{V}$. Now if $\cmat{V}$ is general enough, each entry of $\transpose{\cvec{e}_j}\cmat{V}$ should be roughly of the magnitude of $1/\sqrt{\cscalar{V}}$ (because their squares sum to $1$), and if $\cscalar{V}$ is sufficiently large, this $\transpose{\cmat{V}}\cvec{e}_j\transpose{\cvec{e}_j}\cmat{V}$ term should be negligible compared to the identity matrix $\cmat{I}_{\cscalar{V}}$. 

Rigorously, this ``general enough'' argument can be conveyed by assuming a Haar measure over the space of $\cmat{V}$ and bounds the error with ``overwhelming probability'', as similar to \citet[Appendix C]{sgd_in_large}.
\end{proof}
Although our proof of Lemma~\ref{lem:e2x2} relies on specifics of the Gaussian distribution (such as Wick's formula), the conclusion on high dimension limit is likely to hold for other random vectors (asymptotically) of mean $0$ and variance $1$.

\subsection{Equivalent Batch Factor}
\label{sec:equiv_batch_factor_derivation}

In this section, we discuss how to apply Prediction~\ref{pred:grad_norm} to the calculation of 2nd moment of \emph{mini-batched} gradients of \emph{DNN}s.

For stochastic training with batch-size $\cscalar{B}$, we use $\nabla_\cscalar{B}$ to denote the \emph{per-batch} gradient. Then, instead of \eqref{eq:gi}, we have
\begin{equation}
\label{eq:bgi}
\vvec{g}^\cscalar{B}_i\approx\frac{1}{\cscalar{B}}\sum_{k=1}^\cscalar{B}\sqrt{\lambda_i}\vmat{G}_{t+s}\cvec{x}^{(k)}\,\transpose{\cvec{x}^{(k)}}\cmat{V}\cvec{e}_i
\end{equation}
where $\vvec{g}^\cscalar{B}_i$ is the $i$-th column vector of the $\cscalar{C}\times\cscalar{D}$ matrix $(\nabla_\cscalar{B}\vmat{W}_{t+s})\cmat{U}$, and $\cvec{x}^{(k)}$ the $k$-th data point in a batch. Note the normalization factor $1/\cscalar{B}$ which is to make $\mathbb{E}(\nabla\vmat{W})=\mathbb{E}(\nabla_\cscalar{B}\vmat{W})$. We can generalize Lemma~\ref{lem:e2x2} to this mini-batched setting, as below.
\begin{lemma}\label{lem:e2bx2}
Assume $\cmat{V}$ is ``general enough'', and $\cscalar{B}^2\ll\cscalar{D}$. Then, we have
\[
\mathbb{E}\Big(\sum_{k=1}^\cscalar{B}\big(\cvec{x}^{(k)}\,\transpose{\cvec{x}^{(k)}}\big)\cmat{V}\cvec{e}_i\transpose{\cvec{e}_j}\transpose{\cmat{V}}\sum_{k=1}^\cscalar{B}\big(\cvec{x}^{(k)}\,\transpose{\cvec{x}^{(k)}}\big)\Big)=\cscalar{B}\,\delta_{ij}\cmat{I}_{\cscalar{V}}
\]
at the $\cscalar{D}\to\infty$ limit. Here, $\delta_{ij}$ is the Kronecker delta and $\cmat{I}_{\cscalar{V}}$ is the $\cscalar{V}\times\cscalar{V}$ identity matrix.
\end{lemma}
\begin{proof}
The proof is mostly similar to Lemma~\ref{lem:e2x2}, except that for Case 3, we have $\cscalar{B}^2$ many terms of the form $\cmat{V}\cvec{e}_i\transpose{\cvec{e}_j}\transpose{\cmat{V}}$, due to cross-interaction inside the batch. Therefore, an additional condition $\cscalar{B}^2\ll\cscalar{D}$ is required to ensure that the batch-size is not too large, so that Case 3 is still negligible.
\end{proof}
Therefore, we can calculate
\begin{equation}
\label{eq:b2egg}
\mathbb{E}(\vvec{g}^\cscalar{B}_i\transpose{\vvec{g}^\cscalar{B}_j})\approx\frac{1}{\cscalar{B}}\,\mathbb{E}(\vvec{g}_i\transpose{\vvec{g}}_j)
\end{equation}
as long as the batch-size $\cscalar{B}$ is not too large ($\cscalar{B}\ll\sqrt{\cscalar{D}}$). This formula is compatible with the intuition that increasing the batch-size will decrease gradient noise.

Now, regarding $\mathbb{E}(\nabla\vvec{y}_{t+s}\nabla\transpose{\vvec{y}}_{t+s})$, for Random Feature models where $y_{t+s}$ is a scalar we have
\[
\nabla y_{t+s}=c\big(\transpose{\vvec{w}}_{t+s}\cmat{A}\cvec{x}-\func{b}{\cvec{x}}\big)
\]
hence
\[
\mathbb{E}\big((\nabla y_{t+s})^2\big)=\mathbb{E}\Big(c^2\big(\transpose{\vvec{w}}_{t+s}\cmat{A}\cvec{x}-\func{b}{\cvec{x}}\big)^2\Big)=2c\,\mathbb{E}\big(\func{\ell}{\cvec{x};\vvec{w}_{t+s}}\big),
\]
which is a multiple of the validation loss.

However, the above simplification does \emph{not} hold for DNNs in general. Even if we use Assumption~\ref{ass:trace} to reduce $y_{t+s}$ to a scalar, the \emph{point-wise} loss of DNNs is of the form
\begin{equation}
\label{eq:dnn_pointwise_loss}
\func{\ell}{\mathfrak{x};\vvec{w}}\approx\frac{\func{c}{\mathfrak{x}}}{2}\big(\transpose{\vvec{w}}\cmat{A}\cvec{x}-\func{b}{\mathfrak{x}}\big)^2\quad\mbox{and}\quad\nabla y_{t+s}\approx\func{c}{\mathfrak{x}}\big(\transpose{\vvec{w}}_{t+s}\cmat{A}\cvec{x}-\func{b}{\mathfrak{x}}\big).
\end{equation}
% as an approximation to , hence
% \begin{equation}
% \label{eq:dnn_nablay}
% \nabla y_{t+s}=\func{c}{\mathfrak{x}}\big(\transpose{\vvec{w}}_{t+s}\cmat{A}\cvec{x}-\func{b}{\mathfrak{x}}\big).
% \end{equation}
Especially, the coefficient $c$ in \eqref{eq:dnn_pointwise_loss} %and \eqref{eq:dnn_nablay} 
is different across data points $\mathfrak{x}$. 

Nevertheless, in our simulation we estimate $c$ as an average of $\func{c}{\mathfrak{x}}$, such that the Hessian trace of the layer-wise quadratic loss $\func{L}{\vvec{w}_t}$ matches the empirical measurement $\trace\big(\mathbb{E}(\nabla^2\vmat{W}_{t})\big)$. And we still assume that $\mathbb{E}\big((\nabla y_{t+s})^2\big)$ is a multiple of $2c\func{L}{\vvec{w}_{t+s}}$, as a rough approximation. But the coefficient of that multiple needs to be empirically measured, and it would be much larger than 1 (as in contrast to the case of Random Feature models) because $\mathbb{E}(\func{c}{\mathfrak{x}}^2)$ is usually much larger than $\mathbb{E}(\func{c}{\mathfrak{x}})^2$. In this work, we absorb this multiple into a ``batch-size factor'', because it affects the norm of per-batch gradients as similar to a batch-size.

In practice, we do not directly measure $\func{L}{\vvec{w}_t}$ either, because the measurement of the Hessian tensor is costly if not infeasible. Instead, we have $\lVert\func{\frac{\partial L}{\partial\vvec{w}}}{\vvec{w}_t}\rVert^2=\lVert\mathbb{E}(\nabla\vmat{W}_{t})\rVert^2$ by definition, and $\lVert\mathbb{E}(\nabla\vmat{W}_{t})\rVert^2$ is directly measurable. So it is favorable if $\func{L}{\vvec{w}_t}$ can be approximated by $\lVert\func{\frac{\partial L}{\partial\vvec{w}}}{\vvec{w}_t}\rVert^2$.

Recall that $\func{L}{\vvec{w}}$ is a quadratic function of $\vvec{w}$ which can be decomposed into eigenmodes. Let
\[
\func{L}{\vvec{w}}=\frac{c}{2}\sum_{i=1}^\cscalar{D}\lambda_i(w_i-\cscalar{w}^\ast_i)^2
\]
be the eigen decomposition, where $(w_i)_{1\leq i\leq\cscalar{D}}$ and $(\cscalar{w}^\ast_i)_{1\leq i\leq\cscalar{D}}$ are representations of $\vvec{w}$ and $\cvec{w}^\ast$ under the eigen basis, respectively. Then, we have
\[
\func{\frac{\partial L}{\partial\transpose{\vvec{w}}}}{\vvec{w}}=\big(c\lambda_i(w_i-\cscalar{w}^\ast_i)\big)_{1\leq i\leq\cscalar{D}}\quad\mbox{and}\quad\Big\lVert\func{\frac{\partial L}{\partial\vvec{w}}}{\vvec{w}}\Big\rVert^2=c^2\sum_{i=1}^\cscalar{D}\lambda_i^2(w_i-\cscalar{w}^\ast_i)^2.
\]
Thus, the following quantity
\begin{equation}
\label{eq:lapprox}
\frac{\sum_{i=1}^\cscalar{D}\lambda_i}{\sum_{i=1}^\cscalar{D}\lambda_i^2}\big\lVert\mathbb{E}(\nabla\vmat{W}_{t})\big\rVert^2
\end{equation}
would be an easy-to-measure approximation for $2c\func{L}{\vvec{w}_{t}}$. This approximation should work well if either of the conditions below is satisfied:
\begin{enumerate}
\item[(1)] For most $i$, values of $(w_i-\cscalar{w}^\ast_i)^2$ are approximately of the same magnitude;
\item[(2)] Or, the values of $(w_i-\cscalar{w}^\ast_i)^2$ are ``random enough'', the distribution of $(\lambda_i)$ not too heavy-tailed, and $\cscalar{D}$ sufficiently large.
\end{enumerate}

To summarize, we have assumed that $\mathbb{E}\big((\nabla y_{t+s})^2\big)$ is a multiple of $2c\func{L}{\vvec{w}_{t+s}}$, and propose to use \eqref{eq:lapprox} as an approximation for $2c\func{L}{\vvec{w}_{t}}$. Combining this with \eqref{eq:b2egg} and Prediction~\ref{pred:grad_norm}, we predict
\begin{equation}
\label{eq:bgrad_norm}
\mathfrak{B}\,\mathbb{E}(\lVert\nabla_\cscalar{B}\vmat{W}_{t+s}\lVert^2)\approx\sum_{i=1}^\cscalar{D}\lambda_i\,2c\func{L}{\vvec{w}_{t+s}}
\end{equation}
where $\mathfrak{B}$ is empirically measured as
\begin{equation}\label{eq:equiv_batch_factor_def}
\mathfrak{B}\eqdef\frac{(\sum_{i=1}^\cscalar{D}\lambda_i)^2}{\sum_{i=1}^\cscalar{D}\lambda_i^2}\dfrac{\lVert\mathbb{E}(\nabla\vmat{W}_{t})\rVert^2}{\mathbb{E}(\lVert\nabla_\cscalar{B}\vmat{W}_{t}\lVert^2)}.
\end{equation}
We call $\mathfrak{B}$ the \emph{equivalent batch factor}, because it should coincide with the batch-size $\cscalar{B}$ for Random Feature models.

\subsection{Detailed Derivation of Algorithm~\ref{alg:simulation}}
\label{sec:derive_alg}

We track the vector $\vvec{z}=(z_i)_{1\leq i\leq\cscalar{D}}$ where $z_i\eqdef c(w_i-\cscalar{w}^\ast_i)$, with $(w_i)_{1\leq i\leq\cscalar{D}}$ and $(\cscalar{w}^\ast_i)_{1\leq i\leq\cscalar{D}}$ being the representations of $\vvec{w}$ and $\cvec{w}^\ast$ under the eigen basis, respectively. Since 
\[
\func{L}{\vvec{w}}=\frac{1}{2c}\sum_{i=1}^\cscalar{D}\lambda_i z_i^2,
\]
the sign of each $z_i$ does not matter -- our simulation only depends on the absolute value $\lvert z_i\rvert$. Initially, $\lvert z_i\rvert$ is set to match the Root Mean Square (RMS) of the gradient $\mathbb{E}(\nabla\vmat{W}_{t})$:
\begin{itemize}
 \item The coordinates of $\mathbb{E}(\nabla\vmat{W}_{t})$ are changed to the eigen basis by $\mathbb{E}(\nabla\vmat{W}_{t})\cmat{U}$;
 \item The shape of $\mathbb{E}(\nabla\vmat{W}_{t})\cmat{U}$ is $\cscalar{C}\times\cscalar{D}$; We reduce the $\cscalar{C}$-dimensional output axis by taking $\rms\big(\mathbb{E}(\nabla\vmat{W}_{t})\cmat{U},\,\texttt{axis}=0\big)$;
 \item The $i$-th entry of $\rms\big(\mathbb{E}(\nabla\vmat{W}_{t})\cmat{U},\,\texttt{axis}=0\big)$ should be the same as
 \[
 \Big\lvert\func{\frac{\partial L}{\partial w_i}}{\vvec{w}}\Big\rvert=\lambda_i\lvert z_i\lvert,
 \]
 in order to match the gradient at each eigenmode.
 \item Therefore, we set
 \[
 \vvec{z} \gets \rms\big(\mathbb{E}(\nabla\vmat{W}_{t})\cmat{U},\,\texttt{axis}=0\big)/\vvec{\lambda}
 \]
 where ``$/$'' denotes element-wise division.
\end{itemize}

At each simulation step, let $\vvec{g}=(g_k)_{1\leq k\leq\cscalar{D}}$ be the synthetic per-batch gradient under the eigen basis. The mean $\mathbb{E}(g_k)$ is given by
\begin{equation}\label{eq:synth_grad_mean}
\mathbb{E}(g_k)=\func{\frac{\partial L}{\partial w_k}}{\vvec{w}}=\lambda_k z_k,
\end{equation}
and the 2nd moment
\begin{equation}\label{eq:synth_grad_2nd}
\mathbb{E}(g_k^2)\approx\lambda_k\,2c\func{L}{\vvec{w}}/\mathfrak{B}=\lambda_k\sum_{i=1}^\cscalar{D}\lambda_i z_i^2/\mathfrak{B}
\end{equation}
according to Prediction~\ref{pred:grad_norm} and \eqref{eq:bgrad_norm}. Therefore, the variance is calculated as
\begin{equation}\label{eq:synth_grad_var}
\func{\mathbb{VAR}}{g_k}=\mathbb{E}(g_k^2)-\big(\mathbb{E}(g_k)\big)^2=\lambda_k\Big(\sum_{i=1}^\cscalar{D}\lambda_i z_i^2/\mathfrak{B}-\lambda_k z_k^2\Big).
\end{equation}
Our algorithm simply draws $g_k$ as Gaussian of mean $\mathbb{E}(g_k)$ and variance $\func{\mathbb{VAR}}{g_k}$ as given in \eqref{eq:synth_grad_mean} and \eqref{eq:synth_grad_var}, respectively. 

Once the synthetic per-batch gradient $\vvec{g}$ is obtained, we convert it to the superficial coordinates by $\cmat{U}\vvec{g}$. Then we apply the optimizer $\func{\phi}{\bullet}$, and convert back to the coordinates under eigen basis again:
\[
\mbox{(update to $\vvec{w}$)}=\vvec{u}= \transpose{\cmat{U}}\func{\phi}{\cmat{U}\vvec{g}}.
\]
Our algorithm finishes this simulation step by updating $\vvec{z}=c(\vvec{w}-\cvec{w}^\ast)$ accordingly.

\subsection{Derivation of Learning-rate Threshold}
\label{sec:derive_gamma}

We analyze the dynamics of validation loss $\mathbb{E}\big(\func{\ell}{\cvec{x};\vvec{w}_{t}}\big)$ for Random Feature models defined in \eqref{eq:rf}, by taking eigen decomposition of the Hessian
\[
\cmat{H}\eqdef\mathbb{E}\Big(\func{\frac{\partial^2\ell}{\transpose{\partial\vvec{w}}\partial\vvec{w}}}{\cvec{x};\vvec{w}_{t}}\Big)=c\,\mathbb{E}(\cmat{A}\cvec{x}\transpose{\cvec{x}}\transpose{\cmat{A}}).
\]
Let $\nabla_\cscalar{B}\vvec{w}_{t}$ be the mini-batched gradient of batch-size $\cscalar{B}$. Let $g^{\cscalar{B}}_{i}$, $w_{i,t}$ and $\cscalar{w}^\ast_i$ be the projections of $\nabla_\cscalar{B}\vvec{w}_{t}$, $\vvec{w}_t$ and $\cvec{w}^\ast$ onto the $i$-th eigenmode, respectively. The SGD update is given by
\[
w_{i,t+1}=w_{i,t}-\gamma\,g^{\cscalar{B}}_{i},
\]
where $\gamma$ is the learning-rate. Now, recall that $\mathbb{E}(\nabla_\cscalar{B}\vvec{w}_{t})=\mathbb{E}(\nabla\vvec{w}_{t})=\cmat{H}(\vvec{w}_t-\cvec{w}^\ast)$; so we have
\begin{align}
(w_{i,t+1}-\cscalar{w}^\ast_i)^2 &= (w_{i,t}-\cscalar{w}^\ast_i)^2-2\gamma\,(w_{i,t}-\cscalar{w}^\ast_i)g^{\cscalar{B}}_{i}+\gamma^2 (g^{\cscalar{B}}_{i})^2 \label{eq:random_rf_valid_loss} \\
\mbox{and}\quad
\mathbb{E}\big((w_{i,t+1}-\cscalar{w}^\ast_i)^2\big) &=(1-\gamma\lambda_i)^2(w_{i,t}-\cscalar{w}^\ast_i)^2+\gamma^2\Big(\mathbb{E}\big((g^{\cscalar{B}}_{i})^2\big)-\big(\mathbb{E}(g^{\cscalar{B}}_{i})\big)^2\Big),\label{eq:average_case_rf_valid_loss}
\end{align}
where $\{\lambda_i\}$ is the eigenvalues of $\cmat{H}$. In order to derive the threshold of learning-rate between convergence and divergence, we make the intuitive assumption that when the learning-rate $\gamma$ is at the threshold, the variance of gradients is neither decreasing nor increasing, i.e., it remains a constant independent of the step $t$:
\begin{equation}\label{eq:constant_var}
\mathbb{E}\big((g^{\cscalar{B}}_{i})^2\big)-\big(\mathbb{E}(g^{\cscalar{B}}_{i})\big)^2\equiv V_i\neq 0.
\end{equation}
For more formal and rigorous treatment using the kernel of a Volterra equation, we refer to \citet[C.2]{four_plus_three}. Substituting \eqref{eq:constant_var} into \eqref{eq:average_case_rf_valid_loss}, we immediately get
\begin{equation}\label{eq:limit}
\lim_{s\to\infty}\mathbb{E}\big((w_{i,t+s}-\cscalar{w}^\ast_i)^2\big)=\sum_{s=0}^{\infty}(1-\gamma\lambda_i)^{2s}\gamma^2 V_i=\frac{\gamma}{\lambda_i(2-\gamma\lambda_i)}V_i.
\end{equation}
Furthermore, when $\cscalar{D}=\dim(\vvec{w})$ is sufficiently large and $\cscalar{B}\ll\sqrt{\cscalar{D}}$, we have
\begin{equation}\label{eq:vcalc}
\mathbb{E}\big((g^{\cscalar{B}}_{i})^2\big)=\frac{2}{\cscalar{B}}\lambda_i\,\mathbb{E}\big(\func{\ell}{\cvec{x};\vvec{w}_{t}}\big)=\frac{1}{\cscalar{B}}\lambda_i\,\sum_k \lambda_k(w_{k,t}-\cscalar{w}^\ast_k)^2
\end{equation}
by combining \eqref{eq:b2egg}, \eqref{eq:egg} and \eqref{eq:rf_simplification_volterra}. We also have $\big(\mathbb{E}(g^{\cscalar{B}}_{i})\big)^2\ll\mathbb{E}\big((g^{\cscalar{B}}_{i})^2\big)$ under this condition, based on a similar argument as in Lemma~\ref{lem:e2bx2}. So $V_i$ can be explicitly calculated, therefore \eqref{eq:limit} implies
\begin{equation}\label{eq:threshold_per_eigenmode}
\lim_{s\to\infty}\mathbb{E}\big((w_{i,t+s}-\cscalar{w}^\ast_i)^2\big)=\frac{\gamma}{\lambda_i(2-\gamma\lambda_i)}V_i=\frac{\gamma}{\cscalar{B}(2-\gamma\lambda_i)}\sum_k \lambda_k(w_{k,t}-\cscalar{w}^\ast_k)^2.
\end{equation}
Then, by summing across all $i$-th eigenmodes in \eqref{eq:threshold_per_eigenmode}, we obtain:
\begin{equation}
\lim_{s\to\infty}\mathbb{E}\big(\sum_i \lambda_i(w_{i,t+s}-\cscalar{w}^\ast_i)^2\big)=\sum_i\frac{\gamma\lambda_i}{\cscalar{B}(2-\gamma\lambda_i)}\sum_k \lambda_k(w_{k,t}-\cscalar{w}^\ast_k)^2.
\end{equation}
Thus the following equation regarding the threshold $\gamma$, by taking the limit at $t\to\infty$:
\begin{equation}\label{eq:gamma_derived}
1=\sum_i\frac{\gamma\lambda_i}{\cscalar{B}(2-\gamma\lambda_i)}.
\end{equation}

\subsection{Concentration Effect Only Holds for \texorpdfstring{$\alpha\leq 0.5$}{alpha<=0.5}}
\label{sec:concentration_alpha}

We need to analyze the random process \eqref{eq:random_rf_valid_loss}, instead of the average-case behavior \eqref{eq:average_case_rf_valid_loss}. The difference between \eqref{eq:random_rf_valid_loss} and \eqref{eq:average_case_rf_valid_loss} is a martingale; When one sums the martingales from all eigenmodes, the ``\emph{concentration effect}'' states that the sum converges to $0$ with ``overwhelming probability'', under a suitable limit at $\cscalar{D}=\dim(\vvec{w})\to\infty$.

We refer to \citet{sgd_in_large} for the formal tools and rigorous arguments regarding convergence of martingales. In this section, we give an intuitive sketch of the essence of the idea.  

Recall the arguments in \S~\ref{sec:sketch}, that $\vvec{w}_t$ can be decomposed into $\vvec{w}_{t}=\overline{\vvec{w}}_{t}+\vvec{r}_{t}$ where $\overline{\vvec{w}}_{t}$ follows the deterministic gradient descent, while $\vvec{r}_{t}$ is an accumulation of random noises of mean $\cvec{0}$. At each step, a new draw of gradient noise is added into $\vvec{r}_{t}$, and \emph{on average} it contributes to the $i$-th eigenmode of $\mathbb{E}\big((w_{i,t+1}-\cscalar{w}^\ast_i)^2\big)$ by $\gamma^2\Big(\mathbb{E}\big((g^{\cscalar{B}}_{i})^2\big)-\big(\mathbb{E}(g^{\cscalar{B}}_{i})\big)^2\Big)$, according to \eqref{eq:average_case_rf_valid_loss}.

From the calculation of \eqref{eq:vcalc}, the variance $\Big(\mathbb{E}\big((g^{\cscalar{B}}_{i})^2\big)-\big(\mathbb{E}(g^{\cscalar{B}}_{i})\big)^2\Big)$ is proportional to $\lambda_i$; so we can write the \emph{random} contribution of the gradient noise to the $i$-th eigenmode of $(w_{i,t+1}-\cscalar{w}^\ast_i)^2$, by $\lambda_iX_i^2$ where $X_i$ is a random variable, independent and of the same variance across all $i$-th eigenmodes.

Now, by summing across all eigenmodes, the gradient noise contributes to the random validation loss $\sum_i\lambda_i(w_{i,t+1}-\cscalar{w}^\ast_i)^2$ by $\gamma^2\sum_i\lambda_i^2X_i^2$. The concentration effect states that
\[
\gamma^2\sum_{i=1}^\cscalar{D}\lambda_i^2X_i^2\approx\gamma^2\sum_{i=1}^\cscalar{D}\lambda_i^2\mathbb{E}(X_i^2)
\]
with high probability, at the limit $\cscalar{D}\to\infty$.

Note that, in order for the model to converge, the learning-rate $\gamma$ should be less than the threshold which satisfies Equation~\eqref{eq:gamma_derived}, meaning $\gamma < 2\cscalar{B}/\sum_i\lambda_i$. To make the limit $\cscalar{D}\to\infty$ properly defined, $\gamma$ should actually scale in proportion to $1/\sum_i\lambda_i$. Thus, we obtain the following
\begin{equation}\label{eq:weighted_average}
\frac{\sum_{i=1}^\cscalar{D}\lambda_i^2X_i^2}{\big(\sum_{i=1}^\cscalar{D}\lambda_i\big)^2}
\end{equation}
and note that $\mathbb{E}(X_i^2)$ scales approximately in proportion to $\sum_{i=1}^\cscalar{D}\lambda_i$. So \eqref{eq:weighted_average} can be viewed as an average of $\{\lambda_i X_i^2/(\sum_{i=1}^\cscalar{D}\lambda_i)\}$, weighted by the probability measure $\big(\lambda_i/(\sum_{i=1}^\cscalar{D}\lambda_i)\big)$. 

%So the expectation of the above can have a finite non-zero limit at $\cscalar{D}\to\infty$. %The concentration effect states that the random value concentrates around its expectation with high probability, at $\cscalar{D}\to\infty$.

For PLRF models with power-law exponent $\alpha$, the eigenvalues $(\lambda_i)$ are asymptotically proportional to $(i^{-2\alpha})$. Then, assuming $\{X_i\}$ are i.i.d.~random variables, and $X_i^2/(\sum_{i=1}^\cscalar{D}\lambda_i)$ has finite expectation and variance, the concentration to expectation is a standard argument by the Law of Large Numbers, that the random fluctuations of \eqref{eq:weighted_average} will cancel out, if and only if
\[
\lim_{\cscalar{D}\to\infty}\sum_{i=1}^\cscalar{D}\lambda_i=\infty.
\]
This is equivalent to $\alpha\leq 0.5$.

\section{PLRF Simulation Experiments}

As a sanity check, we apply our simulation algorithm to PLRF models, and provide the Python code in supplementary materials. The PLRF models are defined the same as in Figure~\ref{fig:plrf}, while we train with SGD and Adam, at batch-sizes $\cscalar{B}=\{1,8,64\}$. Each configuration is simulated by two runs of different random seeds.

\begin{figure}[t]
 \centering
 \begin{subfigure}[b]{0.32\textwidth}
     \centering
     \includegraphics[width=\textwidth]{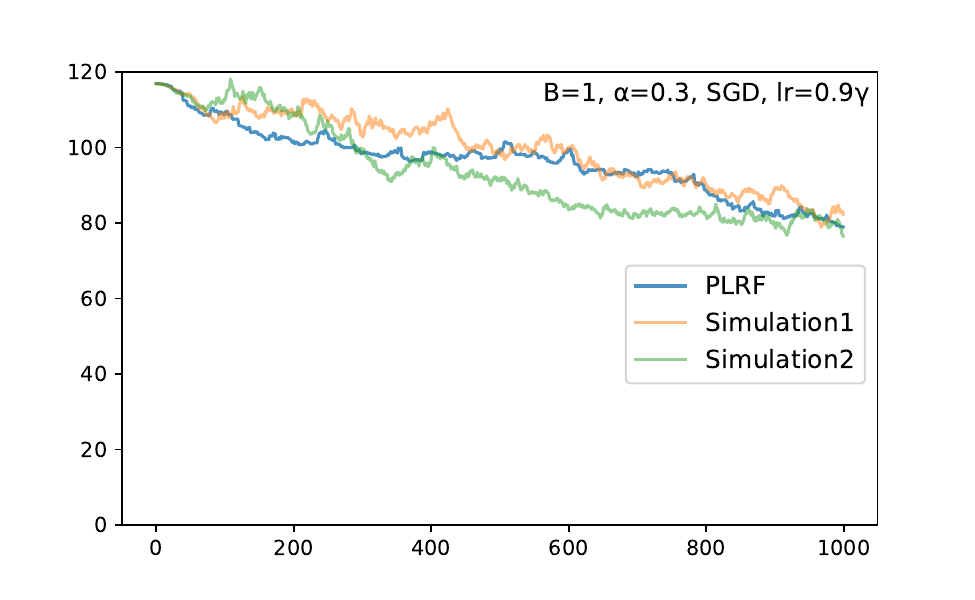}
     %\caption{(a)}
     %\label{fig:y equals x}
 \end{subfigure}
 \begin{subfigure}[b]{0.32\textwidth}
     \centering
     \includegraphics[width=\textwidth]{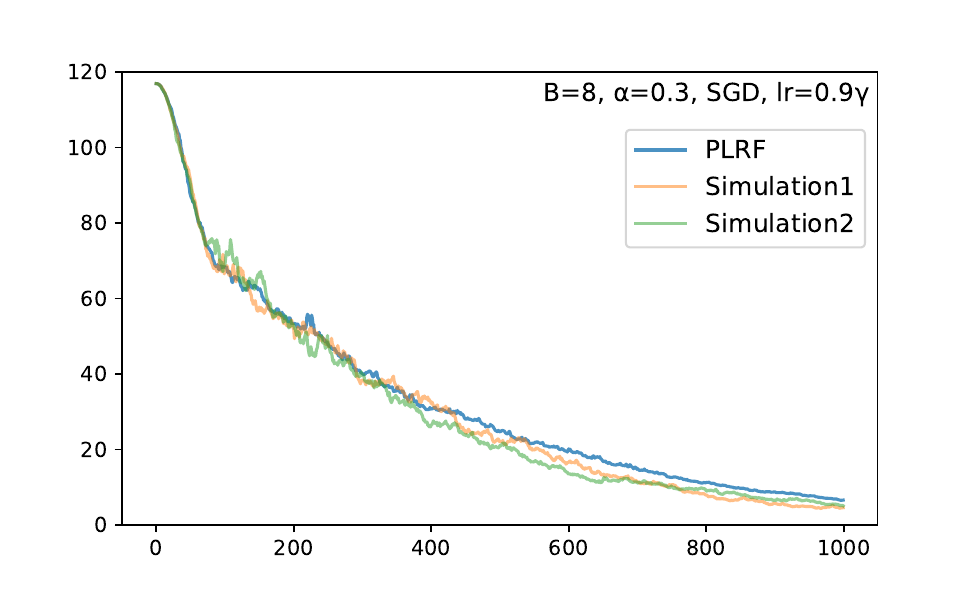}
     %\caption{(a)}
     %\label{fig:y equals x}
 \end{subfigure}
 \begin{subfigure}[b]{0.32\textwidth}
     \centering
     \includegraphics[width=\textwidth]{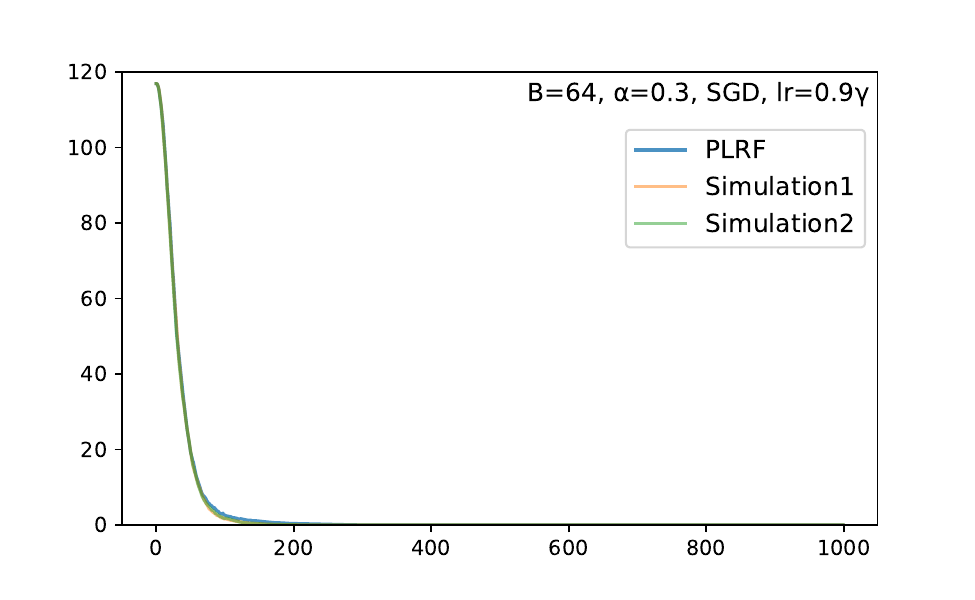}
     %\caption{(a)}
     %\label{fig:y equals x}
 \end{subfigure}\\
 \begin{subfigure}[b]{0.32\textwidth}
     \centering
     \includegraphics[width=\textwidth]{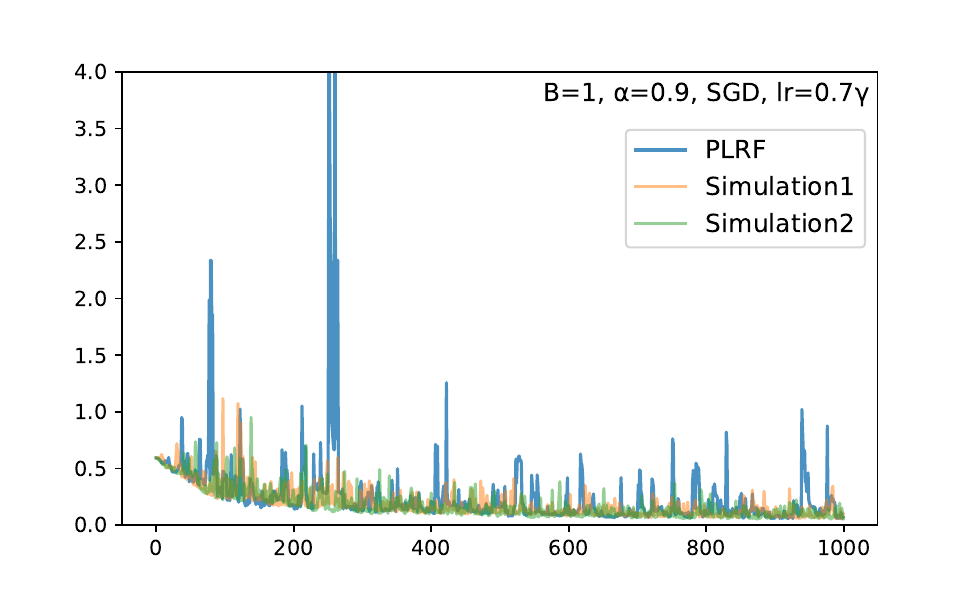}
     %\caption{(a)}
     %\label{fig:y equals x}
 \end{subfigure}
 \begin{subfigure}[b]{0.32\textwidth}
     \centering
     \includegraphics[width=\textwidth]{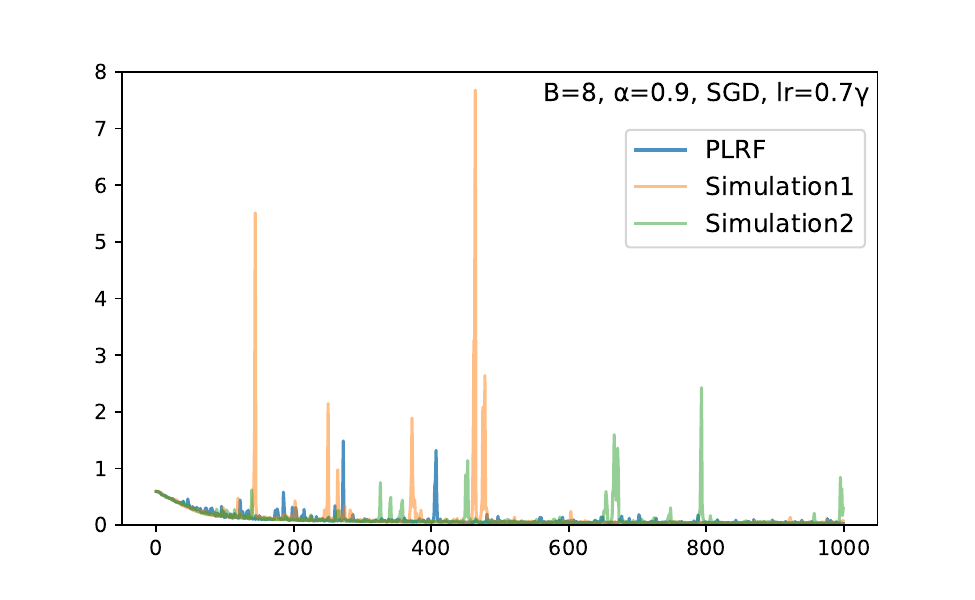}
     %\caption{(a)}
     %\label{fig:y equals x}
 \end{subfigure}
 \begin{subfigure}[b]{0.32\textwidth}
     \centering
     \includegraphics[width=\textwidth]{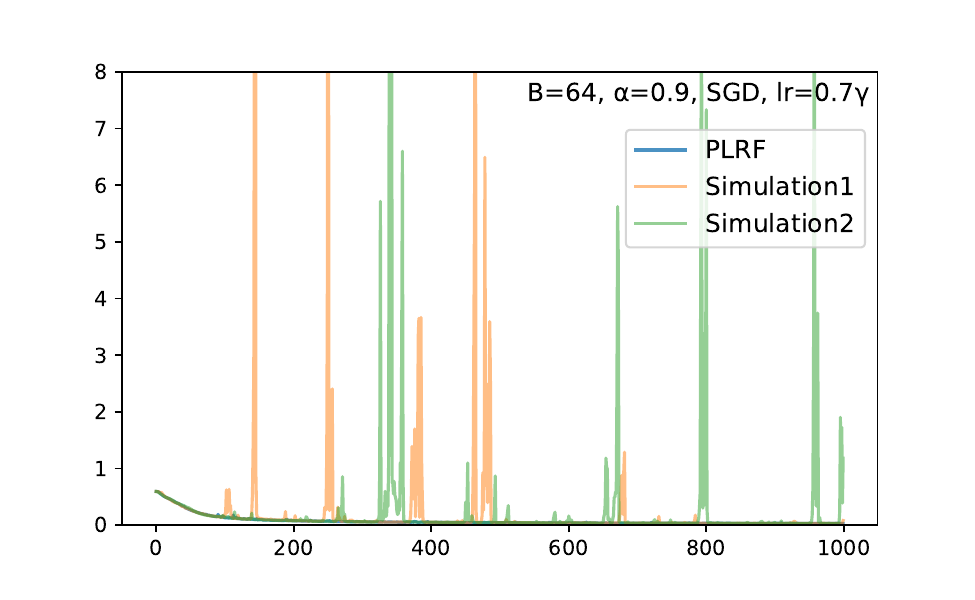}
     %\caption{(a)}
     %\label{fig:y equals x}
 \end{subfigure}\\
 \begin{subfigure}[b]{0.32\textwidth}
     \centering
     \includegraphics[width=\textwidth]{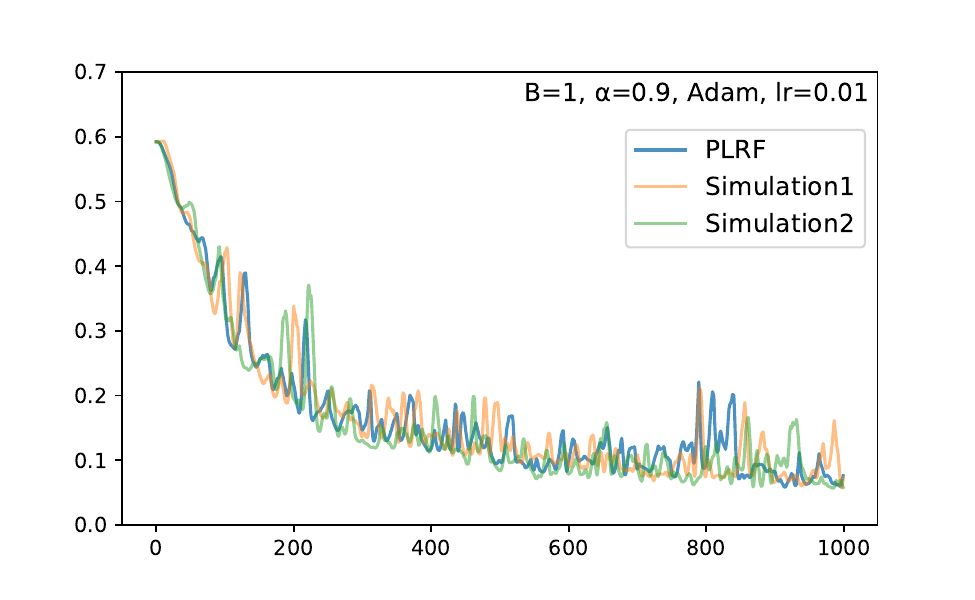}
     %\caption{(a)}
     %\label{fig:y equals x}
 \end{subfigure}
 \begin{subfigure}[b]{0.32\textwidth}
     \centering
     \includegraphics[width=\textwidth]{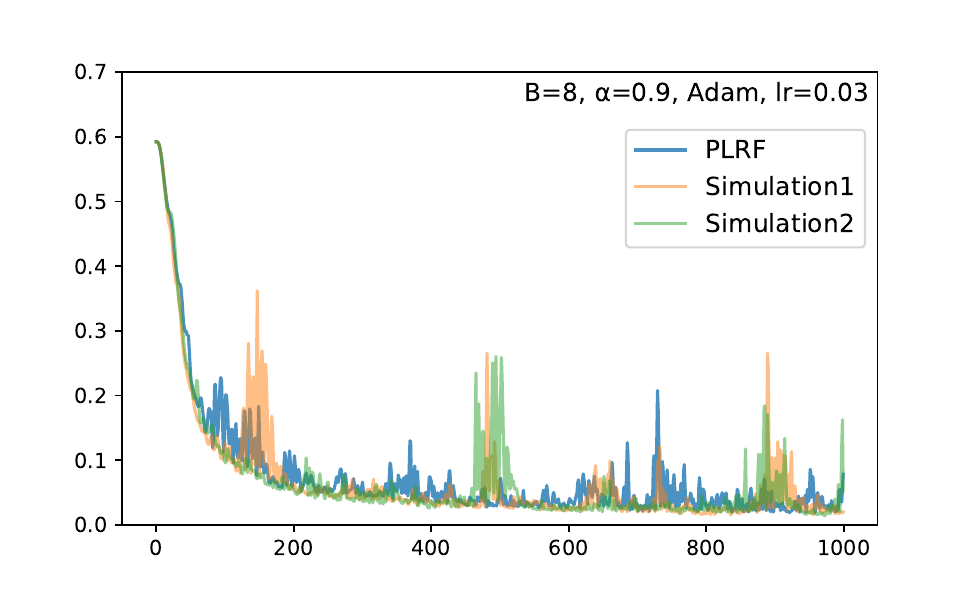}
     %\caption{(a)}
     %\label{fig:y equals x}
 \end{subfigure}
 \begin{subfigure}[b]{0.32\textwidth}
     \centering
     \includegraphics[width=\textwidth]{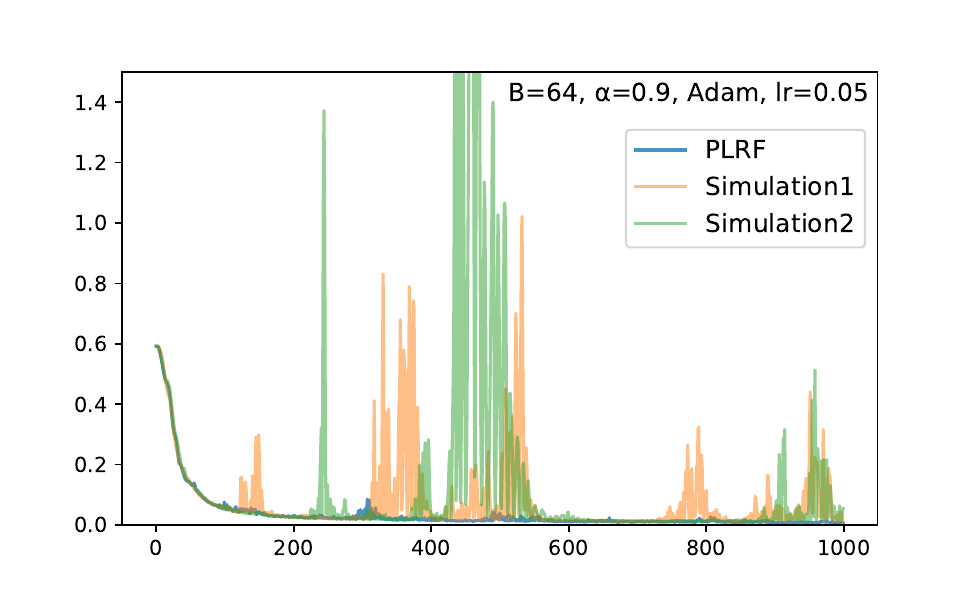}
     %\caption{(a)}
     %\label{fig:y equals x}
 \end{subfigure}\\
 \begin{subfigure}[b]{0.32\textwidth}
     \centering
     \includegraphics[width=\textwidth]{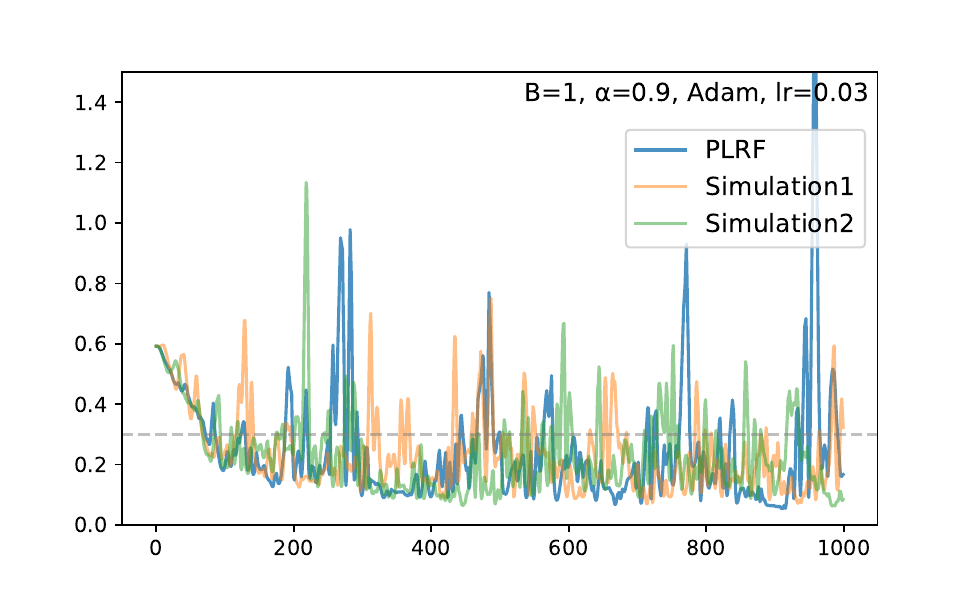}
     %\caption{(a)}
     %\label{fig:y equals x}
 \end{subfigure}
 \begin{subfigure}[b]{0.32\textwidth}
     \centering
     \includegraphics[width=\textwidth]{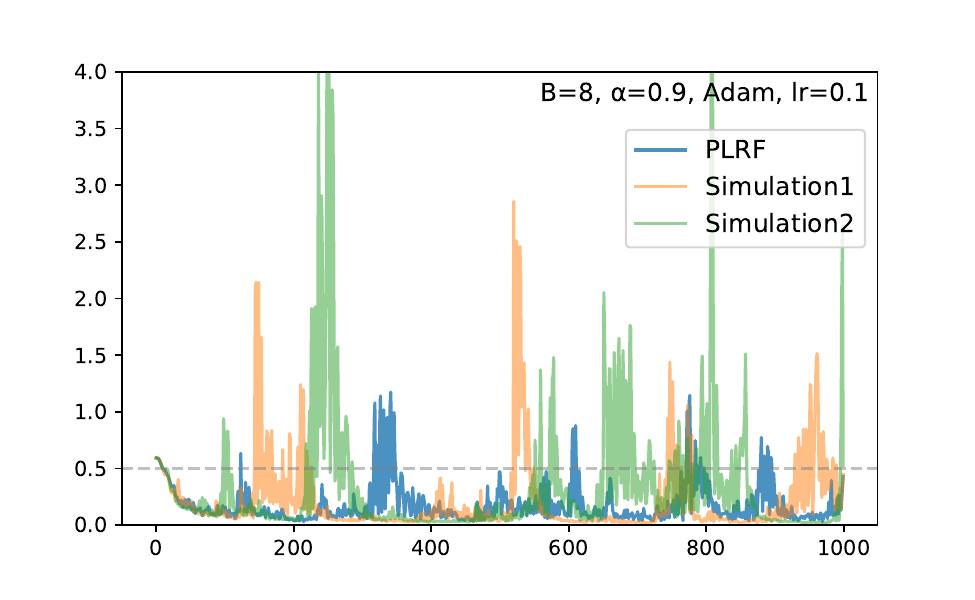}
     %\caption{(a)}
     %\label{fig:y equals x}
 \end{subfigure}
 \begin{subfigure}[b]{0.32\textwidth}
     \centering
     \includegraphics[width=\textwidth]{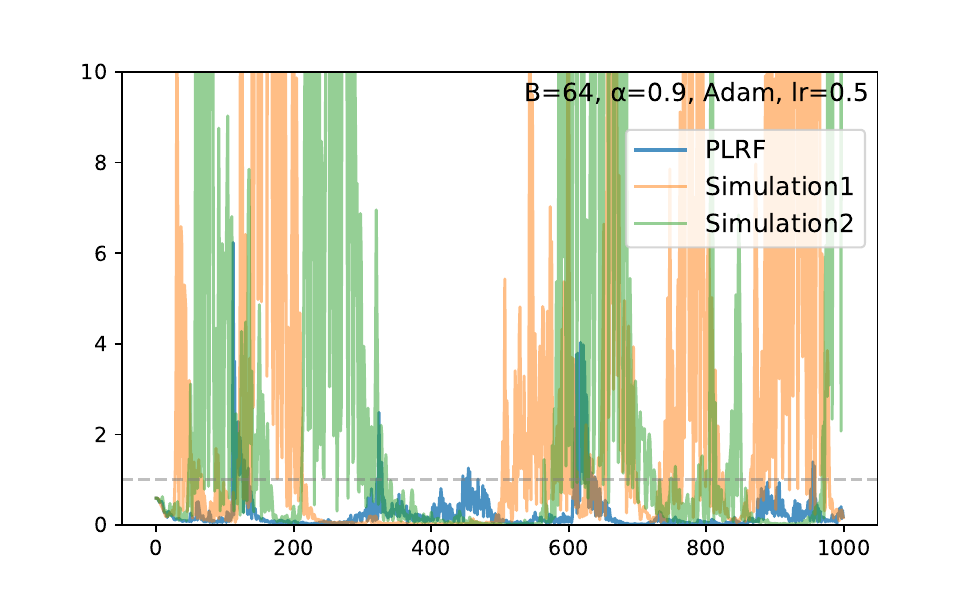}
     %\caption{(a)}
     %\label{fig:y equals x}
 \end{subfigure}
\caption{Simulating PLRF by Algorithm~\ref{alg:simulation}. We apply SGD and Adam at different batch-sizes. Each plot includes one PLRF training and two simulation runs (with different random seeds).}
\label{fig:plrf_simulation}
\end{figure}

Figure~\ref{fig:plrf_simulation} shows the results. We observe that the simulation correctly captures the average-case behavior of the validation loss: the underlying ``flows'' match each other as long as we ignore the spikes. Moreover, the simulation qualitatively captures the spiky behavior: Simulation is spiky if and only if the actual PLRF training trajectory is spiky.

However, the simulation does not always correctly predict every statistics of training spikes: It seems that the simulation tends to underestimate the ``largeness'' of spikes when $\cscalar{B}=1$, while significantly overestimate it when $\cscalar{B}$ is large. 

\begin{table}[t]
\footnotesize
\centering
\begin{tabular}{@{}lccccccccc@{}}
\toprule
 & \multicolumn{3}{c}{\texttt{B=1,LR=0.03,TH=0.3}} & \multicolumn{3}{c}{\texttt{B=8,LR=0.1,TH=0.5}} & \multicolumn{3}{c}{\texttt{B=64,LR=0.5,TH=1.0}} \\
 \cmidrule(lr){2-4} \cmidrule(lr){5-7} \cmidrule(lr){8-10}
            & Q90 & RAT & MRLAT & Q90 & RAT & MRLAT & Q90 & RAT & MRLAT \\ \midrule
PLRF        & 0.89 & 0.18 & 8.2 & 1.1 & 0.051 & 2.2 & 3.8 & 0.035 & 1.3 \\
Simulation1 & 0.59 & 0.14 & 4.9 & 1.4 & 0.096 & 2.7 & 22. & 0.38 & 4.6 \\
Simulation2 & 0.53 & 0.18 & 5.1 & 2.5 & 0.19 & 3.7 & 44. & 0.36 & 6.9 \\ \bottomrule
\end{tabular}
\caption{Statistics of training spikes, for three configurations of batch-sizes $\cscalar{B}=\{1,8,64\}$, with respectively picked learning-rate (\texttt{LR}) and threshold (\texttt{TH}).}
\label{tab:stats_spikes}
\end{table}

This tendency is quantitatively documented in Table~\ref{tab:stats_spikes}: For each of the three configurations in the last row of Figure~\ref{fig:plrf_simulation}, we pick a threshold (shown as dashed lines in the figure) and empirically measure the following statistics to reflect the nature of spikes:
\begin{itemize}
 \item \textbf{Quantile at $90\%$ (Q90)}: The 90th percentile of loss values larger than the threshold at steps $\geq200$. This measures how large the spikes are.
 \item \textbf{Ratio Above Threshold (RAT)}: The ratio of loss values larger than the threshold, for steps $\geq200$. This measures how frequently the spikes occur.
 \item \textbf{Mean Run Length Above Threshold (MRLAT)}: Count the number of consecutive steps ($\geq 200$) where loss values exceed the threshold before they drop back down. Then, take the average of these lengths. This measures the average duration of a spike.
\end{itemize}
We observe that for $\cscalar{B}=1$, simulations have slightly lower Q90; while for $\cscalar{B}=64$, all statistics (Q90, RAT and MRLAT) become much larger, indicating that the simulations are overestimating spikes. We believe the reasons are due to: (1) The distribution of \emph{per-example} gradients of PLRF is not Gaussian; (2) Our estimate of the 2nd moment of \emph{per-batch} gradients, as used in the simulation, is only asymptotically accurate when the batch-size is small.

\section{Transformer LM Experiments}

In this section, we give details of our case-study of a Transformer LM training.

\subsection{Model and Training Settings}
\label{sec:llm_train_settings}

The model is a vanilla decoder-only Transformer, with Self-Attention \citep{NIPS2017_3f5ee243}, Pre-Layernorm \citep{xiong2020layernormalizationtransformerarchitecture}, Rotary Position Encoding \citep{rope} and QK-layernorm \citep{pmlr-v202-dehghani23a} for each layer. We build our experiments upon an existing series of model sizes used in previous research \citep{DBLP:conf/icml/EverettXWANLGSK24,DBLP:journals/corr/abs-2503-09799}, where the sizes of model components are parameterized by a single factor $k$, as below:
%Our scaling experiments are built upon the work of \url{http://go/scaling-law-baseline}; the model is a decoder-only Transformer, with Self-Attention, Rotary Position Encoding, QK-Norm, 2-layer MLP, GeLU activation, and PreNorm for each layer. The sizes of model components are parameterized by a single scaling factor $k$, as below:
\begin{itemize}
\item Sequence length: $2048$.
\item Vocab size: $32000$.
\item Number of blocks: $12*k$.
\item Number of attention heads: $16*k$.
\item Dimension of attention heads: $64$.
\item Dimension of model hidden vectors = (Attention head dimension) $*$ (Number of attention heads) = $64*16*k=1024*k$.
\item MLP intermediate dimension = $4$ $*$ (Hidden dimension) = $4096*k$.
\end{itemize}
As detailed in \citet{DBLP:conf/icml/EverettXWANLGSK24}, the training hyperparameters are found by sweeping smaller models and verified by fitting a power-law curve. We take one point $k=2$ for our case-study.

\begin{figure}[t]%{0.32\textwidth}
\centering
\includegraphics[width=0.7\textwidth]{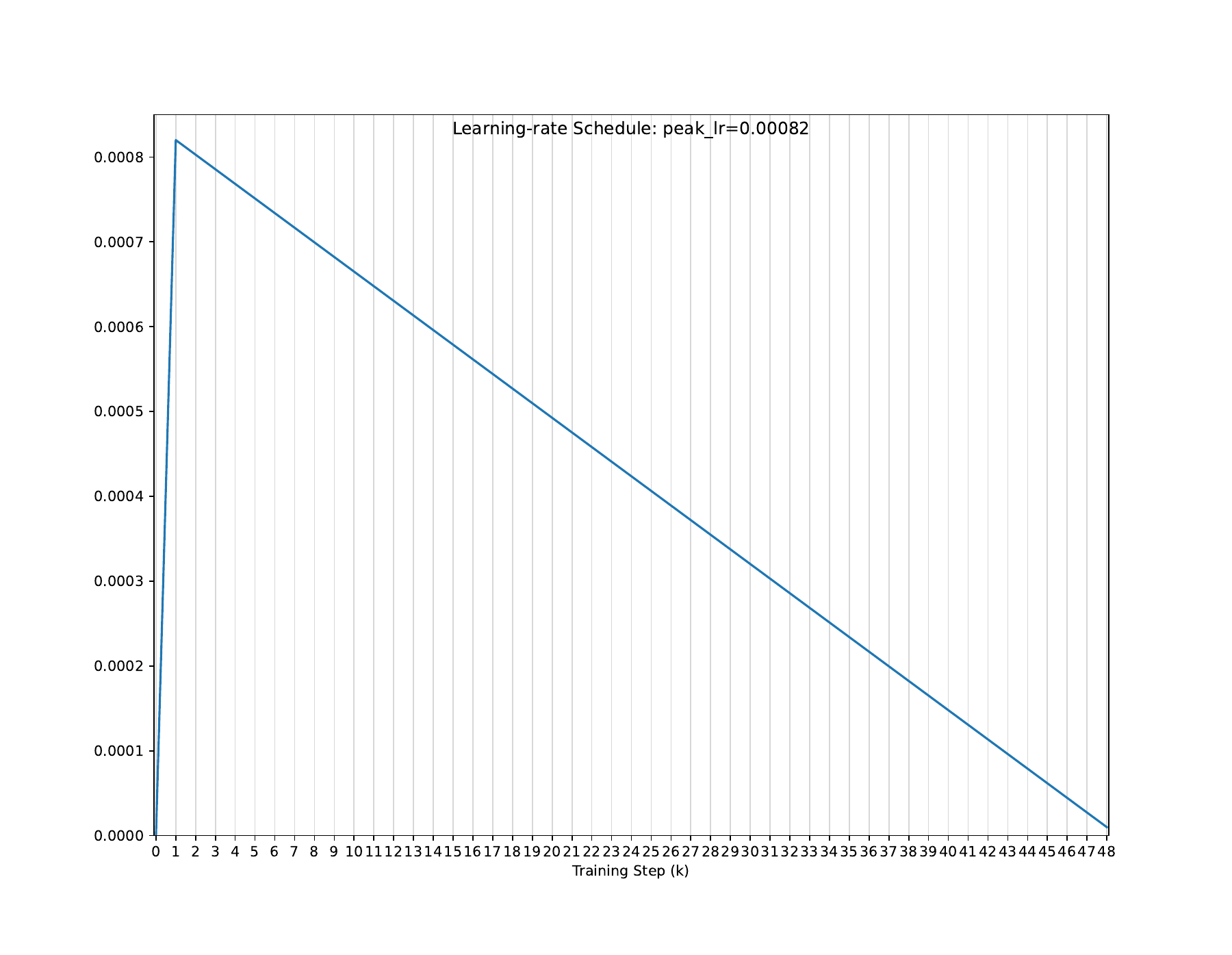}
\caption{The learning-rate schedule used for our Transformer LM training.}
\label{fig:lr_schedule}
\end{figure}

For this point, the hyperparameters are given as below.
\begin{itemize}
\item Optimizer: AdamW \citep{loshchilov2018decoupled} with $\beta_1=0.9$, $\beta_2=0.95$. We do not apply any gradient clipping to the training.
\item Batch-size (number of sequences per batch): Set to $256$.
\item Number of training steps: Determined by the Chinchilla rule \citep{hoffmann2022trainingcomputeoptimallargelanguage}, i.e., (Total number of tokens trained) / (Tokens per batch) = ($20$ $*$ Number of model parameters) / $(256\times2048)$ $\approx$ $48500$.
\item Learning-rate schedule: We use a simple schedule with warm-up in $1000$ steps, then linearly decay to $0$. The peak learning-rate is $0.00082$. The schedule is shown in Figure~\ref{fig:lr_schedule}.
%\item Peak learning-rate: Determined by a Bayesian fit of the power law; then with the near-optimality verified at 9.8B scale ($k=4$) by a grid search.
\item Weight-decay: Set to $0.2$.
\end{itemize}

\begin{figure}[t!]
 \centering
 \begin{subfigure}[b]{0.49\textwidth}
     \centering
     \includegraphics[width=\textwidth]{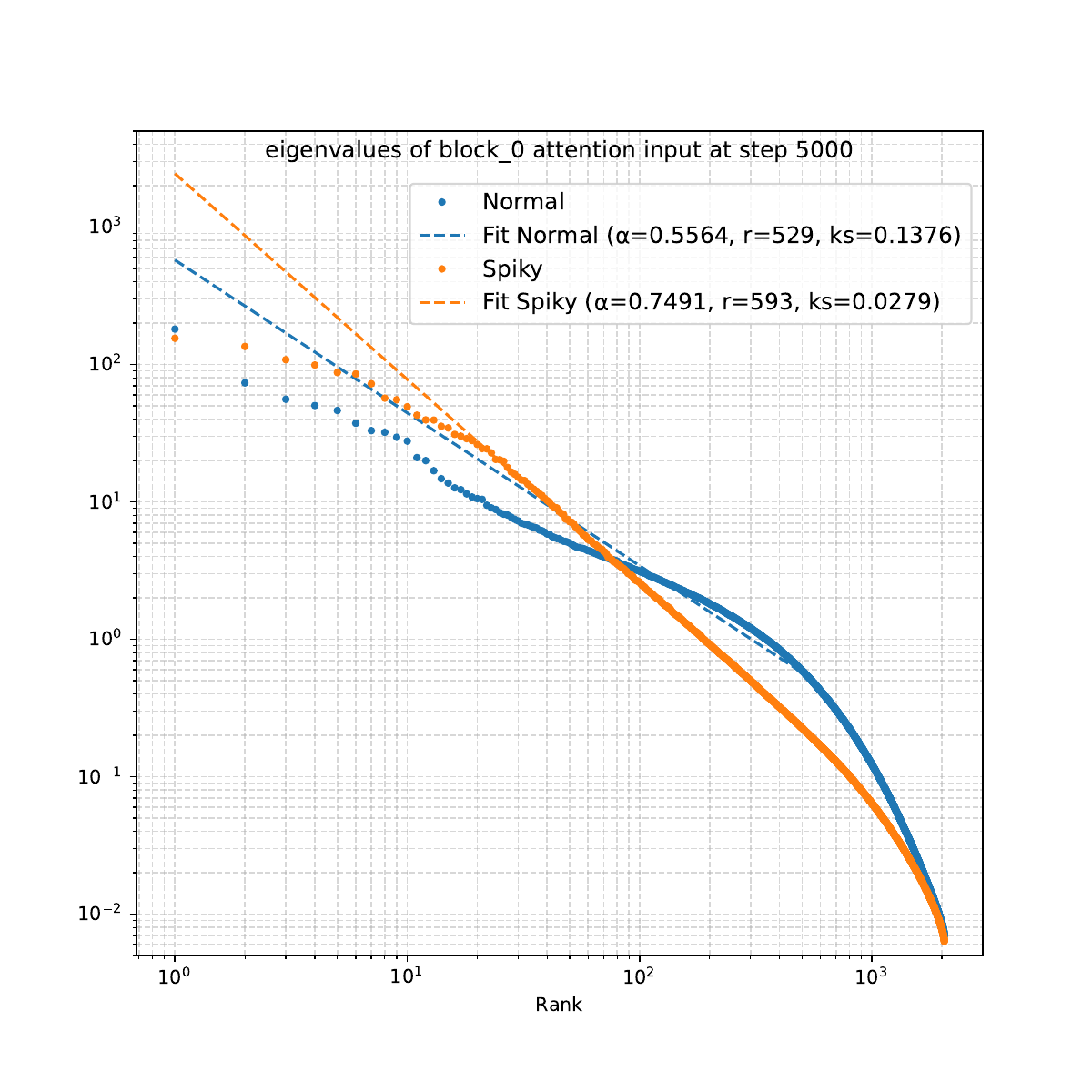}
     %\caption{(a)}
     %\label{fig:y equals x}
 \end{subfigure}
 \begin{subfigure}[b]{0.49\textwidth}
     \centering
     \includegraphics[width=\textwidth]{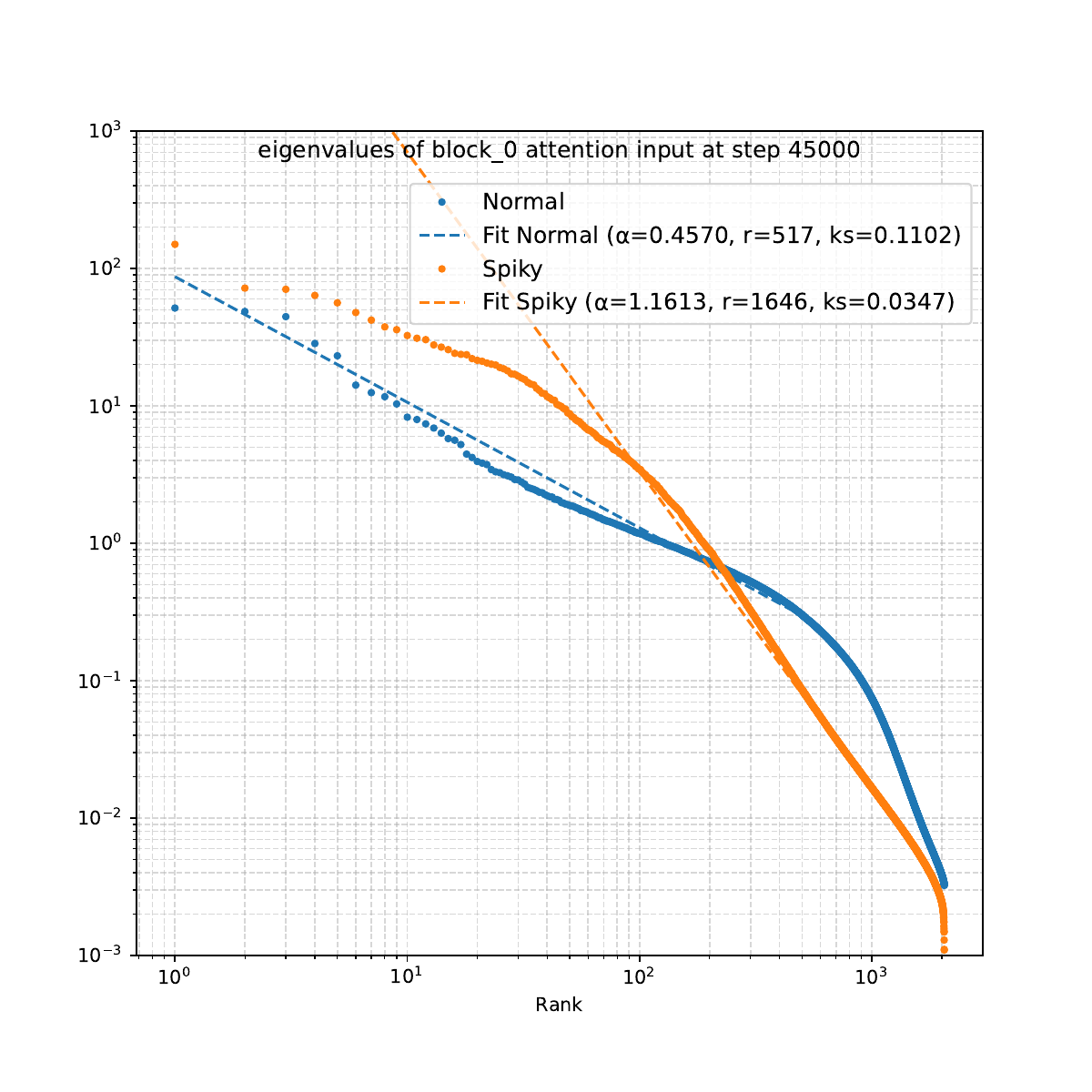}
     %\caption{(a)}
     %\label{fig:y equals x}
 \end{subfigure}\\
 \begin{subfigure}[b]{0.49\textwidth}
     \centering
     \includegraphics[width=\textwidth]{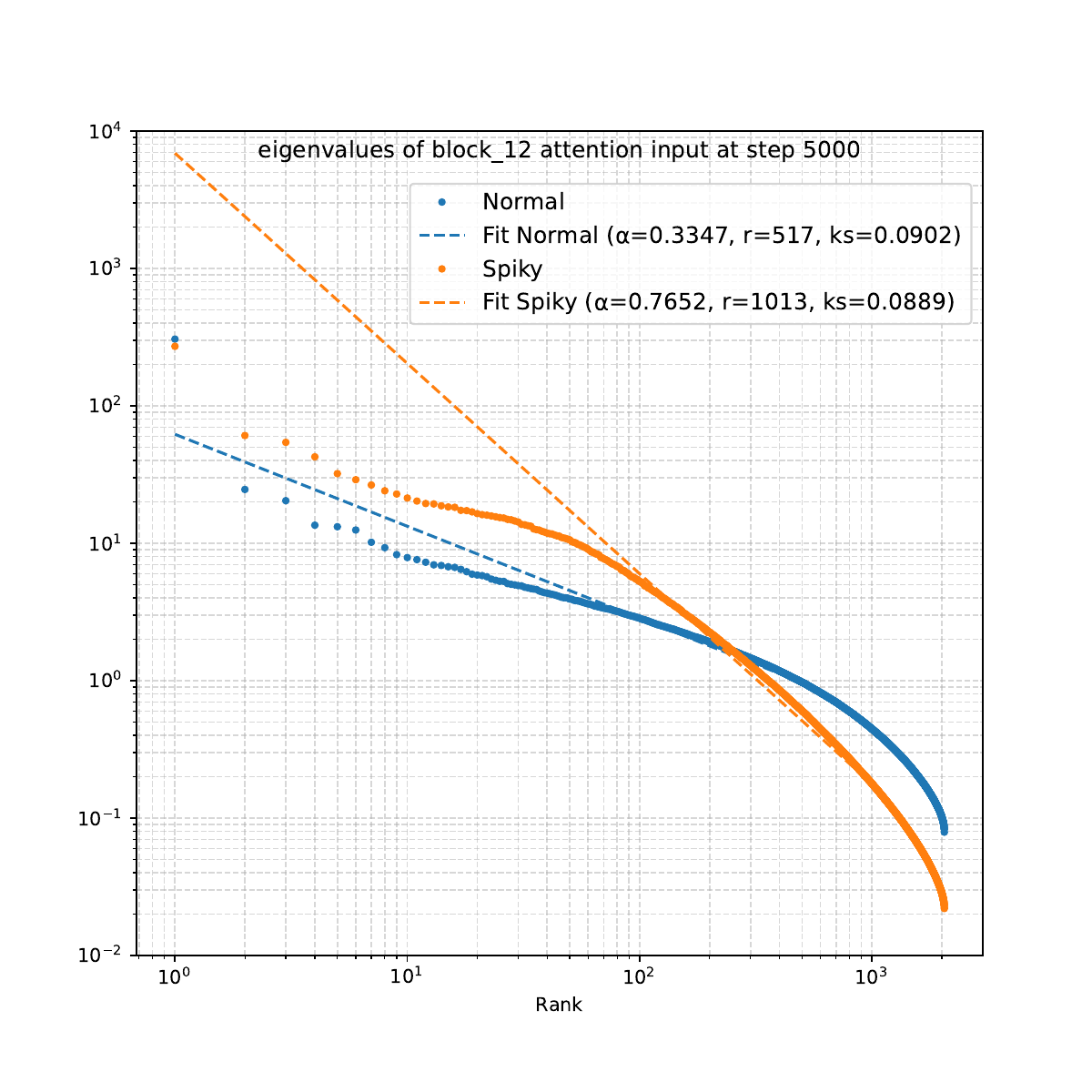}
     %\caption{(a)}
     %\label{fig:y equals x}
 \end{subfigure}
 \begin{subfigure}[b]{0.49\textwidth}
     \centering
     \includegraphics[width=\textwidth]{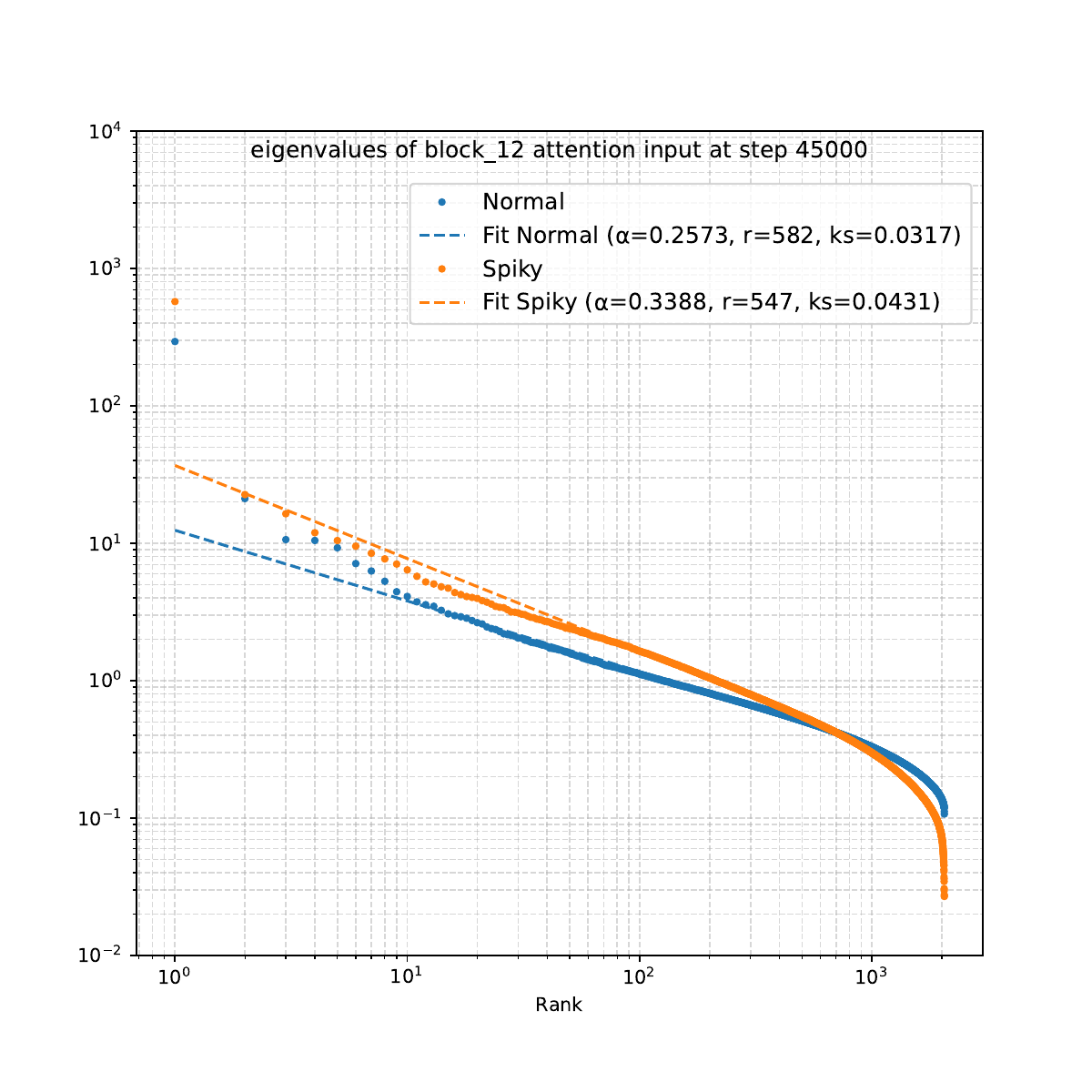}
     %\caption{(a)}
     %\label{fig:y equals x}
 \end{subfigure}\\
 \begin{subfigure}[b]{0.49\textwidth}
     \centering
     \includegraphics[width=\textwidth]{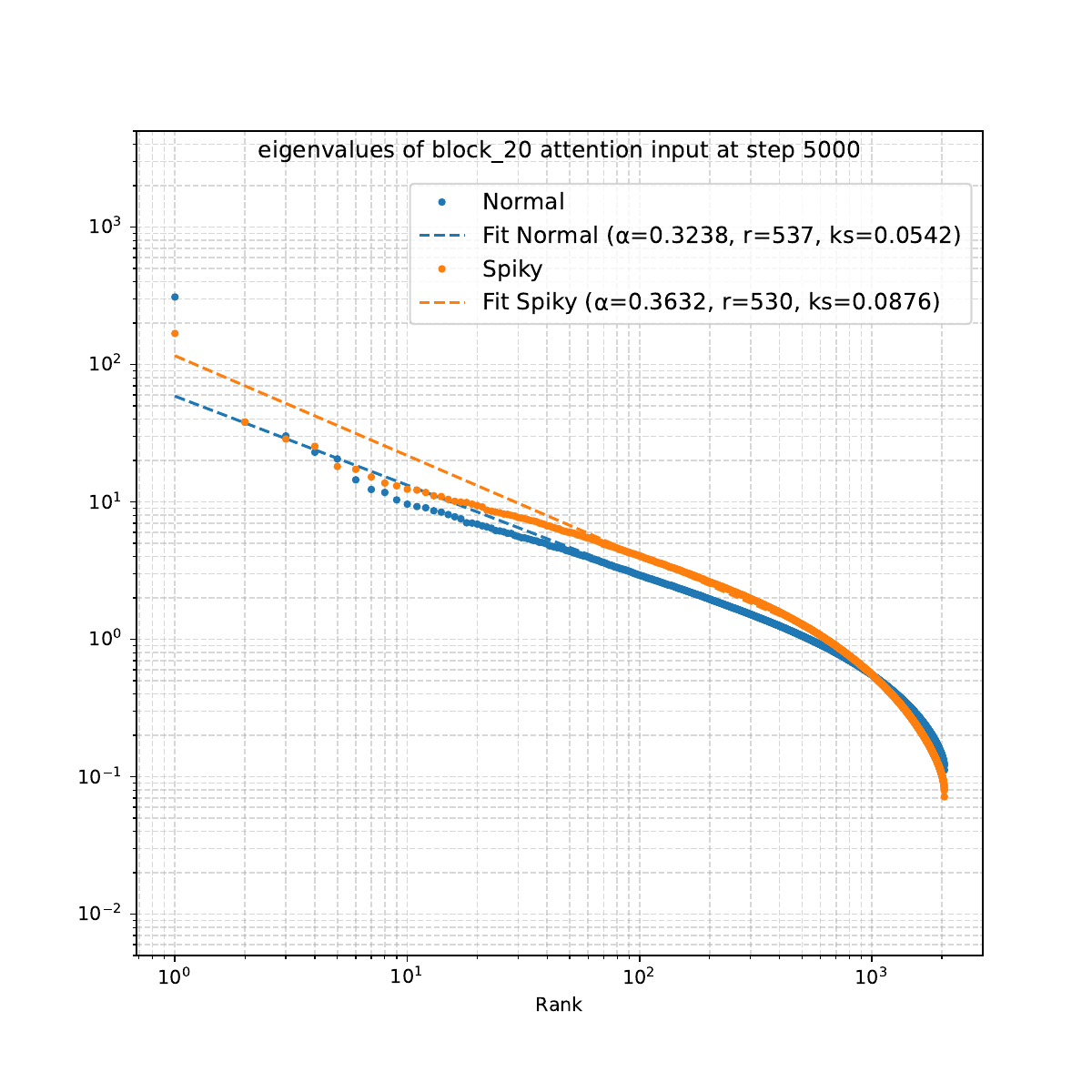}
     %\caption{(a)}
     %\label{fig:y equals x}
 \end{subfigure}
 \begin{subfigure}[b]{0.49\textwidth}
     \centering
     \includegraphics[width=\textwidth]{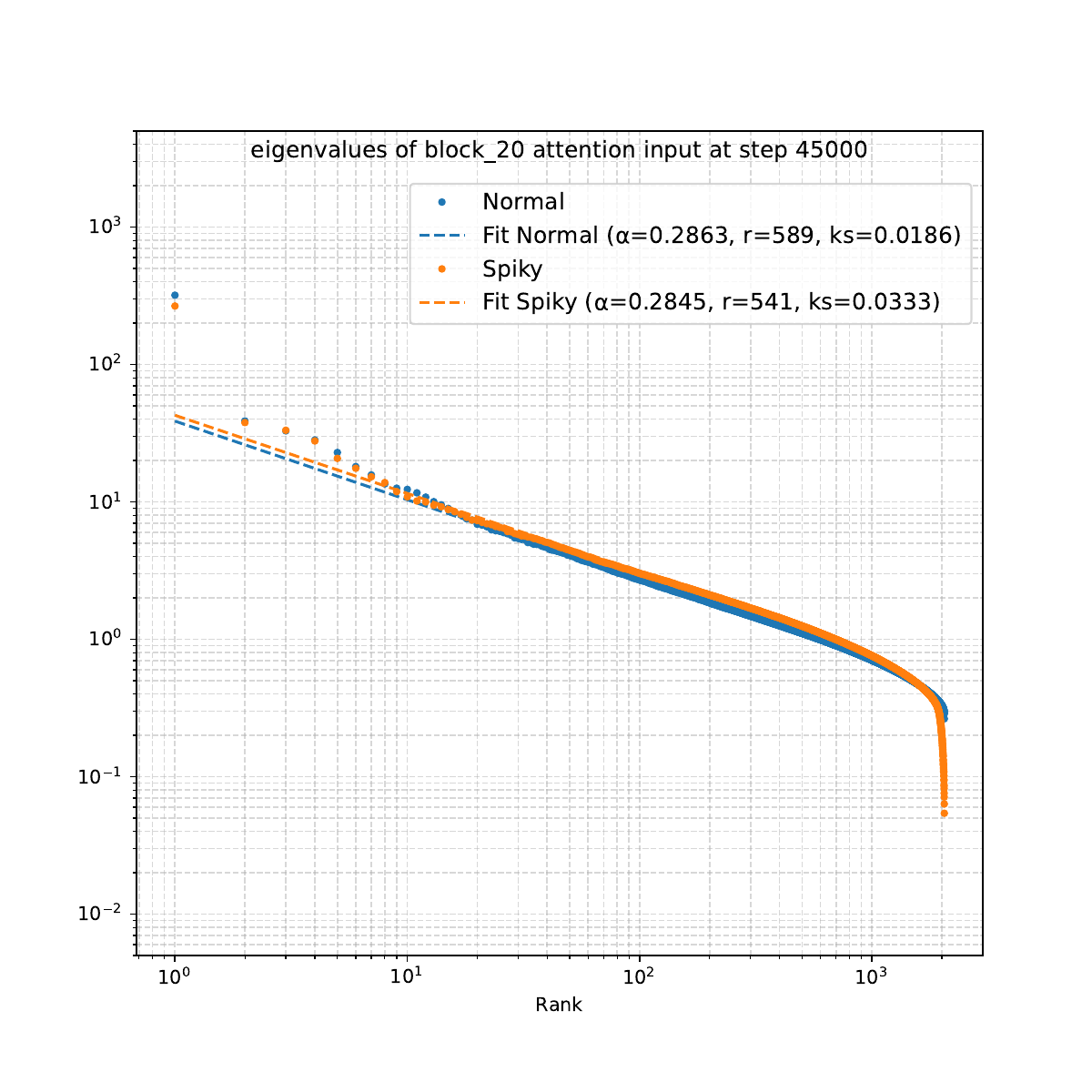}
     %\caption{(a)}
     %\label{fig:y equals x}
 \end{subfigure}
\caption{Log-log plot with truncated power-law fit, for ranked eigenvalues of correlation matrices.}
\label{fig:loglog}
\end{figure}

The model is trained with cross-entropy loss normalized by the number of tokens per batch. We note this normalization factor is applied to gradient -- the loss should be considered \emph{per-token} cross-entropy without any normalization factor; while the \emph{per-batch} gradient is summed across all tokens in a batch, so we normalize it by the number of tokens per batch, in order to ensure that the \emph{average gradient} over a dataset does not depend on the batch-size of the calculation.

This convention of normalization becomes critical when we discuss and distinguish between the squared norms of \emph{per-example} gradient $\lVert\nabla\vmat{W}_{t}\lVert^2$ and \emph{per-batch} gradient $\lVert\nabla_\cscalar{B}\vmat{W}_{t}\lVert^2$. Here, $\lVert\nabla\vmat{W}_{t}\lVert^2$ denotes the squared norm of gradient over \emph{one token}, and $\mathbb{E}(\lVert\nabla\vmat{W}_{t}\lVert^2)$ is averaged across \emph{all tokens} in a dataset. In contrast, $\lVert\nabla_\cscalar{B}\vmat{W}_{t}\lVert^2$ denotes the squared norm of gradient of \emph{a batch}, and $\mathbb{E}(\lVert\nabla_\cscalar{B}\vmat{W}_{t}\lVert^2)$ takes average across \emph{all batches} in a dataset. There is no obvious relation between $\mathbb{E}(\lVert\nabla\vmat{W}_{t}\lVert^2)$ and $\mathbb{E}(\lVert\nabla_\cscalar{B}\vmat{W}_{t}\lVert^2)$, because the square in $\lVert\nabla_\cscalar{B}\vmat{W}_{t}\lVert^2$ is taken \emph{after} summing over the batch! Throughout this paper we will carefully distinguish, prove and verify relations between $\mathbb{E}(\lVert\nabla\vmat{W}_{t}\lVert^2)$ and $\mathbb{E}(\lVert\nabla_\cscalar{B}\vmat{W}_{t}\lVert^2)$. 

Empirically, the calculation of $\mathbb{E}(\lVert\nabla\vmat{W}_{t}\lVert^2)$ can still be batched, because we have the following
\begin{equation}
(\nabla\vmat{W}_{t})^{\odot2}=(\nabla\vvec{y}_{t})^{\odot2}\,\transpose{\cvec{h}_{t}^{\odot2}},
\end{equation}
where $\;^{\odot2}$ denotes element-wise square. Therefore, by taking squares of $\nabla\vvec{y}_{t}$ and $\cvec{h}_{t}$ first, one can obtain results of squares \emph{before} summing over the batch, while still batching the whole calculation and summing across all tokens.

In order to calculate the Hessian trace $\trace\big(\mathbb{E}(\nabla^2\vmat{W}_{t})\big)$, we use the fact that for any random vector $\vvec{v}=(v_i)$, where $\{v_i\}$ are independent, of mean $0$ and variance $1$, we have
\begin{equation}
\trace(\cmat{H})=\mathbb{E}(\transpose{\vvec{v}}\cmat{H}\vvec{v}).
\end{equation}
Therefore, instead of directly calculating the huge tensor $\cmat{H}$, we only need the Hessian form $\transpose{\vvec{v}}\cmat{H}\vvec{v}$ of a few random vectors. In our experiments, we randomly sample $6$ vectors, with each entry drawn from $\{\pm1\}$. Preliminary experiments suggest this setting has relative error $<0.1$. 

The Hessian form $\transpose{\vvec{v}}\cmat{H}\vvec{v}$ can be calculated from Jacobian-Vector Products (a.k.a.~forward-mode autodiff). With some implementation techniques, one can even calculate the layer-wise Hessian form for multiple layers in one forward pass, by concatenating intermediate results from each layer along the way of autodiff.

\subsection{Power-Law Distribution of Forward Activations}
\label{sec:power_law_dist}

Let $\cvec{h}_t$ be the input hidden vector to a layer. In order for the layer to exhibit training spikes, our theory suggests that the eigenvalues of $\mathbb{E}(\cvec{h}_{t}\transpose{\cvec{h}}_{t})$ should follow a heavy-tailed distribution. In this section, we verify that it is indeed the case for our Transformer LM training.

In Figure~\ref{fig:loglog}, we sort the eigenvalues of $\mathbb{E}(\cvec{h}_{t}\transpose{\cvec{h}}_{t})$ in descending order, and show the log-log plots of eigenvalues over ranks, for different layers and training steps. The Normal and Spiky1 training configurations are compared in each sub-plot.

\begin{figure}[t!]
 \centering
 \begin{subfigure}[b]{0.49\textwidth}
     \centering
     \includegraphics[width=\textwidth]{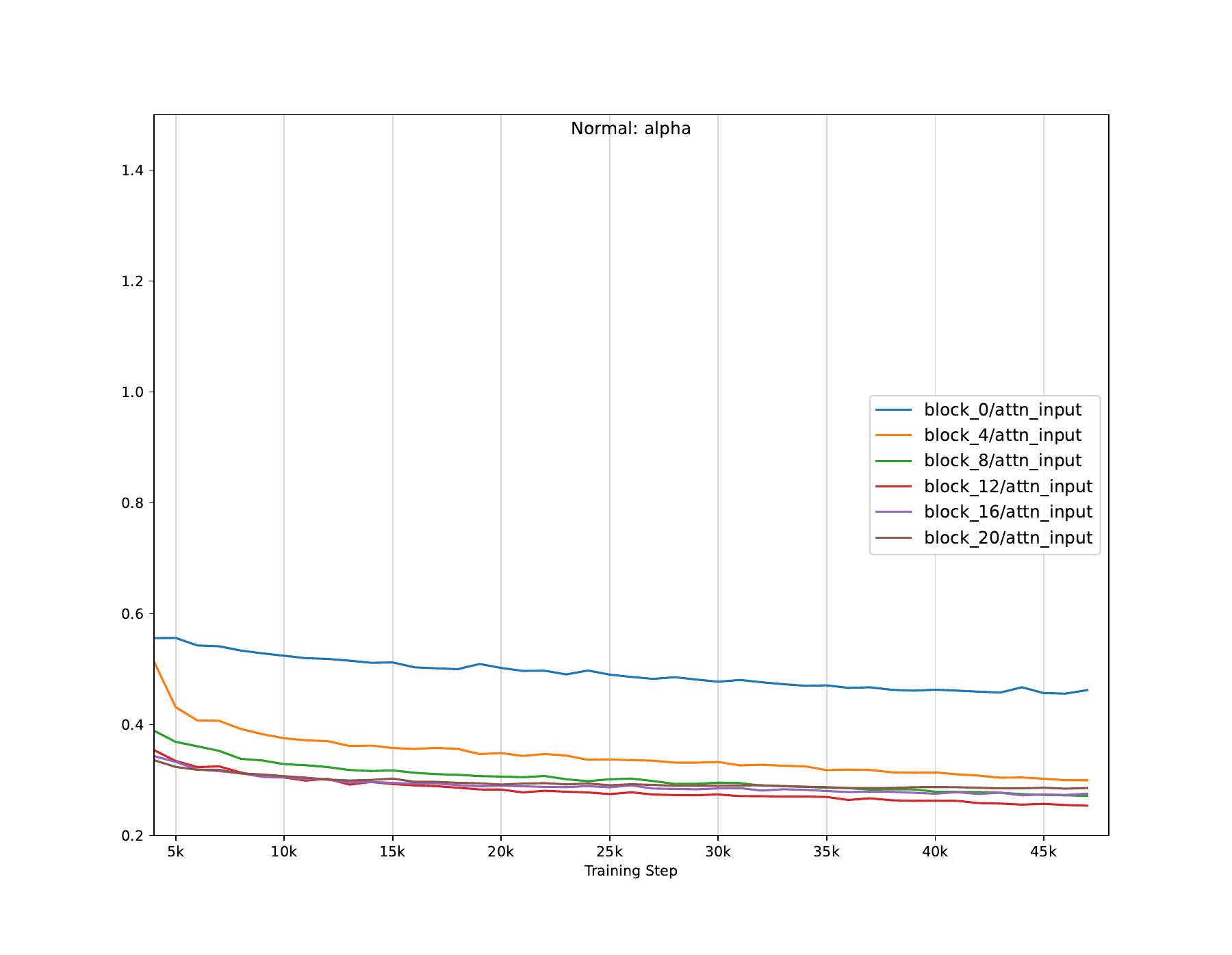}
     %\caption{(a)}
     %\label{fig:y equals x}
 \end{subfigure}
 \begin{subfigure}[b]{0.49\textwidth}
     \centering
     \includegraphics[width=\textwidth]{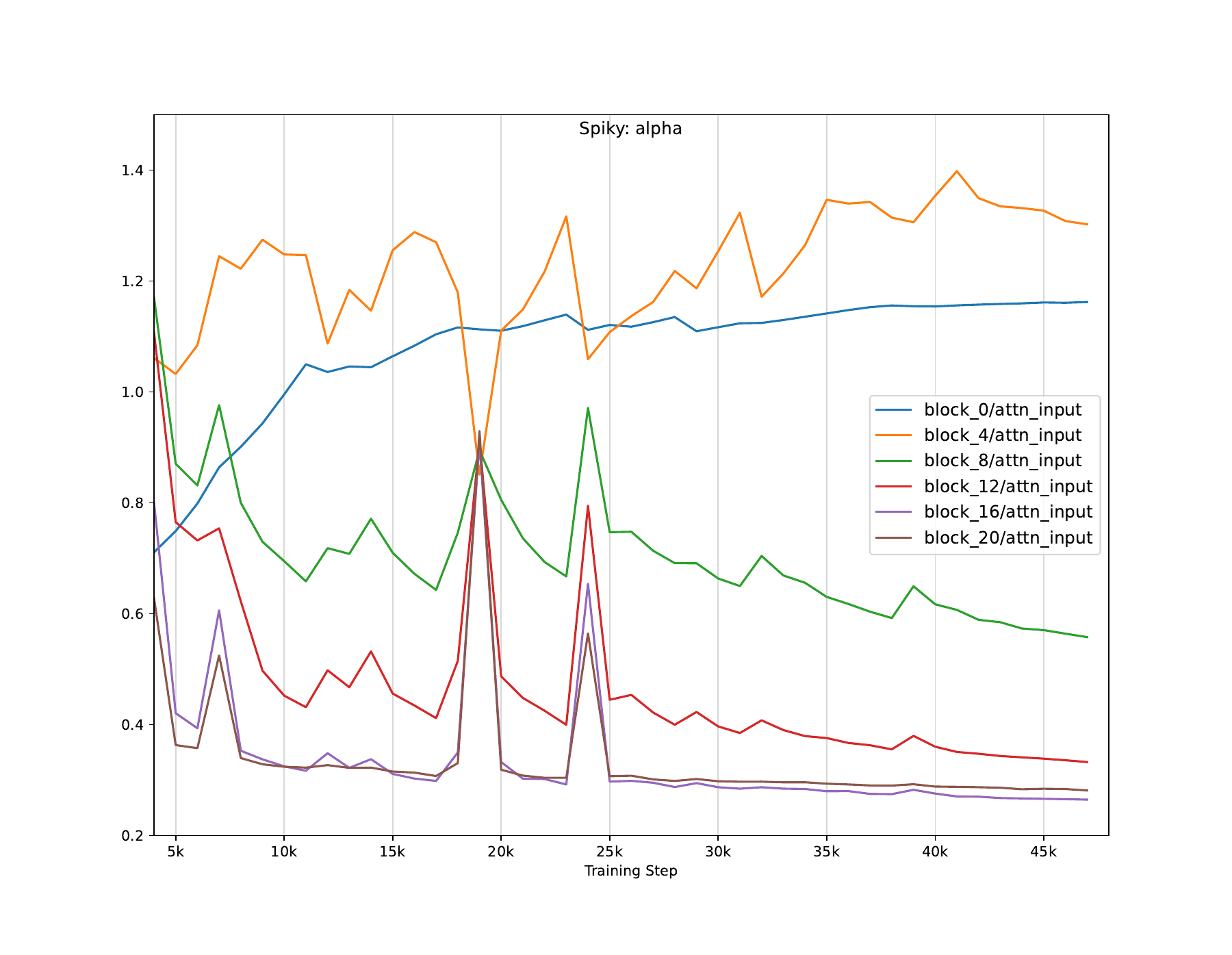}
     %\caption{(a)}
     %\label{fig:y equals x}
 \end{subfigure}\\
 \begin{subfigure}[b]{0.49\textwidth}
     \centering
     \includegraphics[width=\textwidth]{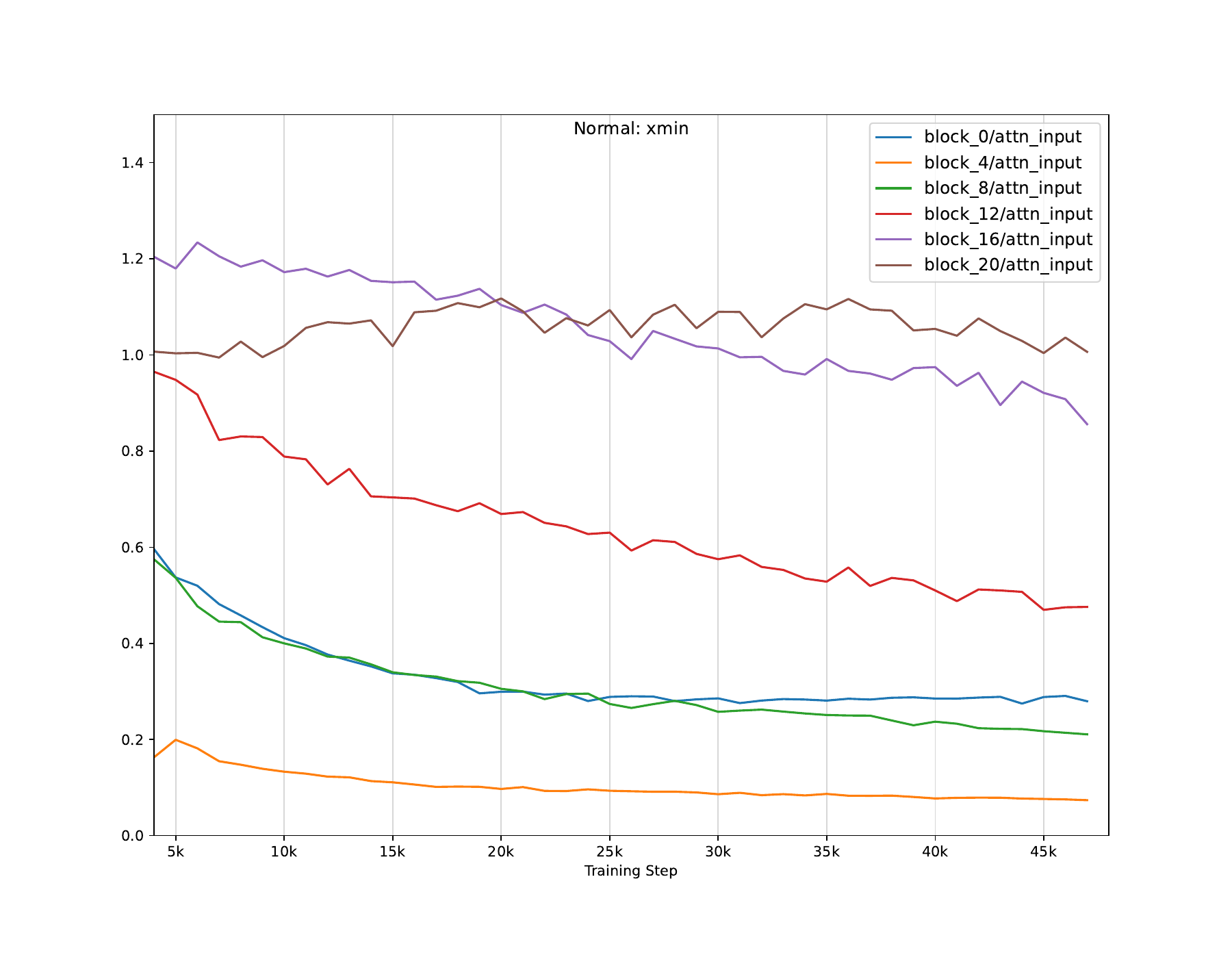}
     %\caption{(a)}
     %\label{fig:y equals x}
 \end{subfigure}
 \begin{subfigure}[b]{0.49\textwidth}
     \centering
     \includegraphics[width=\textwidth]{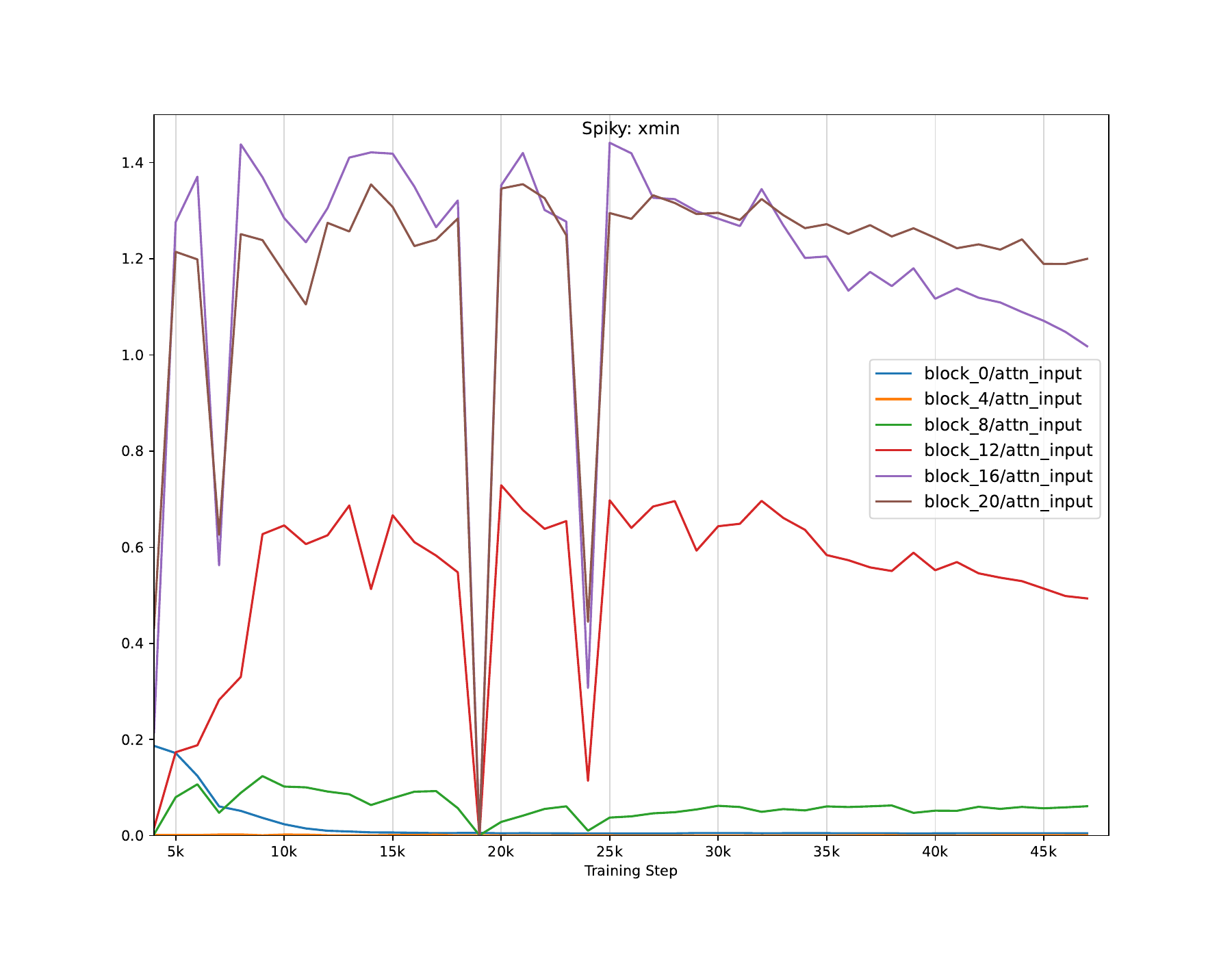}
     %\caption{(a)}
     %\label{fig:y equals x}
 \end{subfigure}\\
 \begin{subfigure}[b]{0.49\textwidth}
     \centering
     \includegraphics[width=\textwidth]{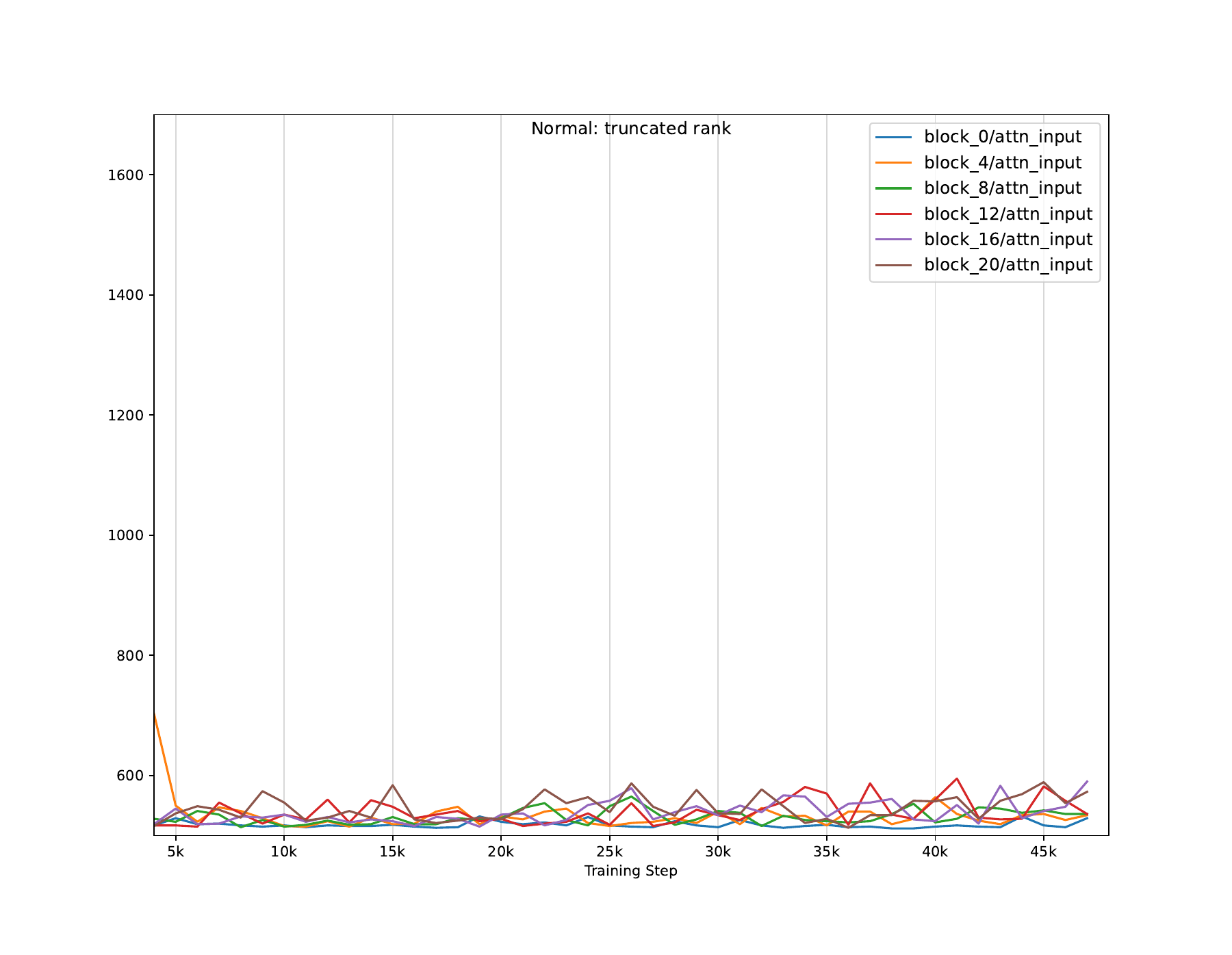}
     %\caption{(a)}
     %\label{fig:y equals x}
 \end{subfigure}
 \begin{subfigure}[b]{0.49\textwidth}
     \centering
     \includegraphics[width=\textwidth]{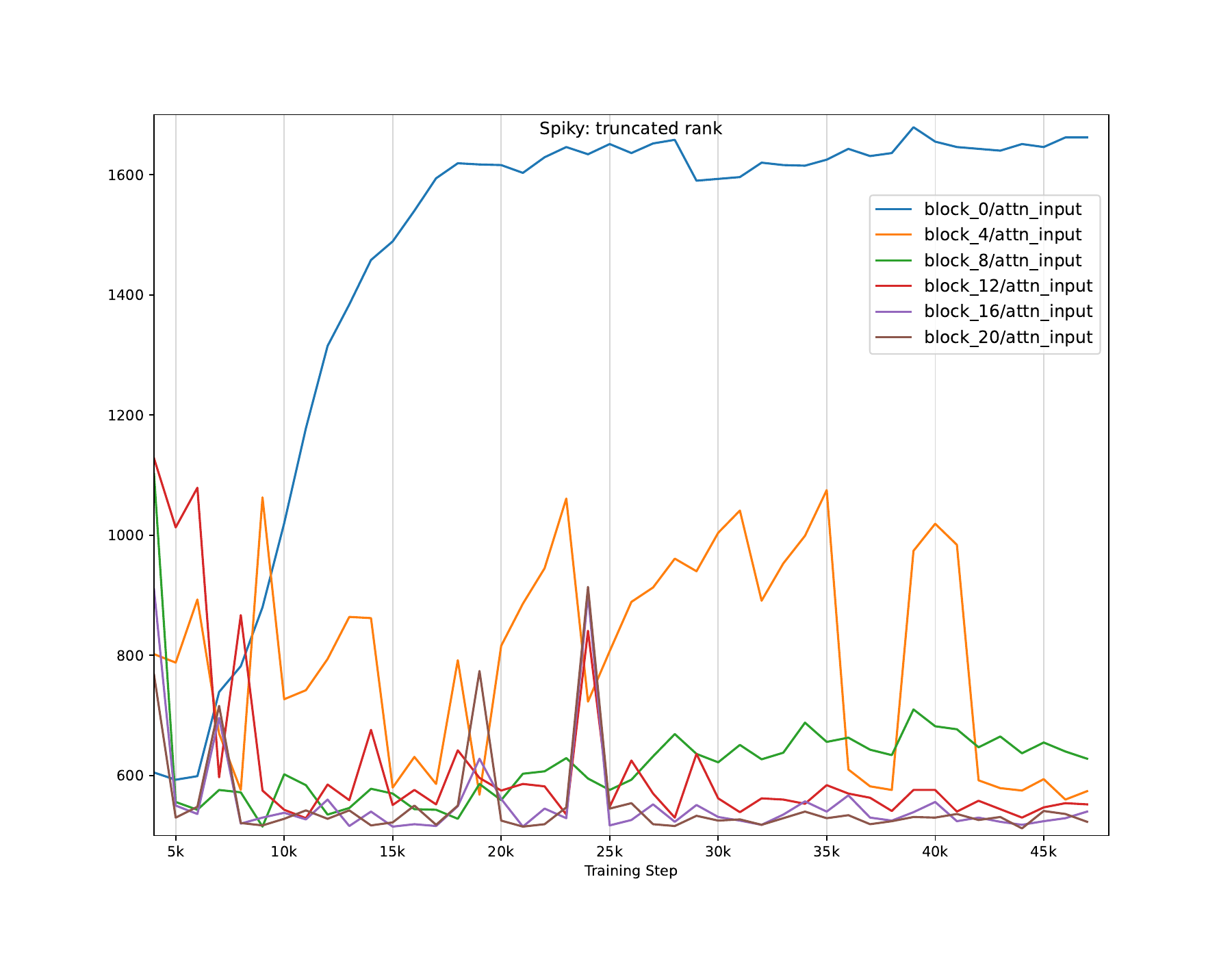}
     %\caption{(a)}
     %\label{fig:y equals x}
 \end{subfigure}
\caption{Evolution of power-law fitting statistics, across layers and training steps.}
\label{fig:powerlaw}
\end{figure}

In the log-log plots, eigenvalues indeed exhibit power-law like distributions, whereas the Spiky training generally shows steeper descending and crosses Normal. In other words, eigenvalues from Spiky are more heavy-tailed.

Quantitatively, we follow \citet{doi:10.1137/070710111} to fit truncated power-law distributions to the log-log plots, which estimates the power-law exponent by maximizing the likelihood, and estimates the truncation threshold by minimizing the Kolmogorov–Smirnov (KS) statistic. The estimated power-law exponent $\alpha$, truncation rank $r$, and the KS statistic are shown in Figure~\ref{fig:loglog}.

For lower layers (e.g., \texttt{block\_0}), the eigenvalues distribution of Spiky seems significantly different from Normal, and the gap remains unclosed until the end. For upper layers (e.g., \texttt{block\_20}) the distributions are more similar, and they seem to converge to the same towards the end of training (e.g., \texttt{block\_\{12,20\}}).

This tendency is documented in Figure~\ref{fig:powerlaw}, where we plot the evolution of power-law fitting statistics along the whole training process. Generally, Spiky training exhibits larger power-law exponent $\alpha$, larger truncation threshold $x_\textrm{min}$ for upper layers and smaller for lower layers, and larger truncated rank $r$ especially for lower layers. The evolution of the statistics is also more unstable for Spiky.

\subsection{Empirical Verification of Prediction~\ref{pred:grad_norm}}
\label{sec:verification_of_pred}

In Prediction~\ref{pred:grad_norm}, we have derived a formula to calculate the norm of per-example gradients, given as Equation~\eqref{eq:grad_norm}. We empirically verify the formula in this section.

\begin{figure}[t]
 \centering
 \begin{subfigure}[b]{0.49\textwidth}
     \centering
     \includegraphics[width=\textwidth]{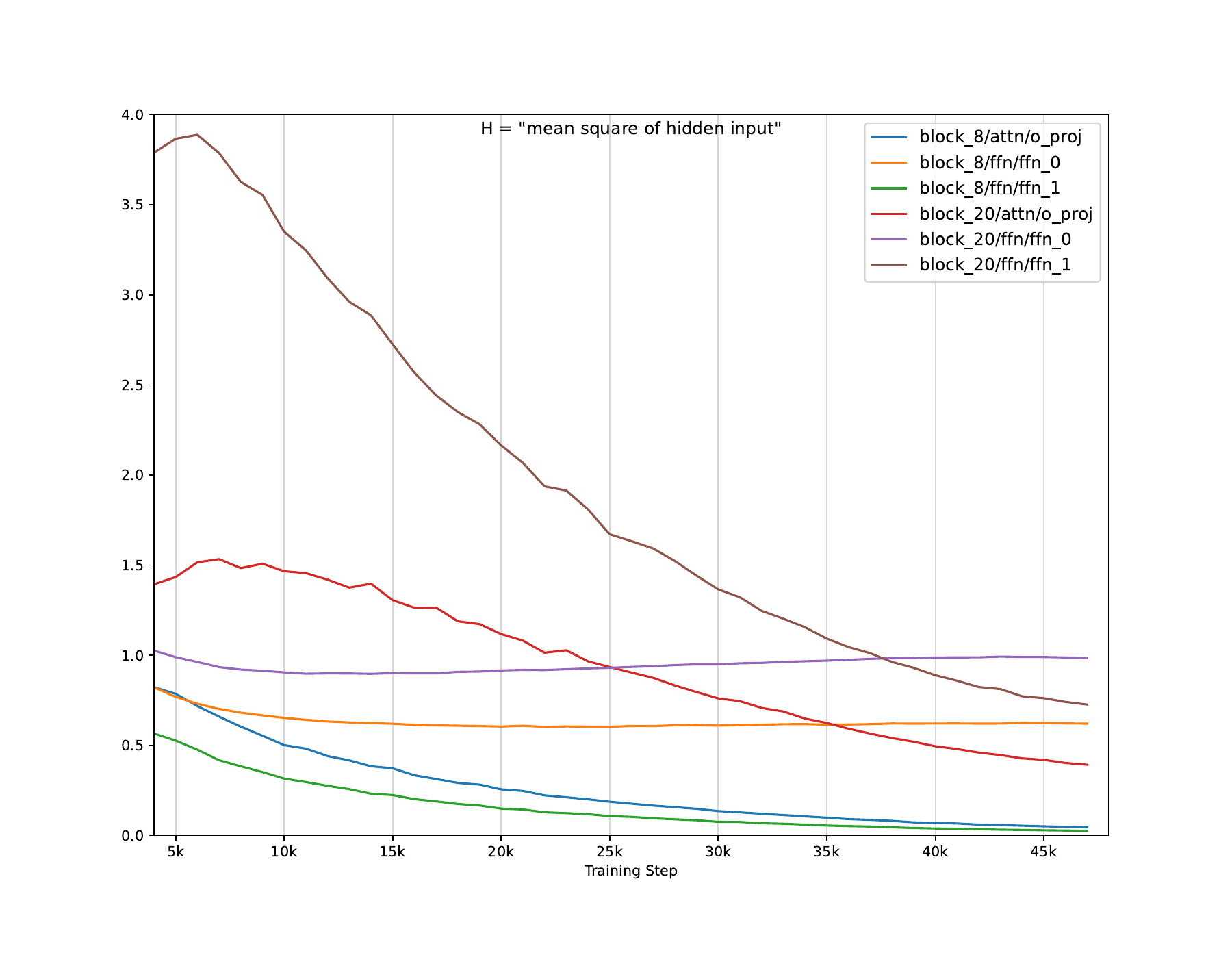}
     %\caption{(a)}
     %\label{fig:y equals x}
 \end{subfigure}
 \begin{subfigure}[b]{0.49\textwidth}
     \centering
     \includegraphics[width=\textwidth]{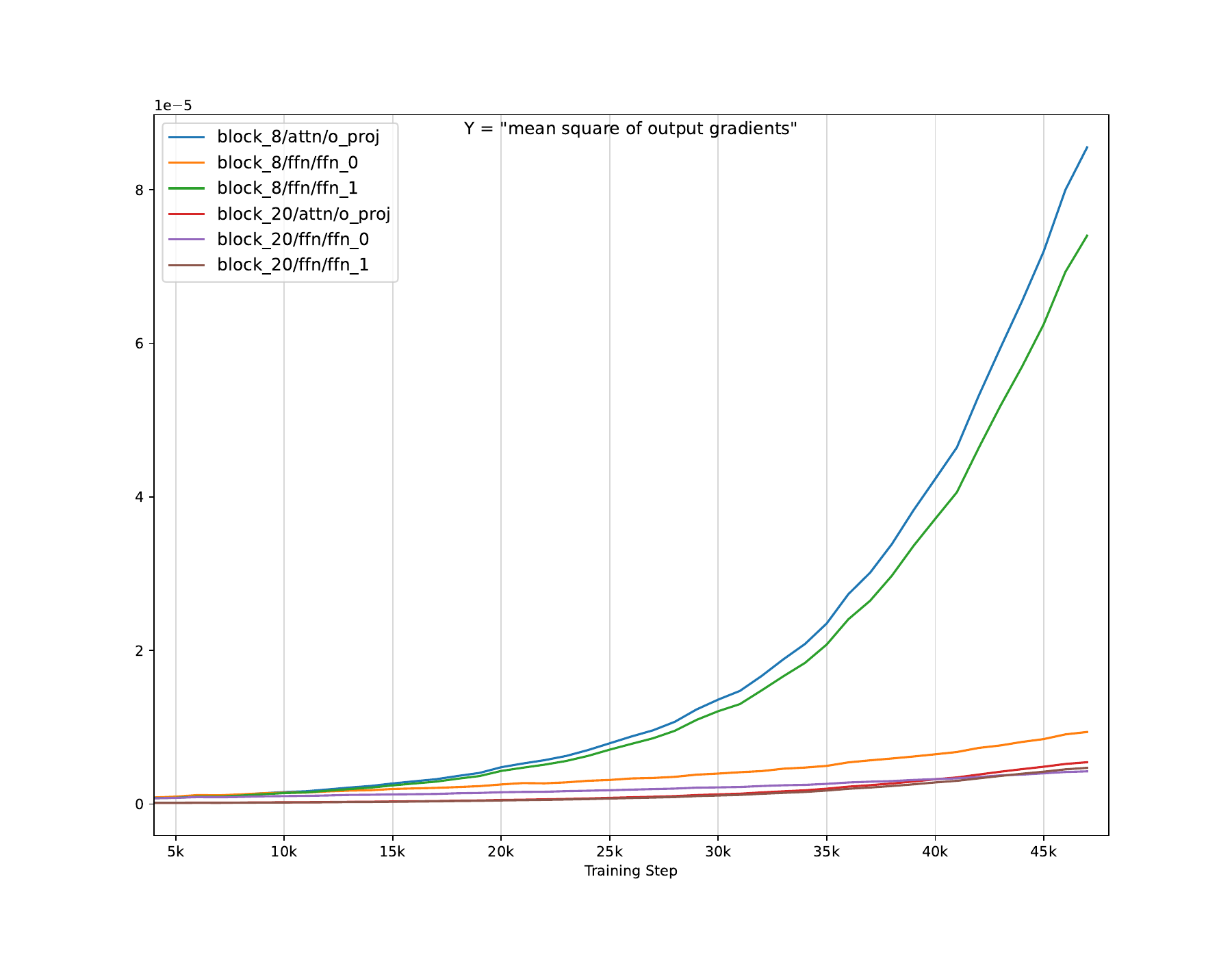}
     %\caption{(a)}
     %\label{fig:y equals x}
 \end{subfigure}\\
 \begin{subfigure}[b]{0.49\textwidth}
     \centering
     \includegraphics[width=\textwidth]{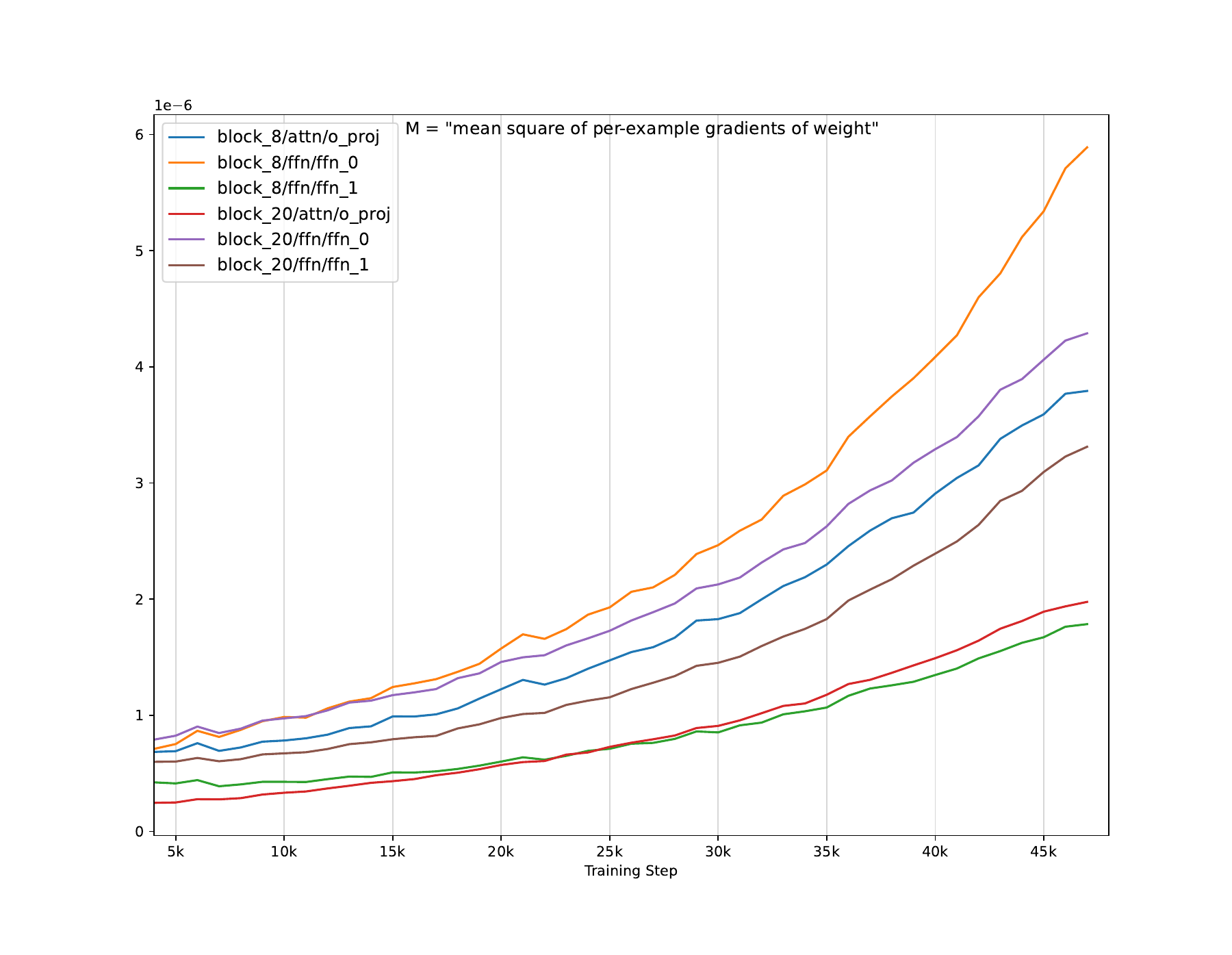}
     %\caption{(a)}
     %\label{fig:y equals x}
 \end{subfigure}
 \begin{subfigure}[b]{0.49\textwidth}
     \centering
     \includegraphics[width=\textwidth]{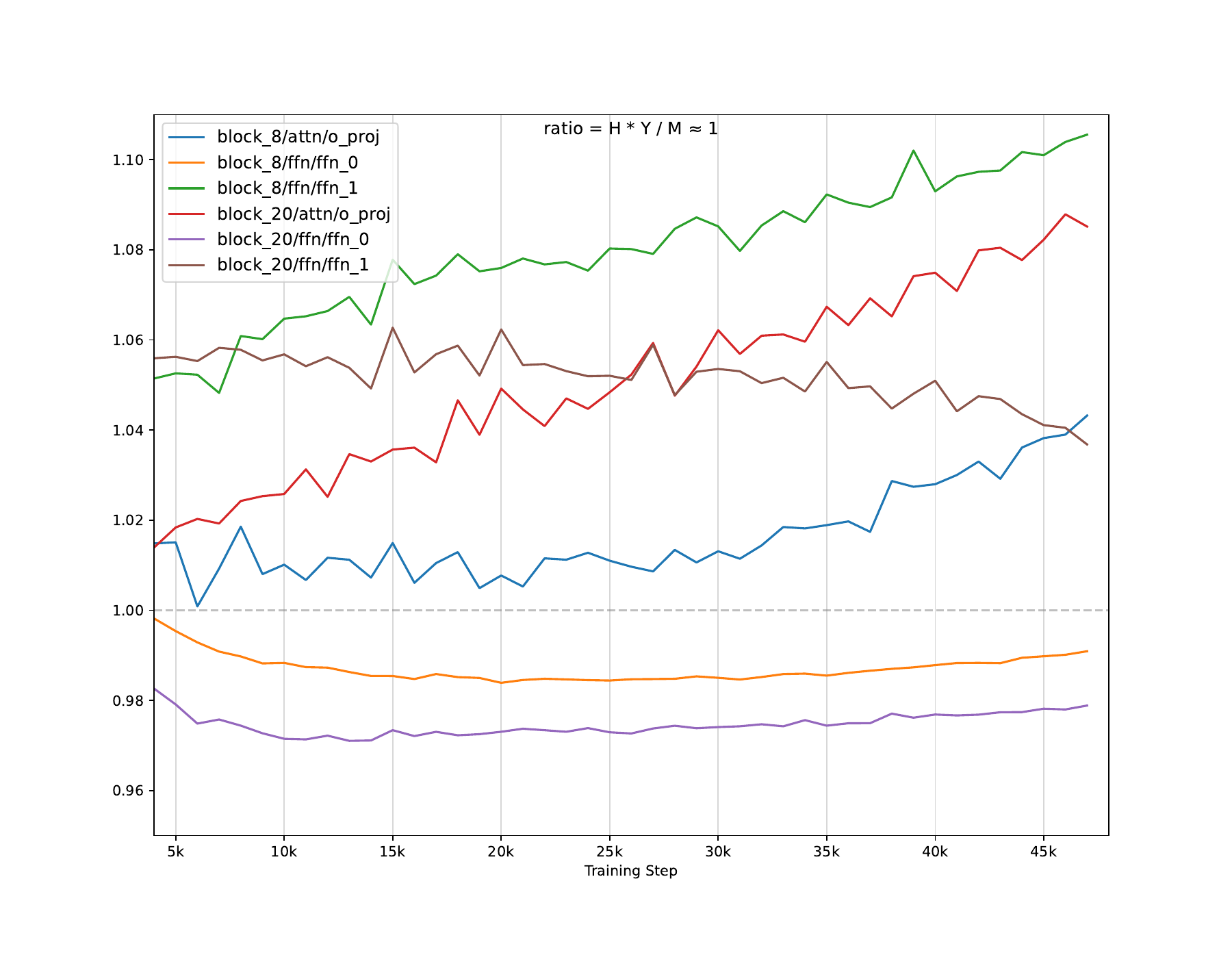}
     %\caption{(a)}
     %\label{fig:y equals x}
 \end{subfigure}
\caption{Empirical verification of Equation~\eqref{eq:grad_norm}.}
\label{fig:prediction_verify}
\end{figure}

Recall that layer $\vmat{W}$ is a $\cscalar{C}\times\cscalar{D}$ matrix, the input $\cvec{h}_t$ a $\cscalar{D}$-dimensional, and the output gradient $\nabla\vvec{y}_t$ a $\cscalar{C}$-dimensional vector. In Figure~\ref{fig:prediction_verify}, we plot:
\begin{enumerate}
 \item[(a)] $\mathfrak{H}\eqdef\sum_{i=1}^\cscalar{D}\lambda_i/\cscalar{D}=\mathbb{E}(\lVert\cvec{h}_t\rVert^2)/\cscalar{D}$, the mean square of hidden inputs.
 \item[(b)] $\mathfrak{Y}\eqdef\mathbb{E}(\lVert\nabla\vvec{y}_{t}\rVert^2)/\cscalar{C}$, the mean square of output gradients.
 \item[(c)] $\mathfrak{M}\eqdef\mathbb{E}(\lVert\nabla\vmat{W}_{t}\lVert^2) / (\cscalar{C}\times\cscalar{D})$, the mean square of per-example gradients of the model weight.
 \item[(d)] The ratio $\mathfrak{H}\mathfrak{Y} / \mathfrak{M}$, which should be close to $1$ according to Equation~\eqref{eq:grad_norm}.
\end{enumerate}
We show different types of layers across training steps. The norm of gradients can vary multiple magnitudes during training, yet the ratio plot in Figure~\ref{fig:prediction_verify} shows that the relative error of Equation~\eqref{eq:grad_norm} is less than $0.1$, mostly within $5\%$.

\subsection{Is the ``Batch-size'' Large or Small?}
\label{sec:equiv_batch_factor_dnn_breakdown}

\begin{figure}[t]
 \centering
 \begin{subfigure}[b]{0.32\textwidth}
     \centering
     \includegraphics[width=\textwidth]{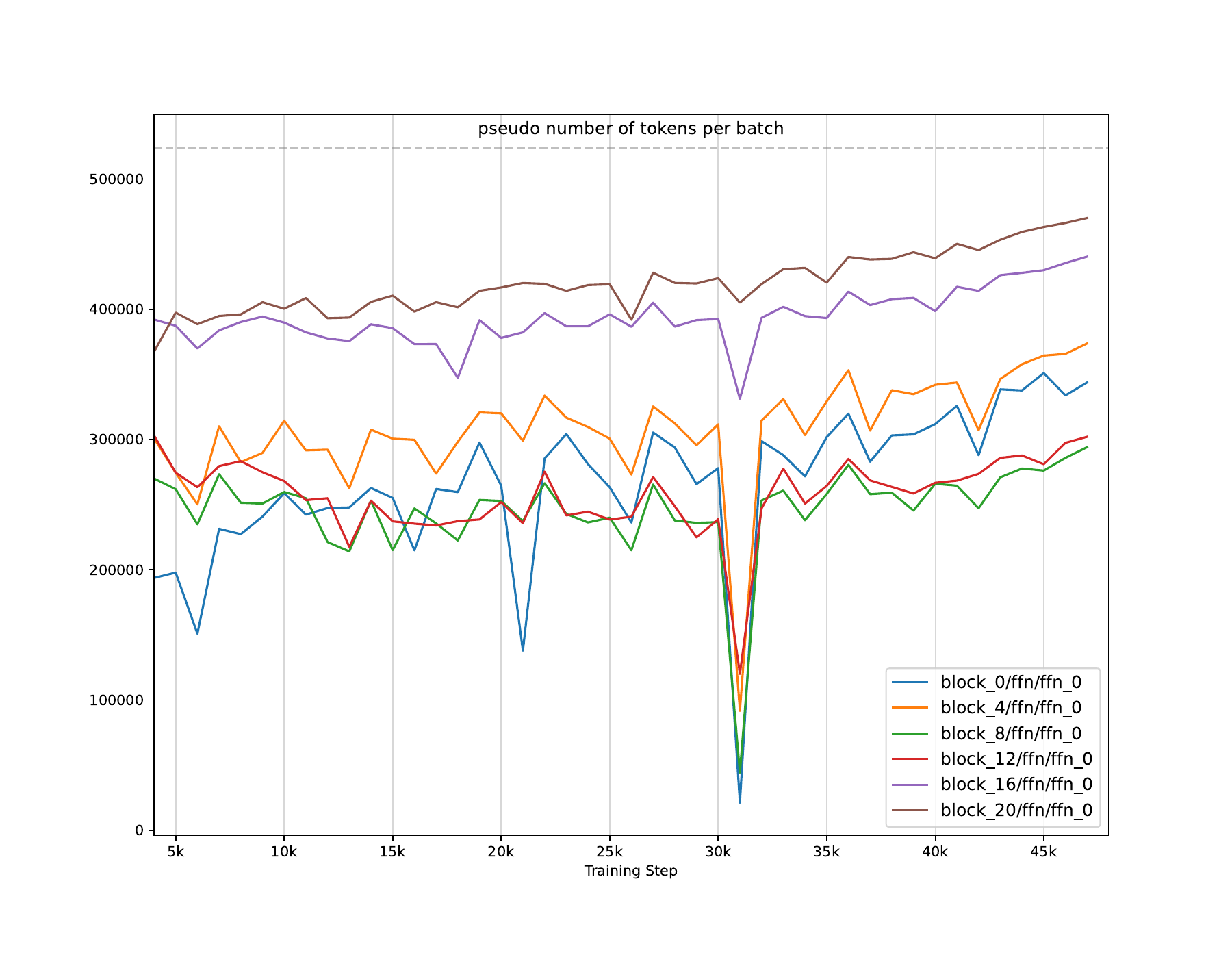}
     % \caption{$\mathbb{E}(\lVert\nabla\vmat{W}_{t}\lVert^2)/\mathbb{E}(\lVert\nabla_\cscalar{B}\vmat{W}_{t}\lVert^2)$}
     % \label{fig:equiv_batch_factor_a}
 \end{subfigure}
 \begin{subfigure}[b]{0.32\textwidth}
     \centering
     \includegraphics[width=\textwidth]{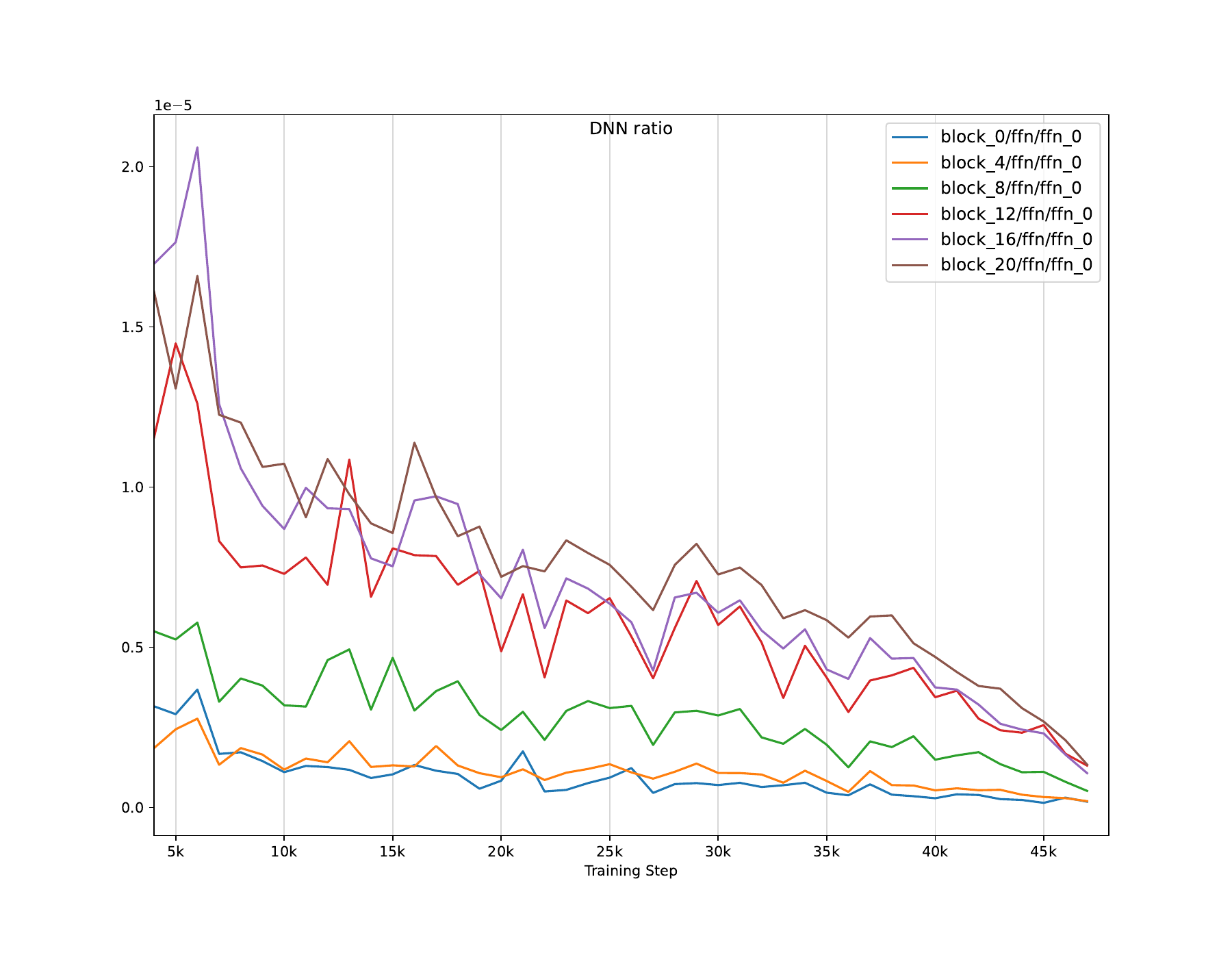}
     % \caption{(b)}
     % \label{fig:equiv_batch_factor_b}
 \end{subfigure}
 \begin{subfigure}[b]{0.32\textwidth}
     \centering
     \includegraphics[width=\textwidth]{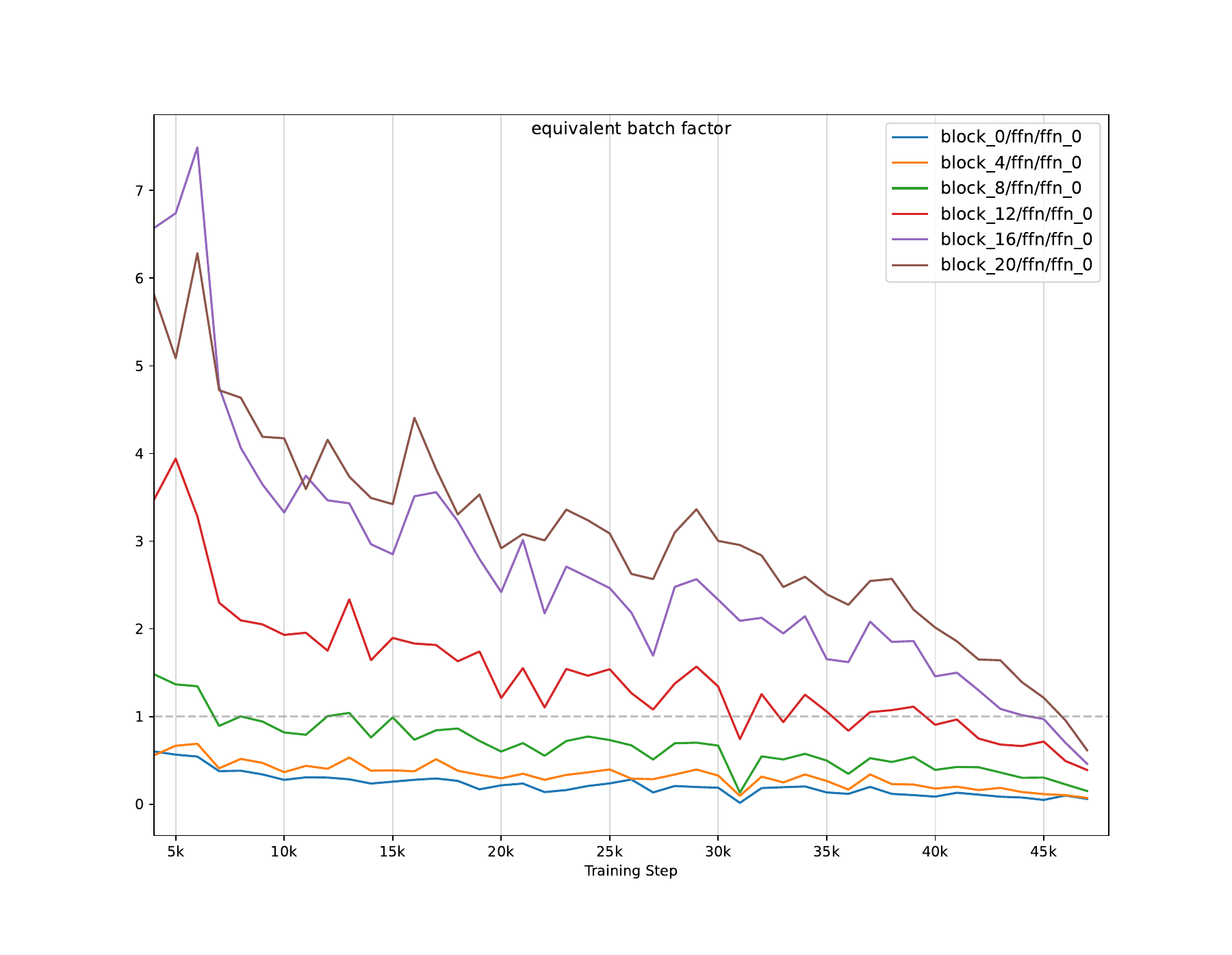}
     % \caption{$\mathfrak{B}$}
     % \label{fig:equiv_batch_factor_c}
 \end{subfigure}
\caption{Breakdown of the equivalent batch factor.}
\label{fig:equiv_batch_factor}
\end{figure}

Specific to the training of Transformer LMs, there is the question of what should be considered as the ``batch-size''. Is it the number of tokens per batch, because each token is corresponding to a summand in the training loss function? Or is it the number of sequences per batch, because tokens within a sequence can have correlations, and the unit of sequence is what is guaranteed to be independent in a batch? Since recent LMs are trained with very long sequences ($\geq$10k), the answer to this question affects our understanding of whether the training of Transformer LMs should be theoretically considered as having a ``large'' or ``small'' batch-size. 

The difference between ``large'' and ``small'' batch-sizes is drastic in theory: If the batch-size is small (compared to $\sqrt{\cscalar{D}}$, where $\cscalar{D}$ is the hidden dimension of the model), the gradient noise plays an important role in stochastic optimization, of which the effect can be accurately analyzed as we demonstrated in this work; On the other hand, our theoretical calculation becomes less accurate as the batch-size increases, and when the batch-size is extremely large as close to a full-batch setting, the theory regarding full-batch gradient descent becomes more relevant.

In this section, we provide a perspective from the \emph{equivalent batch factor} concept. From \eqref{eq:bgrad_norm}, equivalent batch factor $\mathfrak{B}$ directly affects the largeness of gradient noise; In Figure~\ref{fig:equiv_batch_factor}~(Right), we plot $\mathfrak{B}$ for different layers and training steps, as empirically measured according to \eqref{eq:equiv_batch_factor_def}. We observe that $\mathfrak{B}$ is small (close to $1$), which means that we are in a typical stochastic optimization regime.

Furthermore, $\mathfrak{B}$ can be decomposed into three factors as below.
\[
\mathfrak{B}=\left(\frac{\mathbb{E}(\lVert\nabla\vmat{W}_{t}\lVert^2)}{\mathbb{E}(\lVert\nabla_\cscalar{B}\vmat{W}_{t}\lVert^2)}\right)\left(\frac{\sum_{i=1}^\cscalar{D}\lambda_i\mathbb{E}(\lVert\nabla\vvec{y}_{t}\rVert^2)}{\mathbb{E}(\lVert\nabla\vmat{W}_{t}\lVert^2)}\right)\left( \frac{\sum_{i=1}^\cscalar{D}\lambda_i}{\sum_{i=1}^\cscalar{D}\lambda_i^2}\dfrac{\lVert\mathbb{E}(\nabla\vmat{W}_{t})\rVert^2}{\mathbb{E}(\lVert\nabla\vvec{y}_{t}\rVert^2)}\right).
\]
The first factor, $\mathbb{E}(\lVert\nabla\vmat{W}_{t}\lVert^2)/\mathbb{E}(\lVert\nabla_\cscalar{B}\vmat{W}_{t}\lVert^2)$, is plotted in Figure~\ref{fig:equiv_batch_factor}~(Left), which turns out to be slightly smaller than the number of tokens per batch (shown as dashed line in the figure), which is not completely a surprise in light of \eqref{eq:b2egg}. For recent LMs, this is a very large number. 

The second factor, $\sum_{i=1}^\cscalar{D}\lambda_i\mathbb{E}(\lVert\nabla\vvec{y}_{t}\rVert^2)/\mathbb{E}(\lVert\nabla\vmat{W}_{t}\lVert^2)$, is what we have investigated in \S~\ref{sec:verification_of_pred} and Figure~\ref{fig:prediction_verify}, which is close to $1$ as predicted.

The third factor is what makes $\mathfrak{B}$ small. It measures how far a DNN model deviates from \eqref{eq:rf_simplification_volterra} and \eqref{eq:lapprox}, which is satisfied by Random Feature models and would make this factor $\approx 1$. In contrast, the empirical measure in our case turns out to be very small, as shown in Figure~\ref{fig:equiv_batch_factor}~(Middle). It suggests that the point-wise loss of DNNs brings very different gradients to each training example, and constitutes the major source of gradient noise in training of DNNs.

Somehow magically, the first factor and the third factor cancel each other to make $\mathfrak{B}$ close to $1$. If this is not a coincidence (since the training recipe we adopted in \S~\ref{sec:llm_train_settings} has been tuned on batch-size as well), it might suggest heuristics for tuning the batch-size hyperparameter: Because the first factor depends on the batch-size while the third factor does not, one can empirically measure the third factor and tune the batch-size to make $\mathfrak{B}\approx 1$. We leave this direction to future work.

\subsection{The Appropriate \texorpdfstring{$\beta_2$}{beta2}}
\label{sec:appropriate_beta2}

\begin{table}[t]
\centering
\begin{tabular}{cccc}
\toprule
$\beta_2=0.95$ & $\beta_2=0.982$ & $\beta_2=0.992$ & $\beta_2$ schedule \\
\midrule
2.4421         & 2.4405          & 2.4502          & \textbf{2.4404} \\
\bottomrule
\end{tabular}
\caption{Final validation loss for different $\beta_2$ settings.}
\label{tab:beta2_eval}
\end{table}

\begin{figure}[t]%{0.32\textwidth}
\centering
\includegraphics[width=0.7\textwidth]{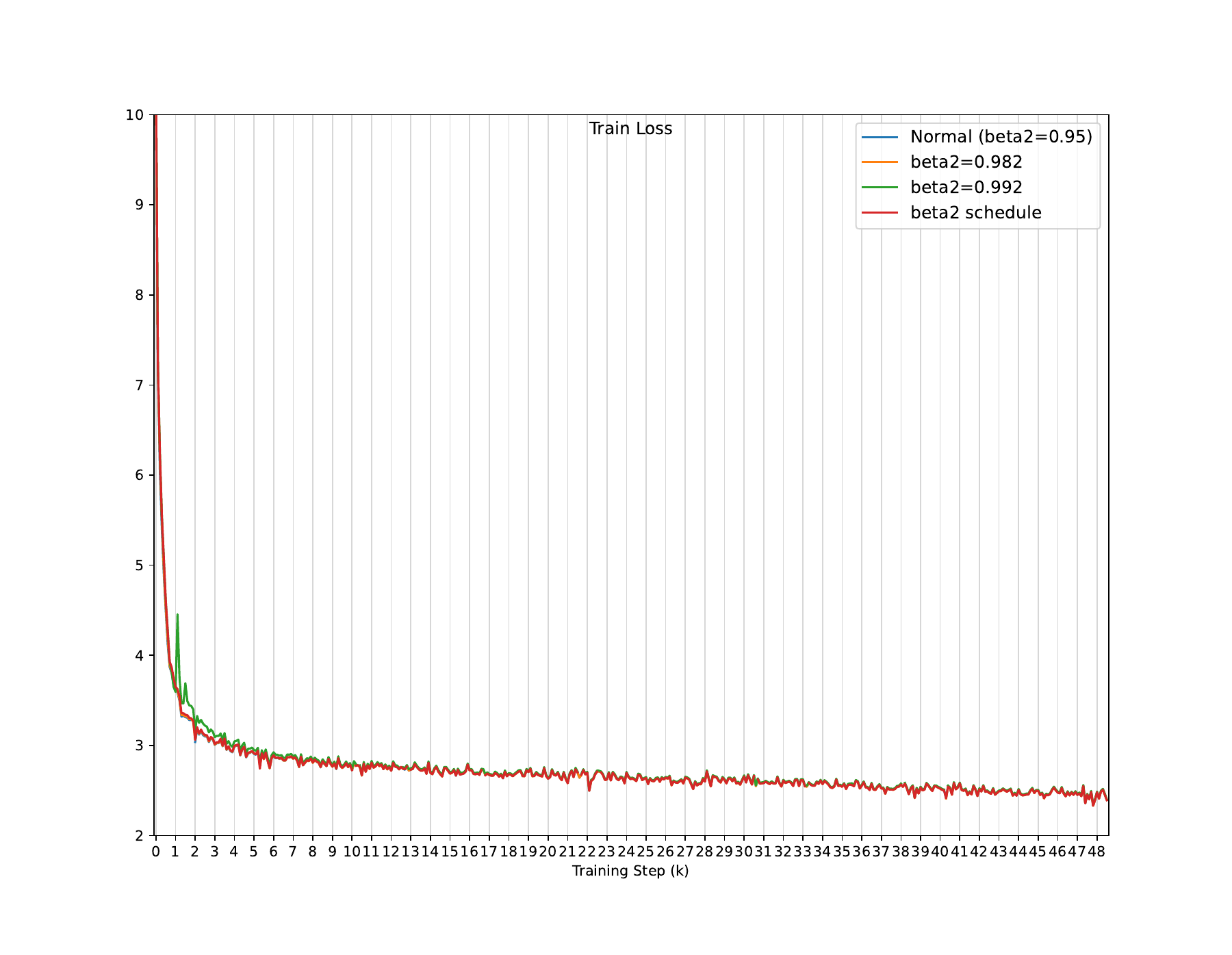}
\caption{Training loss for different $\beta_2$ settings.}
\label{fig:beta2_train_loss}
\end{figure}

For our case-study, we know that the hyperparameter $\beta_2$ is causing training spikes. Is there a recipe to derive the appropriate $\beta_2$, possibly from empirical metrics, so that one can avoid the risk of spiky training? In this section, we propose and verify such a recipe. 

\begin{figure}[t!]
 \centering
 \begin{subfigure}[b]{0.49\textwidth}
     \centering
     \includegraphics[width=\textwidth]{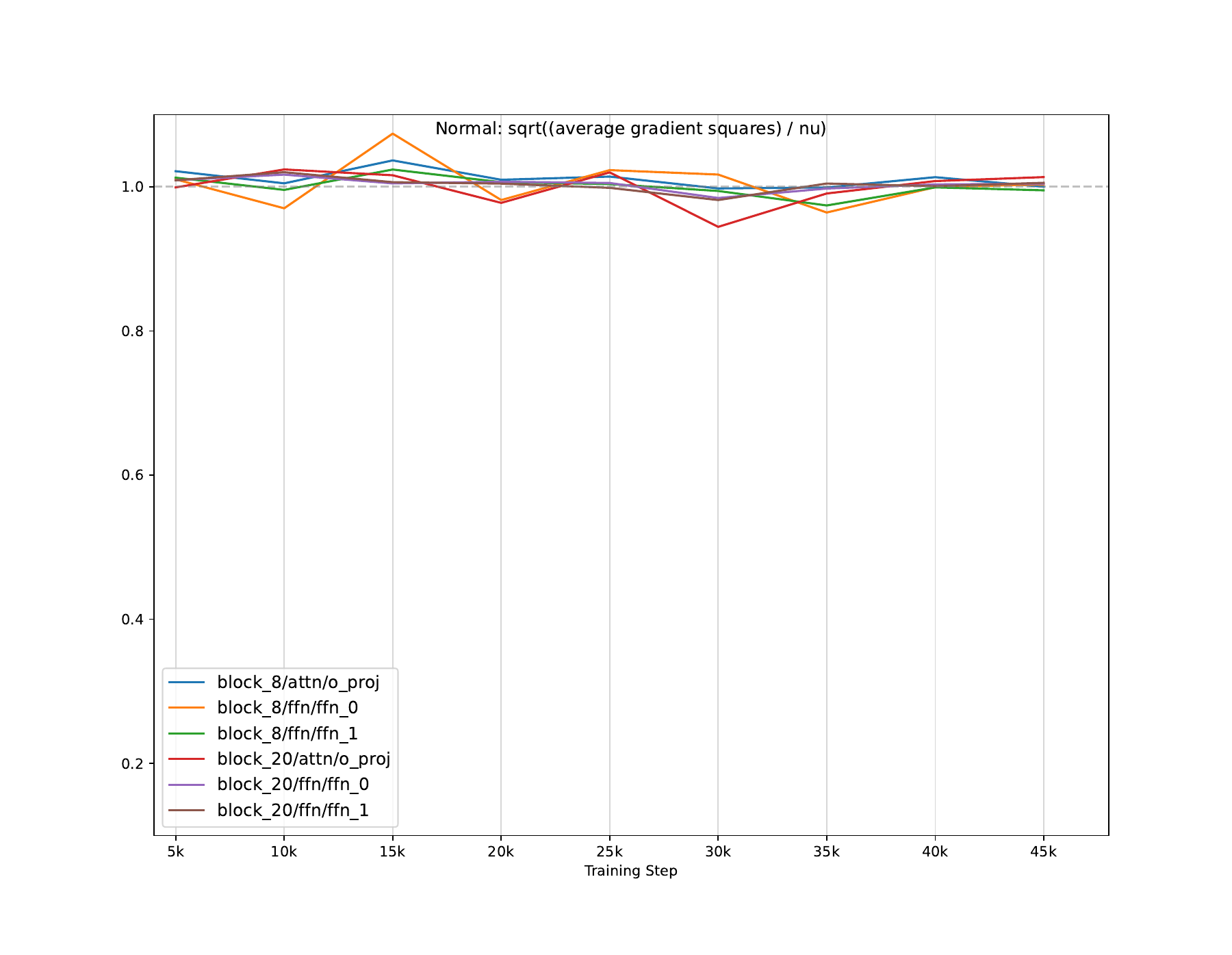}
     %\caption{(a)}
     %\label{fig:y equals x}
 \end{subfigure}
 \begin{subfigure}[b]{0.49\textwidth}
     \centering
     \includegraphics[width=\textwidth]{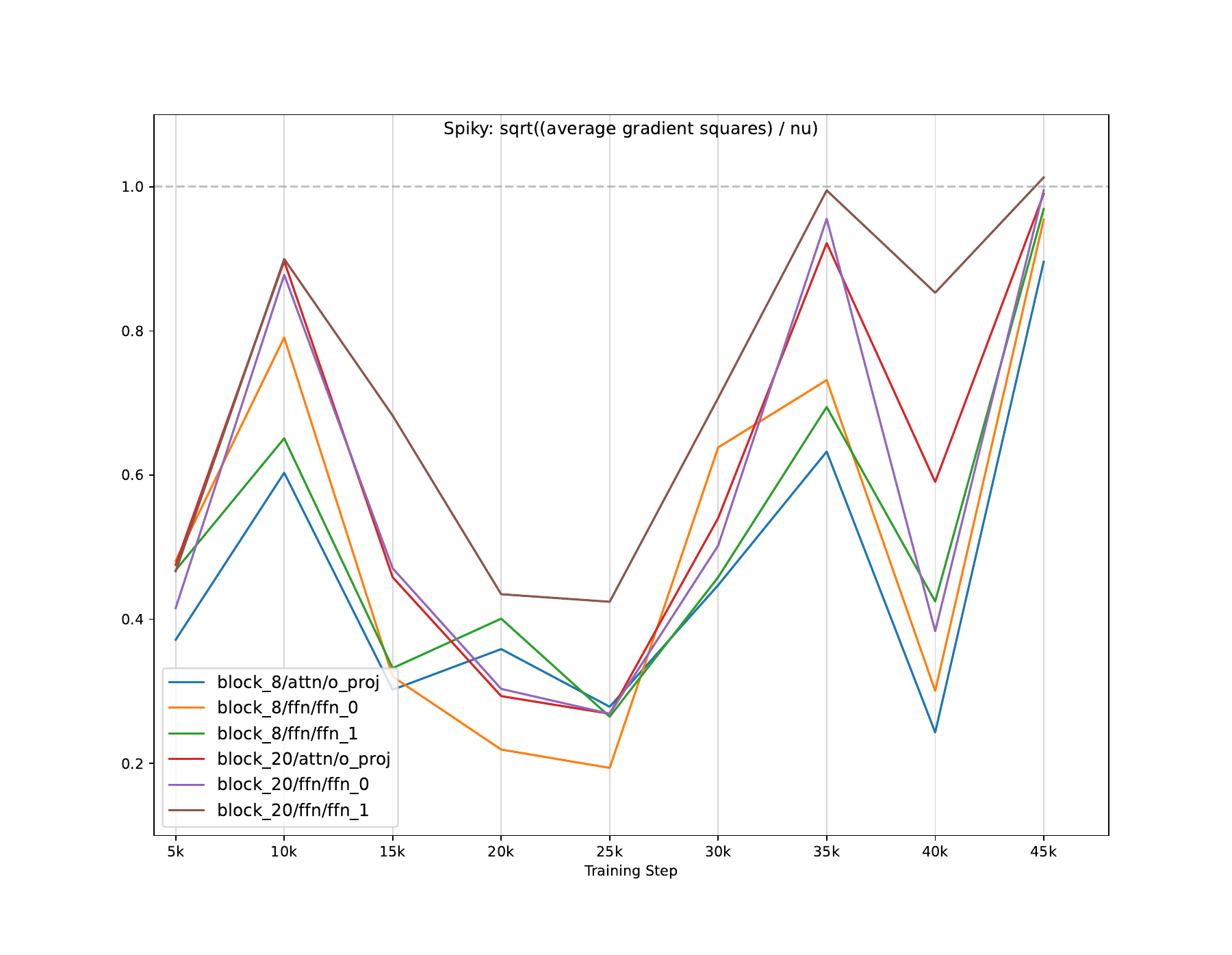}
     %\caption{(a)}
     %\label{fig:y equals x}
 \end{subfigure}\\
 \begin{subfigure}[b]{0.49\textwidth}
     \centering
     \includegraphics[width=\textwidth]{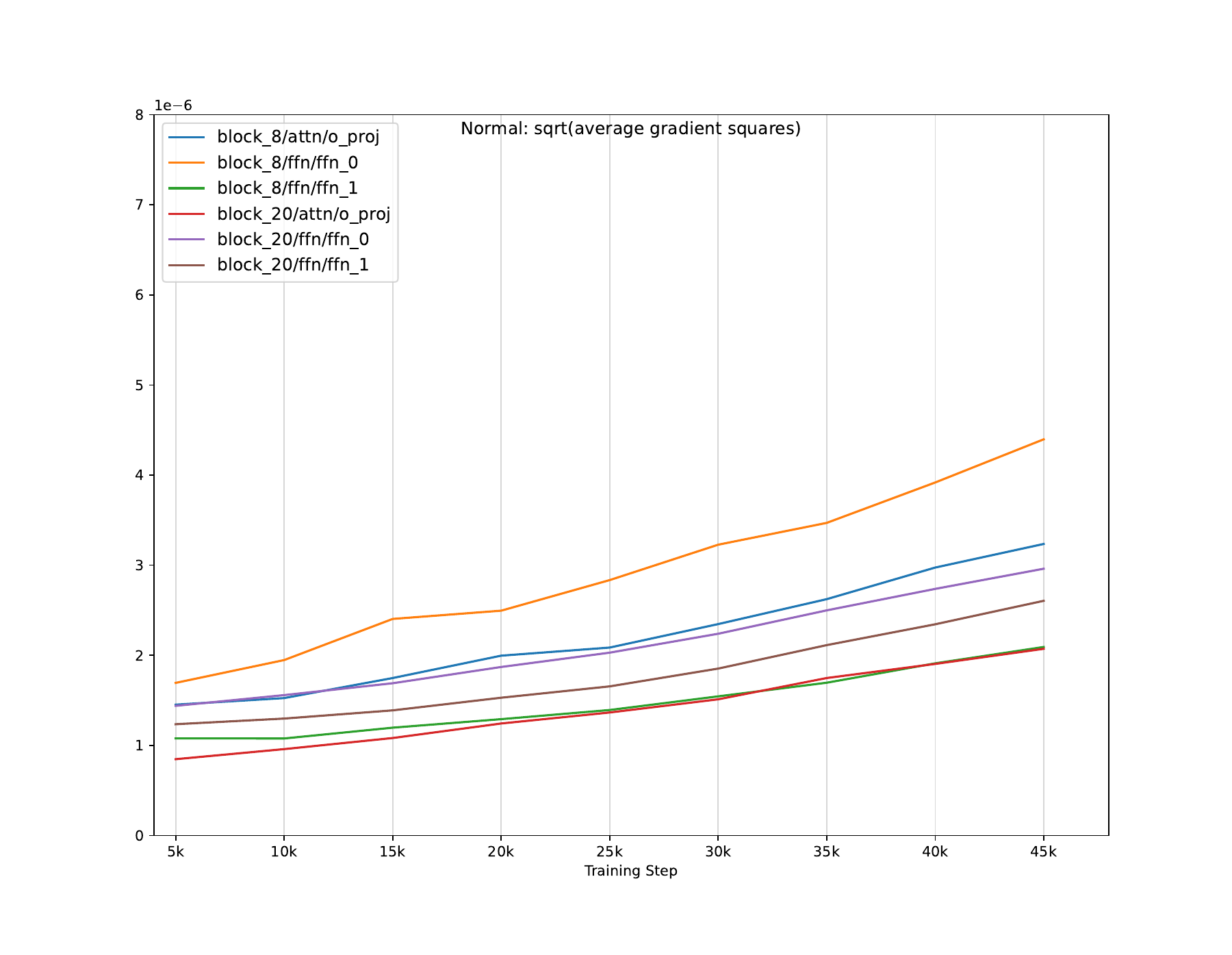}
     %\caption{(a)}
     %\label{fig:y equals x}
 \end{subfigure}
 \begin{subfigure}[b]{0.49\textwidth}
     \centering
     \includegraphics[width=\textwidth]{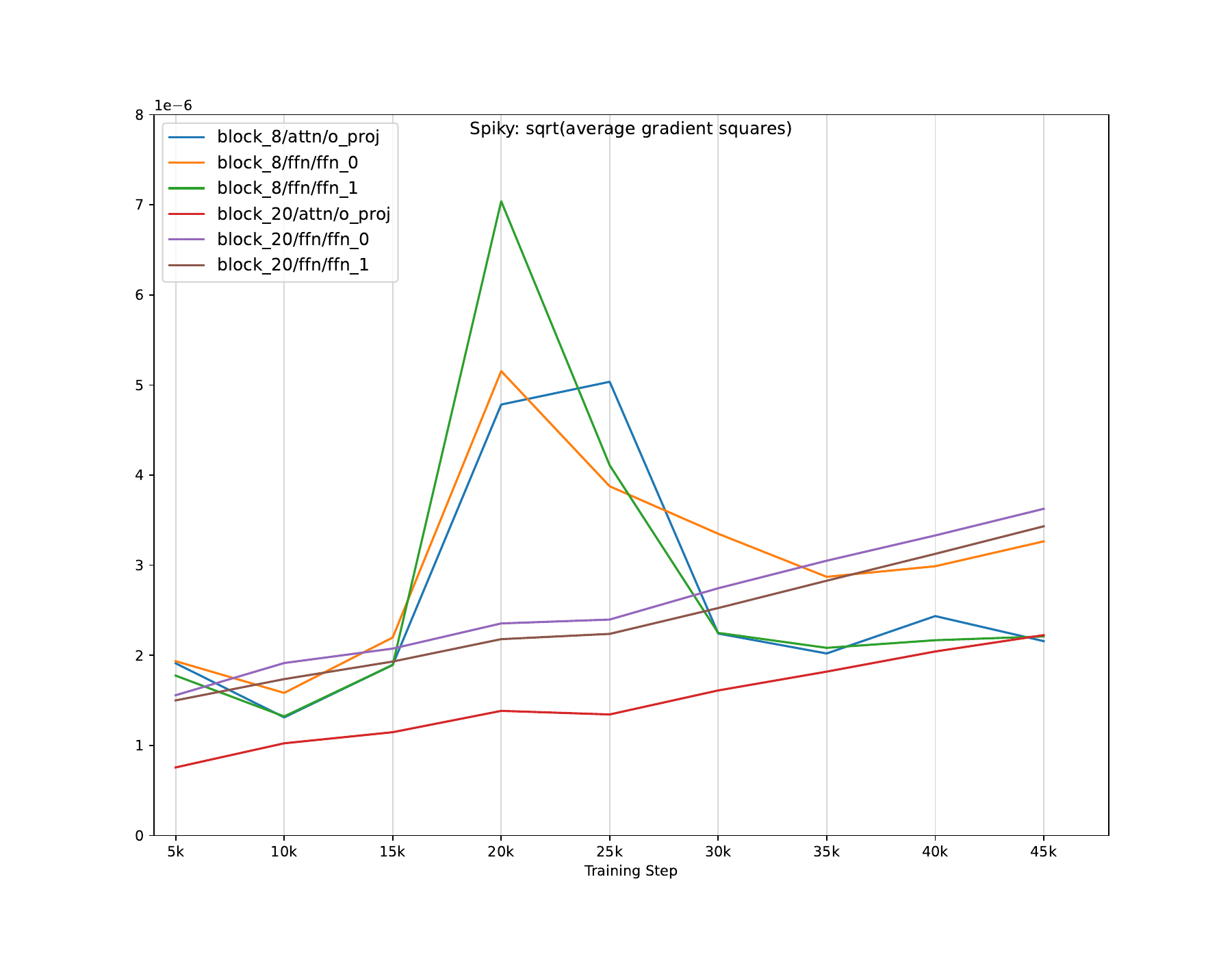}
     %\caption{(a)}
     %\label{fig:y equals x}
 \end{subfigure}\\
 \begin{subfigure}[b]{0.49\textwidth}
     \centering
     \includegraphics[width=\textwidth]{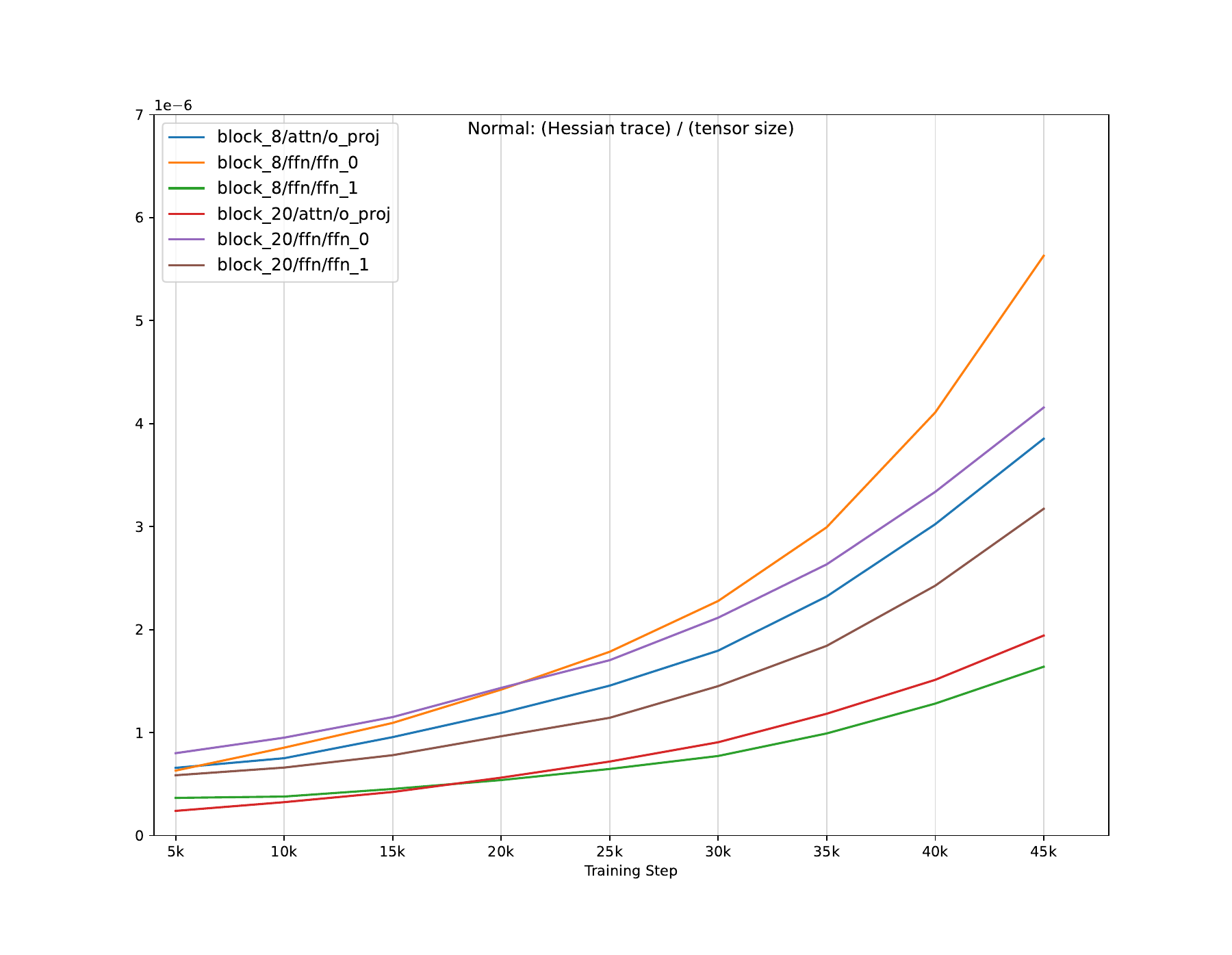}
     %\caption{(a)}
     %\label{fig:y equals x}
 \end{subfigure}
 \begin{subfigure}[b]{0.49\textwidth}
     \centering
     \includegraphics[width=\textwidth]{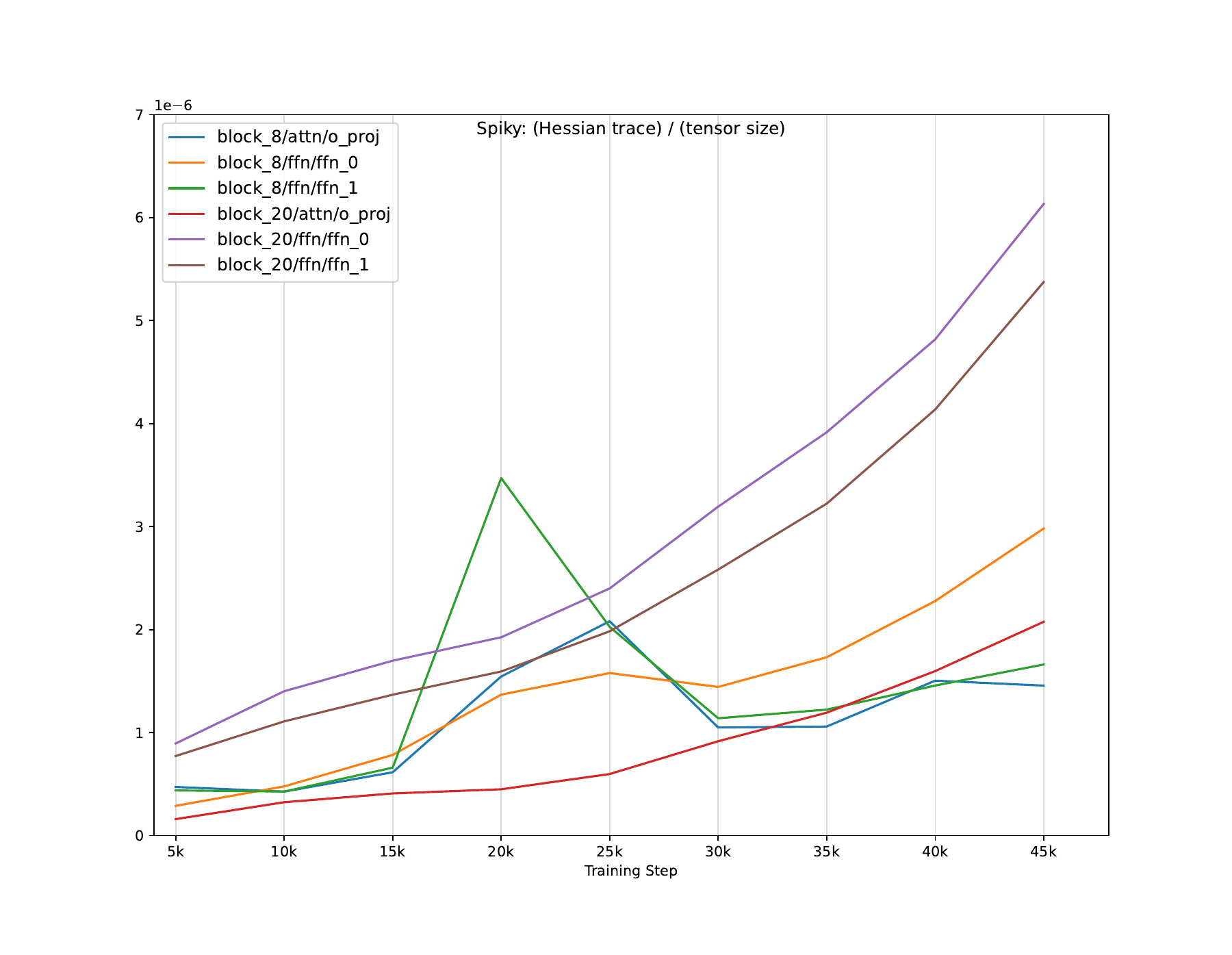}
     %\caption{(a)}
     %\label{fig:y equals x}
 \end{subfigure}\\
 \begin{subfigure}[b]{0.49\textwidth}
     \centering
     \includegraphics[width=\textwidth]{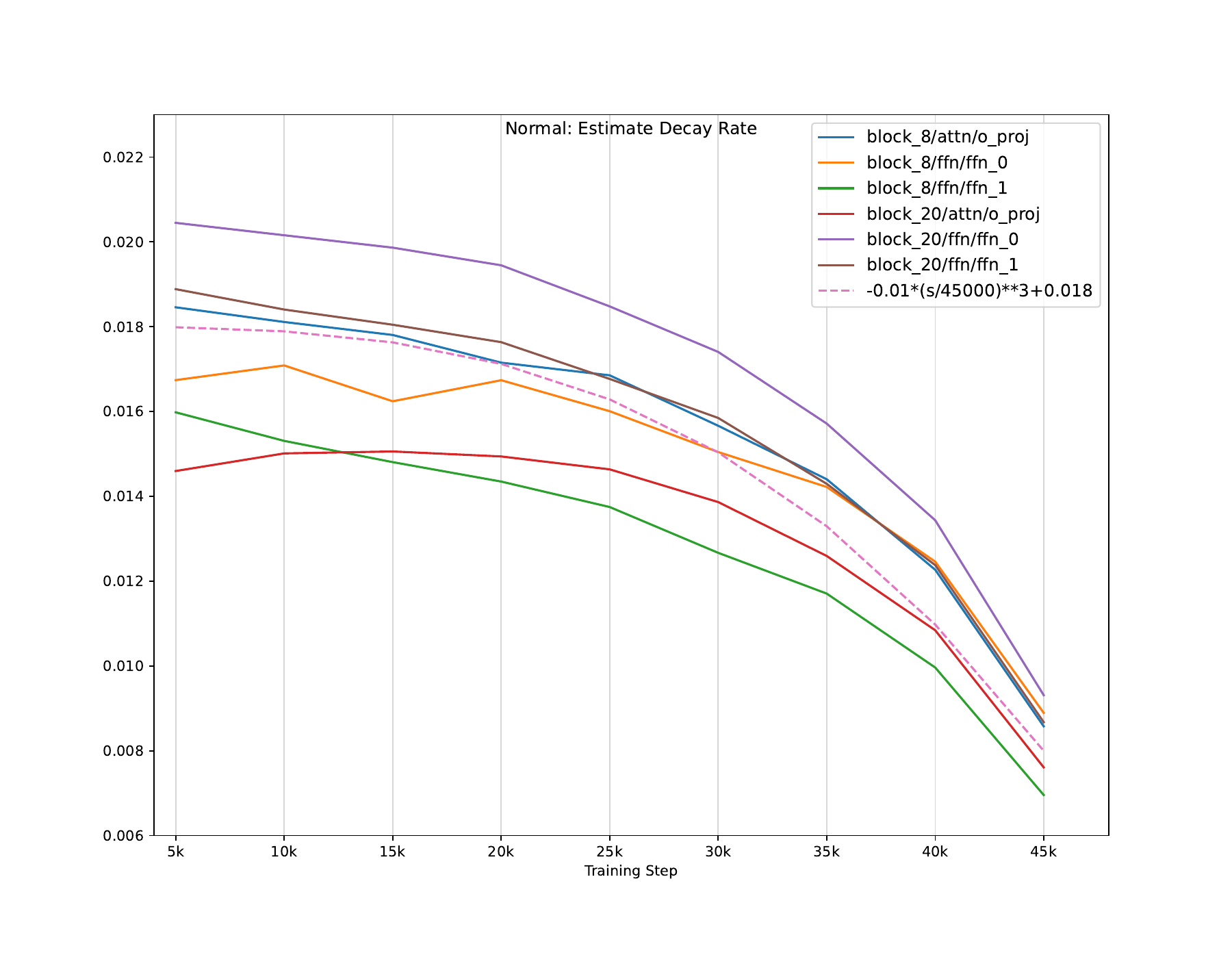}
     %\caption{(a)}
     %\label{fig:y equals x}
 \end{subfigure}
 \begin{subfigure}[b]{0.49\textwidth}
     \centering
     \includegraphics[width=\textwidth]{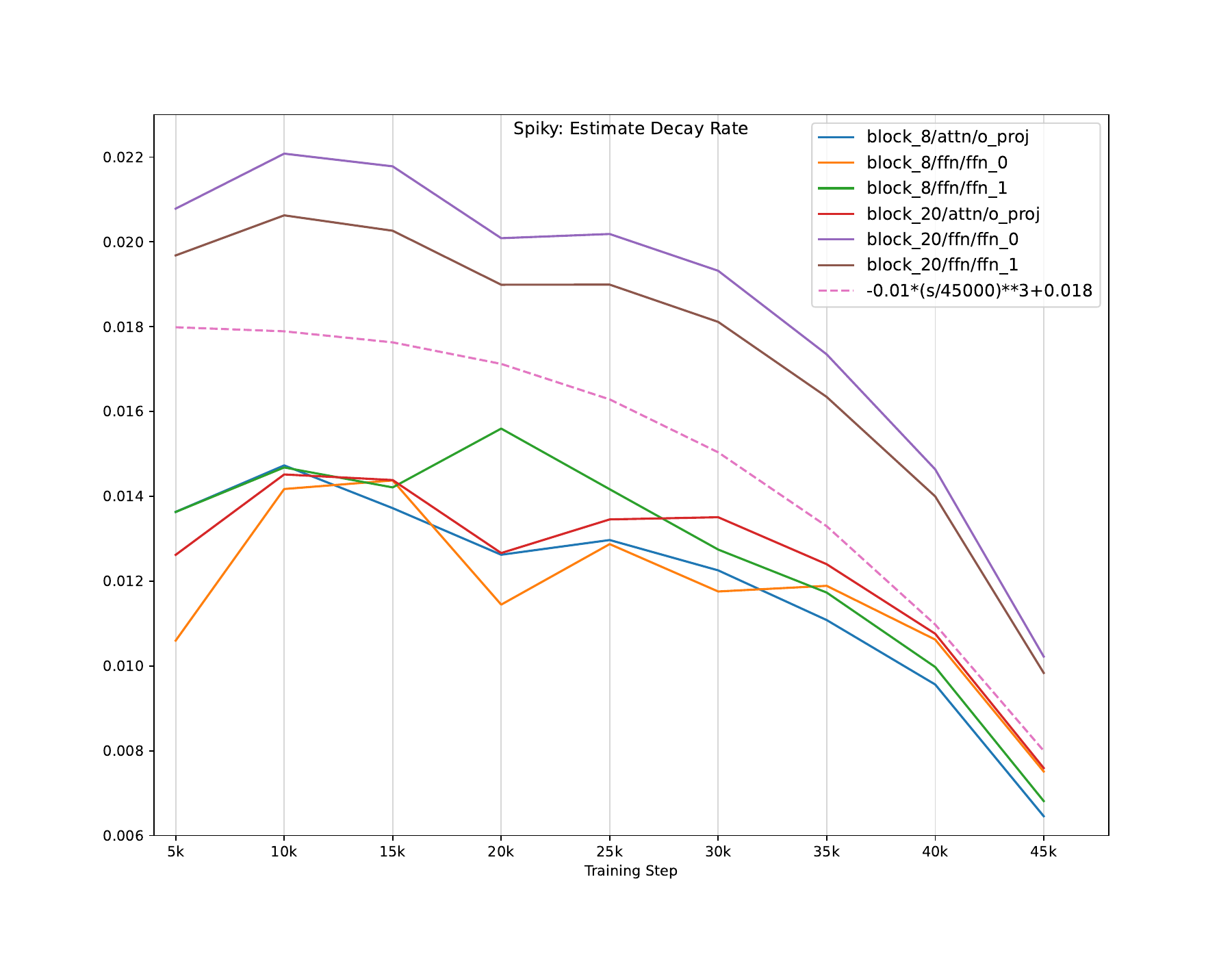}
     %\caption{(a)}
     %\label{fig:y equals x}
 \end{subfigure}
\caption{Estimating an appropriate $\beta_2$ schedule.}
\label{fig:beta2_est}
\end{figure}

The proposal is to set
\begin{equation}\label{eq:beta2_proposal}
\beta_2\approx 1 - \sqrt{\frac{\eta_t\mathfrak{T}}{\sqrt{\mathfrak{N}}}}\;\mbox{ where }\;\mathfrak{T}\eqdef\trace\big(\mathbb{E}(\nabla^2\vmat{W}_{t})\big)/(\cscalar{C}\times\cscalar{D})\mbox{ and }\mathfrak{N}\eqdef\mathbb{E}(\lVert\nabla_\cscalar{B}\vmat{W}_{t}\lVert^2)/(\cscalar{C}\times\cscalar{D}).
\end{equation}
Here, $\eta_t$ is the learning-rate schedule for Adam; We empirically measure $\mathfrak{T}$ the normalized Hessian trace and $\mathfrak{N}$ the mean square of per-batch gradients from validation data. In particular, $\eqref{eq:beta2_proposal}$ depends on the learning-rate schedule $\eta_t$; So if one applies a decaying schedule to the learning-rate, it usually suggests that $\beta_2$ should follow an increasing schedule.

The rationale for \eqref{eq:beta2_proposal} comes from \eqref{eq:randomness}: Because the accumulated gradient noise $\vvec{r}_t$ is decayed by $(\cmat{I}-\eta_{t}\,\cmat{H})$ at each training step, the random fluctuations of the loss dynamics (i.e.~the source of training spikes) has self-correlation within $\approx 1/(\eta_{t}\,\cmat{H})$ steps; Then, our proposal is to ensure that $\vvec{\nu}$, the running average of gradient squares used as an optimizer state in Adam, has self-correlation within $\approx 1/\sqrt{\eta_{t}\,\cmat{H}}$ steps. Since the learning-rate $\eta_t$ in \eqref{eq:randomness} is for SGD, we have to replace it by the adaptive learning-rate $\eta_t/\sqrt{\vvec{\nu}}$ in Adam; We approximate the tensor $\sqrt{\vvec{\nu}}$ by $\sqrt{\mathfrak{N}}$, and the tensor $\cmat{H}$ by $\mathfrak{T}=\trace(\cmat{H})/\dim(\vvec{w})$.

In Figure~\ref{fig:beta2_est}, we first ensure that $\sqrt{\mathfrak{N}}$ is actually a good approximation of $\sqrt{\vvec{\nu}}$ (or vice versa): We plot the ratio $\sqrt{\mathfrak{N}/\mean(\vvec{\nu})}$ in the First Row, for different layers and training steps, comparing the Normal and Spiky configurations. We found that the approximation works well for Normal training, but much worse for Spiky. This is not surprising because the large $\beta_2$ used by Spiky is likely to make any surge of gradient squares last longer in $\vvec{\nu}$, causing it to overestimate the average.

Then, we plot $\sqrt{\mathfrak{N}}$, $\mathfrak{T}$ and finally $\sqrt{\eta_t\mathfrak{T}/\sqrt{\mathfrak{N}}}$ in the Second, Third and Fourth Row of Figure~\ref{fig:beta2_est}, respectively. Although the plots of $\sqrt{\mathfrak{N}}$ and $\mathfrak{T}$ can be spiky, we observe that the overall tendency of the ratio $\sqrt{\eta_t\mathfrak{T}/\sqrt{\mathfrak{N}}}$ is comparatively stable and consistent across all layers, for both Normal and Spiky. In order to capture this tendency, we manually fit a schedule of the form below (shown as pink dashed curves in the Fourth Row of Figure~\ref{fig:beta2_est}).
\begin{equation}\label{eq:1mbeta2_decay_schedule}
(\texttt{final\_value} - \texttt{init\_value})\left(\frac{\mbox{step}}{\mbox{total\_steps}}\right)^{\texttt{exponent}}+\texttt{init\_value},
\end{equation}
with $\texttt{init\_value}=0.018$, $\texttt{final\_value}=0.008$ and $\texttt{exponent}=3$.

We set $(1-\beta_2)$ to the schedule given in \eqref{eq:1mbeta2_decay_schedule} and train a Transformer LM as in \S~\ref{sec:llm_train_settings}, without changing any other hyperparameters. We compare the final validation loss with the original setting $\beta_2=0.95$, and two additional settings $\beta_2=(1-\texttt{init\_value})=0.982$ and $\beta_2=(1-\texttt{final\_value})=0.992$. The results are shown in Table~\ref{tab:beta2_eval}.

We observe that the proposed $\beta_2$ schedule can indeed improve performance in terms of validation loss. Remarkably, the performance significantly differs between $\beta_2=(1-\texttt{init\_value})$ and $\beta_2=(1-\texttt{final\_value})$, even if $\texttt{init\_value}$ and $\texttt{final\_value}$ only differ by $0.01$. We also observe that the training of $\beta_2=(1-\texttt{final\_value})$ is spiky, while $\beta_2=(1-\texttt{init\_value})$ is not (Figure~\ref{fig:beta2_train_loss}). The proposed $\beta_2$ schedule can even improve a bit over $\beta_2=(1-\texttt{init\_value})$, with the final value of $\beta_2$ actually becoming $(1-\texttt{final\_value})$.

\end{document}